\documentclass{article}

\usepackage[utf8]{inputenc} % allow utf-8 input
\usepackage[T1]{fontenc}    % use 8-bit T1 fonts
\usepackage{url}            % simple URL typesetting
\usepackage{booktabs}       % professional-quality tables
\usepackage{amsfonts}       % blackboard math symbols
\usepackage{nicefrac}       % compact symbols for 1/2, etc.
\usepackage{microtype}      % microtypography
\usepackage{amsmath}
\usepackage{amsthm}
\usepackage{amssymb}
\usepackage{algorithmic}
\usepackage{algorithm}
\usepackage{graphbox}
\usepackage{graphicx,subfigure}
\usepackage{tablefootnote}
\usepackage{wrapfig,lipsum,booktabs}
\usepackage{multirow}
\usepackage{color}
\usepackage[dvipsnames]{xcolor}
\usepackage{enumerate}
\usepackage[pagebackref=true,breaklinks=true,colorlinks,bookmarks=true]{hyperref}
\usepackage[numbers]{natbib}

\usepackage{hyperref}

\usepackage[accepted]{icml2021}

\newtheorem{assumption}{Assumption}

\newtheorem{lem}{Lemma}
\newtheorem{defi}{Definition}
\newtheorem{theorem}{Theorem}
\newtheorem{remark}{Remak}
\newtheorem{fact}{Fact}
\newtheorem{col}{Corollary}
\allowdisplaybreaks[4]

\usepackage{amsmath,amsfonts,bm}

\def\eqref#1{eqn.~(\ref{#1})}
\def\1{\bm{1}}

\DeclareMathAlphabet{\mathsfit}{\encodingdefault}{\sfdefault}{m}{sl}
\SetMathAlphabet{\mathsfit}{bold}{\encodingdefault}{\sfdefault}{bx}{n}

\def\A{{\bf A}}
\def\a{{\bf a}}
\def\B{{\bf B}}

\def\D{{\bf D}}

\def\E{{\bf E}}

\def\G{{\bf G}}
\def\H{{\bf H}}
\def\I{{\bf I}}

\def\M{{\bf M}}

\def\Q{{\bf Q}}

\def\R{{\bf R}}

\def\U{{\bf U}}
\def\u{{\bf u}}
\def\V{{\bf V}}

\def\W{{\bf W}}
\def\w{{\bf w}}
\def\X{{\bf X}}
\def\x{{\bf x}}
\def\Y{{\bf Y}}
\def\y{{\bf y}}
\def\Z{{\bf Z}}
\def\O{{\bf O}}

\def\0{{\bf 0}}
\def\1{{\bf 1}}

\def\DM{{\mathcal D}}

\def\FM{{\mathcal F}}
\def\IM{{\mathcal I}}

\def\OM{{\mathcal O}}

\def\RM{{\mathcal R}}

\def\RB{{\mathbb R}}
\def\EB{{\mathbb E}}

\def\PB{{\mathbb P}}

\def\tiQ{{ \widetilde{\Q}}}
\def\tiR{{ \widetilde{\R}}}

\def\tiU{{ \widetilde{\U}}}
\def\bY{{\overline \Y}}
\def\bZ{{\overline \Z}}
\def\bR{{\overline \R}}

\def\Si{\mbox{\boldmath$\Sigma$\unboldmath}}

\def\argmax{\mathop{\rm argmax}}
\def\argmin{\mathop{\rm argmin}}

\def\sgn{\mathrm{sgn}}

\def\orth{\mathsf{orth}}

\newcommand{\DPI}{\texttt{DPI} }
\newcommand{\LP}{\texttt{LocalPower} }

\DeclareMathOperator{\dist}{dist}
\DeclareMathOperator{\gap}{gap}

\icmltitlerunning{Communication-Efficient Distributed SVD via 
	Local Power Iterations}

\begin{document}

\twocolumn[
\icmltitle{Communication-Efficient Distributed SVD via 
	Local Power Iterations}



\icmlsetsymbol{equal}{*}

\begin{icmlauthorlist}
\icmlauthor{Xiang Li}{sms}
\icmlauthor{Shusen Wang}{steve}
\icmlauthor{Kun Chen}{sms}
\icmlauthor{Zhihua Zhang}{sms}
\end{icmlauthorlist}

\icmlaffiliation{sms}{School of Mathematical Sciences, Peking University, China}
\icmlaffiliation{steve}{Department of Computer Science, Stevens Institute of Technology, USA}

\icmlcorrespondingauthor{Xiang Li}{lx10077@pku.edu.cn}

\icmlkeywords{Communication Efficiency, Distributed SVD, Power Iterations, Local Updates}

\vskip 0.3in
]



\printAffiliationsAndNotice{} % otherwise use the standard text.

\begin{abstract}
We study distributed computing of the truncated singular value decomposition problem.
We develop an algorithm that we call \texttt{LocalPower} for improving communication efficiency.
Specifically, we uniformly partition the dataset among $m$ nodes and alternate between multiple (precisely $p$) local power iterations and one global aggregation.
In the aggregation, we propose to weight each local eigenvector matrix with orthogonal Procrustes transformation (OPT).
As a practical surrogate of OPT, sign-fixing, which uses a diagonal matrix with $\pm 1$ entries as weights, has better computation complexity and stability in experiments. 
We theoretically show that under certain assumptions \texttt{LocalPower} lowers the required number of communications by a factor of $p$ to reach a constant accuracy.
We also show that the strategy of periodically decaying $p$ helps obtain high-precision solutions.
We conduct experiments to demonstrate the effectiveness of \texttt{LocalPower}.
\end{abstract}

\section{Introduction}

In this paper we consider the truncated singular value decomposition (SVD) which has broad applications in machine learning, such as dimension reduction~\cite{wold1987principal}, matrix completion~\cite{candes2009exact}, and information retrieval~\cite{deerwester1990indexing}.
Let $\a_1, \cdots, \a_n \in \RB^d$ be sampled i.i.d.\ from some fixed but unknown distribution.
The goal is to compute the $k$ ($k < \min \{ d, n \}$) singular vectors of $\A \triangleq [\a_1, \ldots, \a_n]^\top \in \RB^{n \times d}$.
Let $\V_k \in \RB^{d \times k}$ contain the top $k$ singular vectors.
The power iteration and its variants such as Krylov subspace iterations are common approaches to the truncated SVD. 
They have $\OM (n d)$ space complexity and $\OM (n d k)$ per-iteration time complexity.
They take $\tilde{\OM} ( \log \frac{d}{\epsilon} ) $ iterations to converge to $\epsilon$ precision, where $\tilde{\OM}$ hides the spectral gap and constants \cite{golub2012matrix,saad2011numerical}.

When either $n$ or $d$ is big, the data matrix $\A \in \RB^{n\times d}$ may not fit in the memory, making standard single-machine algorithms infeasible.
A distributed power iteration is feasible and practical for large-scale truncated SVDs.
In particular, we partition the rows of $\A $ among $m$ worker nodes (see Figure~\ref{fig:illustrate}(a)) and let the nodes perform power iterations in parallel (see Figure~\ref{fig:illustrate}(b)).
In every iteration, every node performs $\OM (\frac{n d k}{m})$ FLOPs (suppose the load is balanced), while the server performs only $\OM (d k^2)$ FLOPs.
%In every iteration, two rounds of communications are required, and the total word complexity can be $\OM (d k m)$ or $\OM (d k \log m)$, depending on the computer network structure.

When solving large-scale matrix computation problems, communication costs are not negligible; in fact, communication costs can outweigh computation costs.
The large-scale SVD experiments in \cite{gittens2016matrix,wang2019scalable} show that the runtime caused by communication and straggler's effect\footnote{Straggler's effect means that one outlier node is tremendously slower than the rest, and the system waits for the slowest to complete.} can exceed the computation time.
Due to the communication costs and other overheads, parallel computing can even demonstrate anti-scaling; that is, when $m$ is big, the overall wall-clock runtime increases with $m$.
Reducing the frequency of communications will reduce the communication and synchronization costs and thereby improving the scalability.

\begin{figure*}[tp]
\centering
\subfigure[Partition.]{\includegraphics[height=40mm] {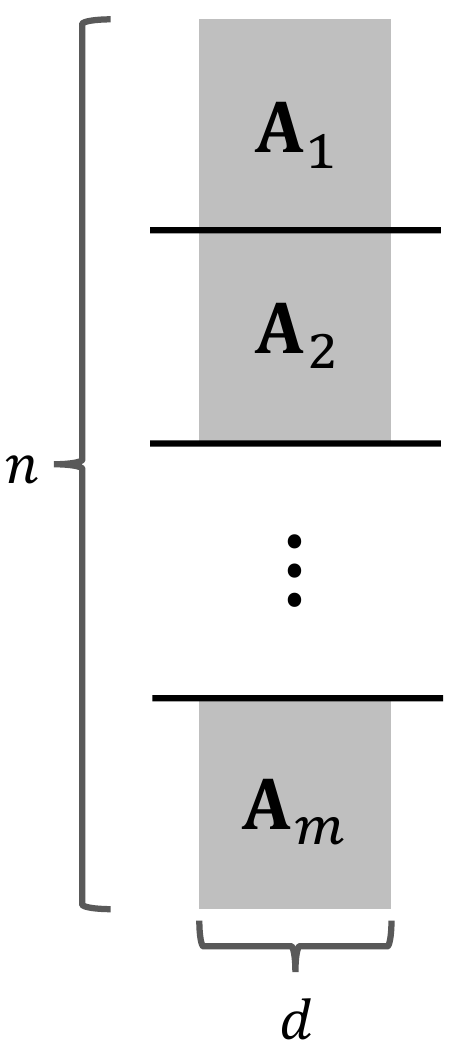}}
\subfigure[Standard parallel power iteration (\texttt{DPI}).]{\includegraphics[height=45mm] {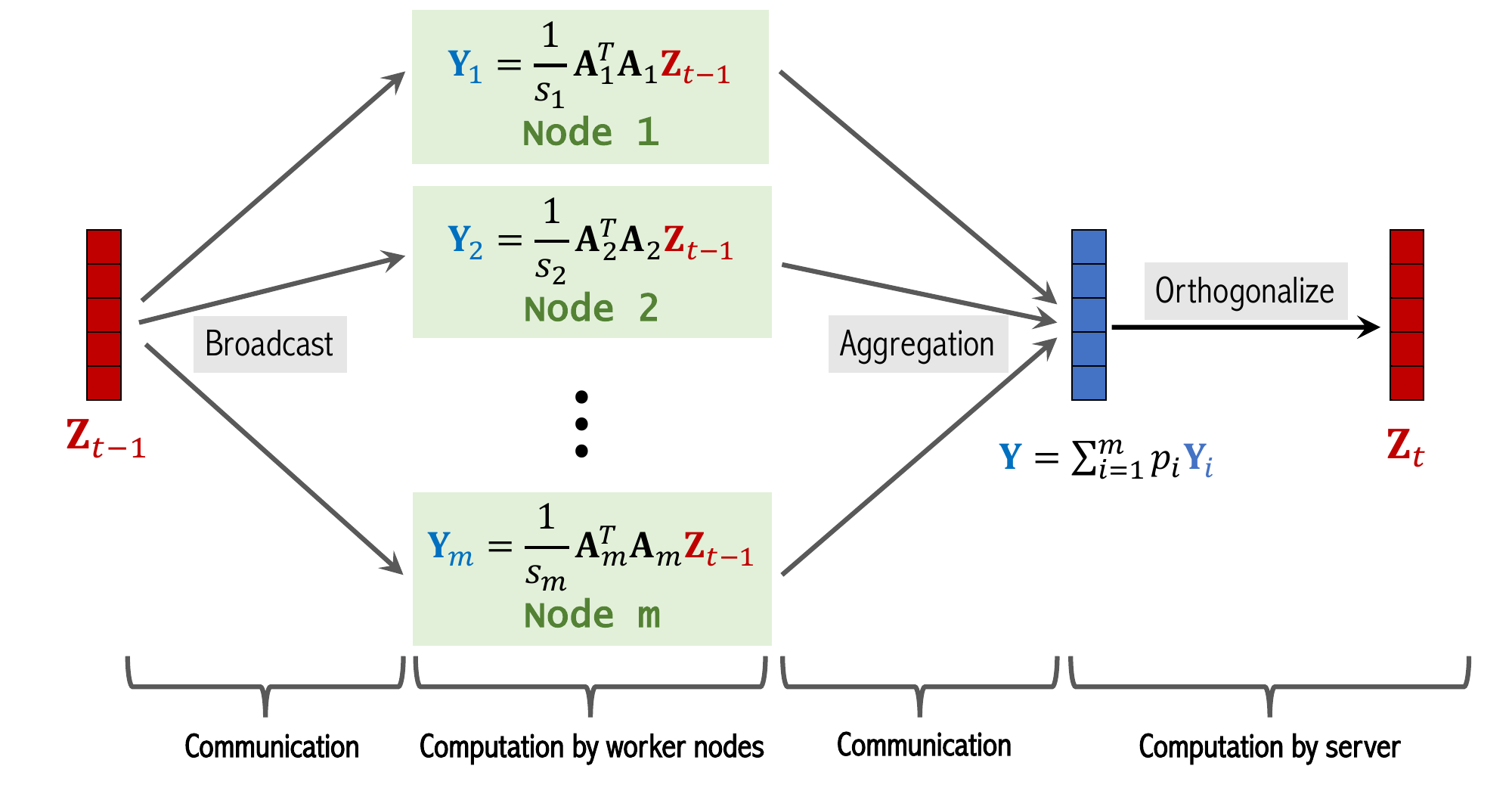}} 
\subfigure[Commonly used symbols]{\includegraphics[height=40mm] {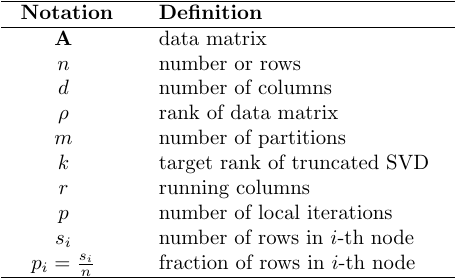}\label{fig:notation}}
\caption{(a) The $n\times d$ data matrix $\A$ is partitioned among $m$ worker nodes.
(b) In every iteration of the distributed power iteration, there are two rounds of communications. Most of the computations are performed by the worker nodes.
(c) Commonly used symbols.}
\label{fig:illustrate}
\end{figure*}

\subsection{Our Contributions}

Inspired by the \texttt{FedAvg} algorithm \cite{mcmahan2017communication}, we propose an algorithm called \texttt{LocalPower} to improve communication-efficiency.
\texttt{LocalPower} is based on the distributed power iteration (\texttt{DPI}) described in Figure~\ref{fig:illustrate}.
The difference is that \texttt{LocalPower} makes every node locally perform orthogonal iterations using its own data for $p$ iterations.
In the case for $p=1$, \texttt{LocalPower} degenerates to \texttt{DPI}.
When $p \ge 2$, local updates are employed to reduce communication frequency.

% \red{We must avoid the notation like $\FM$ and $\OM_k$ in this section. The readers don't know the symbols yet. They will feel uncomfortable to look back and forth. We must state the contributions in plain English.}

In practice, a naive implementation of the proposed \texttt{LocalPower} does not work very well.
We propose three effective techniques for improving \texttt{LocalPower}:
\begin{itemize}
    \item
    We propose to decay the communication interval, $p$, over time.
    In this way, the loss drops fast in the beginning and converge to the optimal solution in the end.
    Without the decay strategy, \texttt{LocalPower} is not guaranteed to converge to the optimum.
    \item 
    Orthogonal Procrustes transformation (OPT) post-processes the output matrices of the $m$ nodes after each iteration so that the $m$ matrices are close to each other. 
    OPT makes \texttt{LocalPower} stable at the cost of more computation.
    \item
    To reduce the computation of OPT, we replace its orthogonal space to the set of all diagonal matrices with $\pm 1$ entries.
    In this way, OPT becomes the sign-fixing technique which is stable (slightly worse than OPT) and efficient.
    Sign fixing was originally proposed by~\citet{garber2017communication} for the special case of $k=1$, while we generalize sign-fixing to $k > 1$. 
\end{itemize}
In summary, this work's contributions include the new algorithm, \texttt{LocalPower}, its convergence analysis, and the effective techniques for improving \texttt{LocalPower}.

The remainder of this paper is organized as follows.
In Section~\ref{sec:pre}, we define notation and give preliminary backgrounds on the orthogonal Procrustes problem and the distributed power iteration.
In Section~\ref{sec:localpower}, we propose \texttt{LocalPower} and its variants and then provide theoretical analysis in Section~\ref{sec:theory}.
In Section~\ref{sec:exp}, we conduct experiments to illustrate the effectiveness of \texttt{LocalPower} and to validate our theoretical results.
In Section~\ref{sec:discuss}, we give further discussions on some aspects of \texttt{LocalPower}.
All proof details can be found at Appendix~\ref{appen:proof}.
In Appendix~\ref{sec:related}, we discuss related work on SVD and parallel algorithms.

\section{Preliminary}
\label{sec:pre}

\paragraph{Notation.}
For any $\A \in \RB^{n \times d}$,  we use $\| \A \|_2$ and $\| \A \|_F$ to denote its spectral norm and Frobenius norm.
Let $\A^{\dagger} \in \RB^{d \times n}$ denote the Moore-Penrose pseudo-inverse of $\A$.
For any positive integer $T$, let $[T]=\{1, 2, \cdots, T\}$.
$\OM_{d \times k}$ is the set of all $d \times k$ column orthonormal matrices ($1 \le k \le d$). 
$\OM_k$, short for $\OM_{k \times k}$, denotes the set of $k \times k$ orthogonal matrices.
$\RM(\U)$ denotes the subspace spanned by the columns of $\U$.
The commonly used notation is summarized in Figure~\ref{fig:notation}.

\paragraph{Power iteration.}
The top $k$ right singular vectors of $\A$ can be obtained by the subspace iteration that repeats
\begin{equation} \label{eq:power}
\Y \: \longleftarrow \:
\M \Z 
\quad \textrm{ and } \quad
\Z \: \longleftarrow \:
\orth \big( \Y \big), 
%\quad \textrm{ with } \quad
%\M = \frac{1}{n} \A^\top \A
\end{equation}
where $\M = \frac{1}{n} \A^\top \A$.
In every power iteration, computing $\Y$ has $\OM ( n d  k )$ time complexity, and orthogonalizing $\Y$ has $\OM (d k^2 ) $ time complexity. 
It is well known that the tangent of principle angles between $\RM(\Z)$ and $\RM(\U_k)$ converges to zero geometrically \cite{arbenz2012lecture,saad2011numerical} and thus so their projection distance.

% \paragraph{Data parallelism.}
% This paper studies distributed algorithms for SVD.
% We consider data parallelism and partition the data (rows of $\A$) among $m$ worker nodes.
% See Figure~\ref{fig:illustrate}(a) for the illustration.
% We partition $\A$ as $\A = [\A_1^\top, \cdots, \A_m^\top]^\top$ where $\A_i \in \RB^{s_i \times d}$ contains $s_i$ rows of $\A$.
% Let $\M =  \frac{1}{n} \A^\top \A \in \RB^{d\times d}$ be the global correlation matrix.
% Let $\M_i = \frac{1}{s_i} \A_i^\top \A_i \in \RB^{d\times d}$ be the local correlation matrix on the $i$-th node.
% It follows that $\M
% % \: = \: \frac{1}{n} \A^\top \A 
% \: = \: \sum_{i=1}^m \frac{1}{n} \A_i^\top \A_i
% \: = \: \sum_{i=1}^m \frac{s_i}{n}  \M_i 
% \: = \: \sum_{i=1}^m p_i \M_i$,
% % \begin{equation*}
% % \M
% % \: = \: \frac{1}{n} \A^\top \A 
% % \: = \: \sum_{i=1}^m \frac{1}{n} \A_i^\top \A_i
% % \: = \: \sum_{i=1}^m \frac{s_i}{n}  \M_i 
% % \: = \: \sum_{i=1}^m p_i \M_i ,
% % \end{equation*}
% where $p_i = \frac{s_i}{n}$ the fraction of data points stored in $i$-th worker node.
% The intermediate variable, $\Y$, in \eqref{eq:power} can be thereby written as $\Y
% \: = \: \sum_{i=1}^m p_i  \M_i \Z 
% \: \in \: \RB^{d\times k},$
% % \begin{equation*}
% % \Y
% % \: = \: \sum_{i=1}^m p_i  \M_i \Z 
% % \: \in \: \RB^{d\times k},
% % \end{equation*}
% which allows for distributed computing.

\paragraph{Distributed power iteration (\texttt{DPI})} is a direct distributed variant of power iteration.
Consider data parallelism and partition the data (rows of $\A$) among $m$ worker nodes. 
See Figure~\ref{fig:illustrate}(a) for the illustration.
We partition $\A$ as $\A^\top = [\A_1^\top, \cdots, \A_m^\top]$ where $\A_i \in \RB^{s_i \times d}$ contains $s_i$ rows of $\A$.
Using $m$ worker nodes and data parallelism, one power iteration works in four steps.
First, the server broadcasts $\Z$ to the workers, which has $\OM (d k)$ or $\OM (d k m)$ communication complexity (depending on the network structure).
Second, every worker (say, the $i$-th) locally computes  
\begin{equation} \label{eq:power_yi}
\Y_i = \M_i \Z \in \RB^{d\times k}
\quad \textrm{ with } \quad
\M_i = \frac{1}{s_i} \A_i^\top \A_i,
\end{equation}
which has $\OM (d^2 k)$ or $\OM (s_i d k)$ time complexity.
Third, the server aggregates $\Y_i$, for all $i \in [m]$, to obtain $\Y = \sum_{i=1}^m p_i \Y_i $;
this step is equivalent to $\Y = \M \Z$, where $\M  =\sum_{i=1}^m p_i \M_i$ with $p_i = \frac{s_i}{n}$.
It has $\OM (d k)$ or $\OM (d k m)$ communication complexity.
Last, the server locally orthogonalizes $\Y$ to obtain $\Z = \orth (\Y)$, which has merely $\OM (d k^2)$ time complexity.
The algorithm is described in Figure~\ref{fig:illustrate}(b).
The following lemma is a well-known result~\cite{arbenz2012lecture,saad2011numerical}.

\begin{lem}
	\label{lem:informal_dpm}
	To obtain a column-orthonormal matrix $\Z$ such that the subspace distance $\text{dist}(\Z, \U_k) \le \epsilon$ (see Definition~\ref{def:proj} for detail), with high probability, the communication needed by \DPI is 
	\begin{equation}
	\label{eq:T0}
	\Omega \left(   \frac{\sigma_k}{\sigma_k - \sigma_{k+1}} \log\left(\frac{ d}{\epsilon}\right)  \right).
	\end{equation}
	Here, $\sigma_j$ is the $j$-th largest singular value of the matrix $\M$. %\red{xxx xxx xxx xxx}.
\end{lem}

\section{Algorithms}
\label{sec:localpower}

\begin{algorithm}[!t]
	\caption{\textbf{\texttt{LocalPower}}} 
	\label{alg:local_power} 
	\begin{algorithmic}[1]
		\STATE {\bfseries Input:} distributed dataset $\{\A_i\}_{i=1}^m$, target rank $k$, iteration rank $r \ge k$, number of iterations $T$. 
		\STATE {\bfseries Initialization:} generate a standard Gaussian matrix, $\Y_0$;
		\FOR{$t=0$ {\bfseries to} $T$}
		\STATE \textbf{Broadcast:} If $t \in \IM_T$, the server sends $\Y_t$ to workers; let $\Y_t^{(i)} \leftarrow \Y_t$;
		\STATE \textbf{Local computation:} For all $i \in [m]$, the $i$-th worker locally computes 
		\vspace{-2mm}
		\begin{equation*}
		    \Z_t^{(i)} = \orth (\Y_t^{(i)})  
		    \quad \textrm{and} \quad 
		    \Y_{t+1}^{(i)} = \tfrac{1}{s_i}  \A_i^\top \A_i \Z_{t}^{(i)} ;
		    \vspace{-2mm}
		\end{equation*}
		\STATE \textbf{Aggregation:} If $(t+1) \in \IM_T$, the server computes $\Y_{t+1} =  \sum_{i=1}^m p_i \Y_{t+1}^{(i)}$; 
		\ENDFOR
		\STATE {\bfseries Output:} $\orth ( \Y_{t+1} )$.
	\end{algorithmic}
\end{algorithm}

\paragraph{\texttt{LocalPower}} is a new algorithm that we propose for improving communication efficiency.
It is described in Algorithm~\ref{alg:local_power}.
Its basic idea is to trade more local power iterations for fewer communications via reducing the communication frequency.
Between two communications, every worker node locally runs \eqref{eq:power_yi} for multiple times.
We let the set $\IM_T$ ($\subseteq [T]$) index the iterations that perform communications; for example, 
\begin{equation*}
    \IM_T \; = \; \big\{ 0, p , 2p , \cdots , T \big\}
\end{equation*}
means that the algorithm communicates once after $p$ lower power iterations.
The cardinality $|\IM_T |$ is the total number of communications.

Suppose \texttt{LocalPower} performs one communication every $p$ iterations.
In $T$ iterations, each worker performs $\OM (s_i d k T)$ FLOPs, the server performs $\OM (d k^2 T/p)$ FLOPs, and the overal communication complexity is $\OM (d k T / p)$.
The standard distributed power iteration is a special case of \texttt{LocalPower} with $p = 1$, that is, $\IM_T = \{0, 1, 2, \cdots , T\}$.\footnote{The reason why we average $\Y_t^{(i)}$ instead of $\Z_t^{(i)}$ is that we hope \LP is reduced to \DPI when $p=1$.}
One-shot SVD, aka divide-and-conquer SVD, \cite{liang2014improved,garber2017communication,fan2019distributed}, is a special case of \texttt{LocalPower} with $p = T$, that is, $\IM_T = \{ 0, T \}$.

\paragraph{Decaying $p$.}
In practice, it is helpful to use a big $p$ in the beginning but let $p=1$ in the end.
For example, we can decrease $p$ by half every few communications.
The rationale is that the error of \texttt{LocalPower} does not converge to zero if $p$ is big; see the theoretical analysis in the next section.
Our empirical observation confirms the theories: if $p$ is set big, then the error drops very fast in the beginning, but it does not vanish with the iterations.

\paragraph{Orthogonal Procrustes Transformation.}
In Algorithm~\ref{alg:local_power}, the $i$-th worker locally computes
\begin{equation*}
    \Y_{t+1}^{(i)} 
    \; = \; \tfrac{1}{s_i}  \A_i^\top \A_i \Z_{t}^{(i)} .
\end{equation*}
When it comes to the time of communication (i.e., $t+1 \in \IM_T$), we replace the equation by the following steps.
First, we choose the device which has the maximum number of samples as a base.
Without loss of generality, we can assume the first device is selected (which indicates $1 = \argmin_{i \in [m]} p_i)$.
Second, we compute
\begin{equation} \label{eq:opt:o}
    \O_t^{(i)}
    \; = \; \argmin_{\O \in \OM_k} \; \big\| \Z^{(i)}_t \O - \Z^{(1)}_t \big\|_F^2.
\end{equation}
%where $\OM_k$ is the set containing all the $k\times k$ matrices.
Eqn.~(\ref{eq:opt:o}) is a classic matrix approximation problem in linear algebra, named as the Procrustes problem~\cite{schonemann1966generalized,cape2020orthogonal}.
The solution to \eqref{eq:opt:o} is referred to as orthogonal Procrustes transformation (OPT) and has a closed form:
\begin{equation*}
    \O_t^{(i)} \; = \; \W_1\W_2^\top, 
\end{equation*}
where $\W_1 \Si \W_2^\top$ is the SVD of $(\Z^{(i)}_t )^\top\Z^{(1)}_t $.
Finally, we compute
\begin{equation*}
    \Y_{t+1}^{(i)} 
    \; = \; \tfrac{1}{s_i}  \A_i^\top \A_i \Z_{t}^{(i)} \O_t^{(i)}.
\end{equation*}

\begin{remark}
Intuitively, such $\O_t^{(i)}$ adjusts $\Z_t^{(i)}$ such that it aligns with $\Z_t^{(1)}$ better.
In an ideal case, all $\Z_t^{(i)}$'s would be identical with $\Z_t^{(1)}$ and thus the aggregation step (line 6 in Algorithm~\ref{alg:local_power}) would be the same as that in \DPI.
From our theory, it is important to use OPT.
It weakens the assumption on the smallness of a residual error which is incurred by local computation.
From our experiments, it stabilizes vanilla \texttt{LocalPower} and achieves much smaller errors. 
\end{remark}

\begin{remark}
To compute such $\O_t^{(i)}$, the $i$-th client should communicate $\Z_t^{(i)}$ to the server, which results in additional communication cost.
However, the cost is the same  in magnitude as that of sending $\Y_{t+1}^{(i)}$ in the aggregation step.
Besides, the computation of $\O_t^{(i)}$ as well as the communication of $\Y_{t+1}^{(i)}$ only happens when $t+1 \in \IM_T$.
These make the additional communication cost affordable.
\end{remark}

\paragraph{Sign-Fixing.}
While OPT makes \texttt{LocalPower} more stable in practice, OPT incurs more local computation.
Specifically, it has time complexity $\OM(dk^2)$ via calling the SVD of $(\Z^{(i)}_t )^\top\Z^{(1)}_t$.
To attain both efficiency and stability, we propose to replace the $k\times k$ matrix $\O^{(i)}$ in \eqref{eq:opt:o} by
\begin{equation} \label{eq:opt:d}
    \D_t^{(i)}
    \; = \; \argmin_{\D \in \DM_k} \; \big\| \Z^{(i)}_t \D - \Z^{(1)}_t \big\|_F^2 ,
\end{equation}
where $\DM_k$ denotes all the $k \times k$ diagonal matrices with $\pm 1$ diagonal entries.
$\D_t^{(i)}$ can be computed in $\OM (k d)$ time by 
\begin{equation*}
    \D_t^{(i)} [j, j]
    \; = \; \sgn \Big( \, \Big\langle  \Z^{(i)}_t [:, j] \, , \: \Z^{(1)}_t [:, j]  \Big\rangle \, \Big), 
    \quad \forall \; j \in [k] .
\end{equation*}
We empirically observe that sign-fixing serves as a good practical surrogate of OPT; it maintains good stability and achieves comparably small errors.
\begin{remark}
If we decay $p$,   $p$ will drop to one after a few communications.
When $p=1$, we stop using OPT (or sign-fixing); we simply set $\O_t^{(i)} $ (or $\D_t^{(i)} $) to the identity matrix.
% When using sign-fixing, we don't need estimate $\D_t^{(i)}$ at all $t \in \IM_T$.
% If $\IM_T$ contains consecutive numbers, for example, $\{t_0, t_0+1, \cdots, t_1\} \subset \IM_T$ with $1 < t_0 \le t_1$ , we only need to compute $\D_{t_0}^{(i)}$ at most, since $\D_t^{(i)} = \I_k$ for $t_0 < t \le t_1$.
% It is also true for OPT.
\end{remark}

\begin{remark}
The technique of sign-fixing has been proposed in the setting of $k=1$ by~\citet{garber2017communication}.
In the $k=1$ setting, OPT and sign-fixing coincide with each other.
In~\eqref{eq:opt:d}, we provide a simple way to extend it to high-dimensional $k > 1$.
We compute $\D_t^{(i)}$ that simultaneously adjusts the signs of columns of $\Z_t^{(i)}$ and $\Z_t^{(1)}$.
There exists other way to handle the high-dimensional sign-fixing problem.
For example, if first $k$ eigenvalues are well-separated from others, we can reduce the top-$k$ sign-fixing problem to the one-dimensional sign-fixing problem instanced $k$ times.
\end{remark}

\section{Convergence Analysis}\label{sec:theory}
In this section we analyze the convergence of \texttt{LocalPower} and show the benefit of OPT under an ideal setting.
We use the projection distance of two subspaces as the metric for convergence evaluation.

\begin{defi} [Projection Distance]
\label{def:proj}
    Let $\U, \tiU \in \OM_{d \times k}$ be any matrices with orthonormal columns.
    The projection distance\footnote{Unlike the spectral norm or the Frobenius norm, the projection norm will not fall short of accounting for global orthonormal transformation. Check~\citet{ye2014distance} to find more information about distance between two spaces.} between them is
    \[
    \dist(\U, \tiU) \; \triangleq \; \big\| \U \U^\top - \tiU \tiU^\top \big\|_2.
    \]
\end{defi}

Projection distance is equivalent to
$ \dist(\U, \tiU) =\sin \theta_k(\U, \tiU)$
where $\theta_k(\U, \tiU)$ denotes the $k$-th largest principal angle between the subspaces spanned by $\U$ and $\tiU$.
Principal angles quantify how different two subspaces are.
We can actually calculate 
\begin{equation*}
    \theta_1 \big(\U, \tiU \big), \;\;
    \theta_2 \big(\U, \tiU \big), \;\;
    \cdots , \;
    \theta_k \big(\U, \tiU \big) 
\end{equation*}
via the SVD of $\U^\top\tiU$.
The $l$-th largest singular value of $\U^\top\tiU$ is equal to $\cos \theta_l (\U, \tiU)$ for all $l = 1, \cdots , k$.

\begin{defi}[Local Approximation] \label{def:eta}
    Let $\M_i = \frac{1}{s_i} \A_i^\top \A_i$ be hosted by the $i$-th worker.
    Let $\M = \frac{1}{n} \sum_{i=1}^m \A_i^\top \A_i = \sum_{i=1}^m p_i \M_i$.
    Define
	\begin{equation*}
	\eta \; \triangleq \; \max_{i \in [m] } \frac{ \|\M_i - \M \|_2 }{  \|\M\|_2 },
	\end{equation*}
\end{defi}
which measures how far the local matrices, $\M_1, \cdots , \M_m$, are from $\M$.
If $s_i = p_i n$ is sufficiently larger than $d$, then $\eta$ is sufficiently small.

\begin{defi} [Residual Error]  \label{def:rho}
If OPT is not used, define
\begin{equation*}
    \rho_t \; \triangleq \; 
    \max_{i \in [m]} \big\|\Z_t^{(i)} - \Z_t^{(1)} \big\|_2.
\end{equation*}
If OPT is used, define
\begin{equation*}
    \rho_t \; \triangleq \; 
    \max_{i \in [m]} 
    \min_{\O \in \OM_k}
    \big\|\Z_t^{(i)} \O - \Z_t^{(1)} \big\|_2.
\end{equation*}
\end{defi}
The residual error $\rho_{t}$  measures how the local top-$k$ eigenspace estimator varies across the $m$ worker.
Based on the definition, using OPT makes $\rho_t$ smaller than without using OPT.
When $t \in \IM_T$, $\Z_t^{(1)} = \cdots =  \Z_t^{(m)}$ and thus $\rho_t = 0$.
When $t \notin \IM_T$, each local update would enlarge $\rho_t$.
Hence, intuitively $\rho_t$ depends on $p$, i.e., the local iterations between two communications.
However, later we will show that with OPT $\rho_t$ does not depend on $p$ (when $p$ is sufficiently large) while it depends on $p$ without OPT.
A residual error is inevitable in previous literature of empirical risk minimization that uses local updates to improve communication efficiency~\cite{stich2018local,wang2018cooperative,yu2019parallel,li2019convergence,li2019communication}.
In our case, it takes the form of $\rho_t$.

\begin{assumption}[Uniformly small residual errors]
\label{assum:error_condition}
Let $r$ be the running column number, $\sigma_j$ be the $j$-th largest singular value of $\M$, and $\epsilon \in (0, 0.5)$ be a constant. 
Assume $\eta \leq \frac{1}{3 \kappa }$ where $\kappa = \|\M\|\|\M^{\dagger}\|$ is the condition number of $\M$.
Assume for all $t \in [T]$,
\begin{equation}
\label{eq:error_cond}
    \eta \cdot 1_{t \notin \IM_T} 
    + (1-p_{\max}) (\rho_t + \rho_{t-1}1_{t \notin \IM_T} )
    \; = \; \OM (\epsilon_0 ),
\end{equation}
where $p_{\max} = \max_{i \in [m]}p_i$,  $1_{t \notin \IM_T}$ is the indication function of the event $\{t \notin \IM_T\}$, and
\[
\epsilon_0 
\; \triangleq \;
\frac{ \sigma_k - \sigma_{k+1}}{\sigma_1 \kappa} \min \left\{ \frac{\sqrt{r} - \sqrt{k-1}}{\tau \sqrt{d}}, \epsilon   \right\} 
\]	
for  some small constant $\tau>0$.
\end{assumption}

% The convergence depends on two assumptions.
% The first is a good approximation of each $\M_i$ (Assumption~\ref{assum:spectrum}), which can be guaranteed by generating each $\A_i$ uniformly sampling columns from $\A$.
% We will discuss this part latter.
% The second is the residual errors $\varPsi_t, \varOmega_t$ are uniformly small (Assumption~\ref{assum:error_condition}).
% They are incurred by multiple local iterations since both quantities vanish when $t \in \IM_T$.
% The residual error is inevitable in previous literature of empirical risk minimization that uses local updates to improve communication efficiency~\cite{stich2018local,wang2018cooperative,yu2019parallel,li2019convergence,li2019communication,li2021delayed}.

% \begin{remark}
% $\rho_{t}$ characterizes how local top-$k$ eigenspace estimators varies among $m$ devices.
% If we use OPT, then $\rho_t$ is defined to be the maximum difference up to all orthogonal matrices; otherwise $\rho_t$ is the absolute difference.
% Besides, $\rho_t$ is zero under some situations.
% If $t \in \IM_T$, based our \LP, $\Z_t^{(1)} = \cdots =  \Z_t^{(m)}$ and thus $\rho_{t} = 0$.
% If Assumption~\ref{assum:spectrum} holds with $\eta = 0$, all devices hold the same data matrix $\M$, then we also have $\rho_{t} = 0$.
% \end{remark}

\begin{theorem}[Convergence for \texttt{LocalPower}]
	\label{thm:main}
	Let $\tau$ be a positive constant, and Assumption~\ref{assum:error_condition} hold.
	Then, after $ | \IM_T|$ rounds of communication where
	\[ T = \Omega \left(   \frac{\sigma_k}{\sigma_k - \sigma_{k+1}} \log\left(\frac{\tau d}{\epsilon}\right)  \right), \]
	with probability at least $1 - \tau^{-\Omega(r+1-k)} - e^{-\Omega(d)}$, we have
	\[ \dist(\Z_T, \U_k) = \sin\theta_k (\Z_T, \U_k) \le \epsilon.   \]
\end{theorem}

Theorem~\ref{thm:main} shows \LP takes $T = \widetilde{\Theta} \left(   \frac{\sigma_k}{\sigma_k - \sigma_{k+1}} \right)$ iterations to obtain an $\epsilon$-optimal solution, the same quantity required by \DPI.
However, \LP uses less communications.
For example, with $\IM_T = \{0,p, 2p, \cdots, T\}$, \LP makes only 
$ | \IM_T| \: = \: \widetilde{\Theta}\left( \frac{1}{p}  \frac{\sigma_k}{\sigma_k - \sigma_{k+1}} \right) $
communications, saving a factor of $p$ than \DPI.

Theorem~\ref{thm:main} depends on Assumption~\ref{assum:error_condition} which requires~\eqref{eq:error_cond} holds for all $t \in [T]$.
What's more, the final error $\epsilon$ is positively related to $\eta$ and $\rho_t$ via~\eqref{eq:error_cond}.
The first part of~\eqref{eq:error_cond} (i.e., $\eta \cdot 1_{t \notin \IM_T} $) is incurred by the variety of $\M_i$'s.
So, if all devices have access to $\M$ (which implies $\M_1=\cdots = \M_m$), then it would vanish.
The second part~\eqref{eq:error_cond} is brought by intermittent communication.
Indeed, if communication happens at iteration $t$ (i.e., $t \in \IM_T$), we have $\rho_t = 0$ and $1_{t \notin \IM_T}=0$, implying ~\eqref{eq:error_cond} holds obviously.
Without communication, $\rho_t$ is likely to grow continually, which is harmful to obtaining an accurate solution.
Therefore, the assumption actually requires the communication interval $p$ is not too large.
From another hand, when $p$ is fixed, the assumption instead imposes restriction on $\eta$ when $t \notin \IM_T$, because we show in Theorem~\ref{thm:rho} that $\rho_t$ is bounded by a function of $\eta$.
If OPT is used, then $\rho_t = \OM(\eta)$, without dependence on $p$.
However, if OPT is not used, then $\rho_{t} = \OM(\sqrt{k} p \kappa^p \eta)$ has an exponential dependence on $p$.

% \red{Consider removing this paragraph.
% {Dependence of $\varPsi_t$ and $\varOmega_t$.} 
% The residual errors ($\varPsi_t$ and $\varOmega_t$) depend on different quantities.
% $\varPsi_t$ mainly depends on the difference between $\M_i$'s, while $\varOmega_t$ mainly depends on $\rho_{t}$, the maximum difference between $\Z_{t-1}^{(i)}$ and the baseline $\Z_{t-1}^{(1)}$.
% As a result of assuming $1 = \argmax\limits_{i \in [m]} p_i$ and using $\Z_{t-1}^{(1)}$ as the baseline, $\varOmega_t$ also depends on a factor of $1-p_{\max}$. 
% Hence, if all data locates in one device or $\IM_T = [T]$, $\varOmega_t$ vanishes as well as $\varPsi_t$.
% It implies Theorem~\ref{thm:main} covers Lemma~\ref{lem:informal_dpm} as a special case.}

% Therefore,~\eqref{eq:error_condition} actually imposes a restriction on the smallness of $\eta$.
% If we use OPT, $\rho_t$ is of order $\OM(\eta)$, a quantity without dependence on $p$ that is the number of local updates.
% If we don't use OPT, $\rho_{t}$ is of order $\OM(\sqrt{k}p \kappa^p \eta)$ with exponential dependence on $p$.
% Hence, in order to make $\rho_{t} = \OM(\epsilon)$, OPT reduces the restriction on $\eta$ from $\OM(\frac{\epsilon}{\sqrt{k}p \kappa^p})$ to merely $\OM(\epsilon)$.
% As a result, we can see that in experiments OPT is more stable and converges with smaller errors.

\begin{theorem}[Benefits of OPT]
\label{thm:rho}
Let $\tau(t) \in \IM_T$ be the nearest communication time before $t$ and $p=t - \tau(t)$. 
Let $\mathrm{e}$ be the natural constant.
Assume $\eta \le \min(\frac{1}{3\kappa}, \frac{1}{p})$.
\begin{itemize}
	\item With OPT, $\rho_t$ is bounded by
	\[
	\min \left\{2\mathrm{e}^2\kappa^pp \eta, \frac{\eta\sigma_1}{\delta_k} + 2 \gamma_k^{p/4}  C_t \right\} = \OM(\eta),
	\]
	where $\gamma_k \in (0, 1)$,  $\delta_k = \Theta(\sigma_{k}-\sigma_{k+1})$, and $\limsup_{t} C_t = \OM(\eta+\epsilon)$.
	\item Without OPT, $\rho_t$ is bounded by
	\[
	4\mathrm{e}\sqrt{k}p \kappa^p \eta=\OM(\sqrt{k}p \kappa^p \eta).
	\]
\end{itemize}
\end{theorem}

Why using OPT has such an exponential improvement on dependence on $p$ in theory?
This is mainly because of the property of OPT.
Let $\O^* = \argmin_{\O \in \OM_{k}} \|\U - \tiU \O\|_F$ for $\U, \tiU \in \OM_{d \times k}$.
Then, up to some universal constant, we have
$\|\U - \tiU \O^*\|_2 \approxeq \dist(\U, \tiU).$
See Lemma~\ref{lem:orth} in  Appendix for a formal statement and detailed proof.
It implies up to a tractable orthonormal transformation, the difference between the orthonormal bases of two subspaces is no larger than the projection distance between the subspaces.
By the Davis-Kahan theorem (see Lemma~\ref{lem:DK}), their projection distance is not larger than $\OM(\eta)$ up to some problem-dependent constants.
However, without OPT, we have to use perturbation theory to bound $\rho_t$, which inevitably results in exponential dependence on $p$.

\section{Experiments}
\label{sec:exp}

\paragraph{Settings.}
We use 15 datasets available on the LIBSVM website.\footnote{This page contains them all. \url{https://www.csie.ntu.edu.tw/~cjlin/libsvmtools/datasets/}. See Table~\ref{tab:dataset} in the Appendix for $n, d$ information.}
The $n$ data samples are randomly shuffled and then partitioned among $m$ nodes so that each node has $s = \frac{n}{m}$ samples.
We set  $m= \max(\lfloor\frac{n}{1000}\rfloor, 3)$ so that each node has $s=1,000$ samples, unless $n$ is too small.
The features are normalized so that all the values are between $-1$ and $1$.
All the algorithms start from the same initialization $\Y_0$.
We fix the target rank to $k=5$.
Our focus is on communication efficiency, so we use communication rounds for evaluating the compared algorithms.
Due to the space limit, we defer more experiment details and additional experiment results to Appendix~\ref{appen:exp_all}.

\paragraph{Compared algorithms.}
We evaluate three variants of \LP: the vanilla version, with OPT, and with sign-fixing.
We compare our algorithms with one-shot algorithms, UDA \cite{fan2019distributed}, WDA \cite{bhaskara2019distributed}, and DR-SVD\footnote{It is a direct distributed variant of Randomized SVD, the latter proposed by~\citet{halko2011finding}.};
the algorithms are described in Appendix~\ref{appen:baseline}.

\paragraph{Final precision.}
In this set of experiments, we study the precision when the algorithms converge.
For three variants of \LP\, we fix $p=4$ (without decaying $p$).
We run each algorithm 10 times and report the mean and standard deviation (std) of the final errors.
Due to limited space, Table~\ref{table:baseline} shows the results on 7 datasets.
Table~\ref{table:baseline_additional} and Figure~\ref{fig:error_box} (in the appendix) present all the results on the 15 datasets.
Out of the 15 datasets, \LP has the smallest error mean and std on 12 datasets.
The results indicate that one-shot methods do not find high-precision solutions unless the local data size is sufficiently large.
% It can be seen that \LP with $p=4$ typically achieves a better precision than the baselines methods, which is an advantage of iterative methods over one-shot methods.
% What's more, \LP seems to be more stable since it has lower variability.
% We will explore the stability of \LP later on.

The final error depends on $p$.
With $p > 1$, the final error, $\lim_{t \rightarrow \infty} \sin\theta_k(\Z_t, \U_k)$, does not convergence to zero; instead, it remains to be a constant after a number of iterations.
Figure~\ref{fig:sinp} shows that the final error depends on $p$: the bigger $p$ is, the bigger the final error is.
The final error is not sensitive to $p$.
The final error stops growing with $p$ when $p$ is sufficiently large.
Note that \LP as $p \rightarrow \infty$ becomes a one-shot algorithm, that is, the algorithm performs only one aggregation.\footnote{The one-shot method is different from those we introduced in related work. It simply averages local top-$k$ eigenvectors rather than distributed averaging methods (see Algorithm~\ref{alg:unweited_distributed} and~\ref{alg:weited_distributed}). }
One-shot algorithms typically have reasonable empirical performance and theoretical bounds.

The final error depends on $m$.
Big $m$ means smaller local sample size, $s = \frac{n}{m}$, and thereby big matrix approximation error, $\eta$ (in Definition~\ref{def:eta}).
Our theory indicates that big $m$ (and thereby big $\eta$) is bad for the final error.
The empirical results in Figure~\ref{fig:sinp} corroborate our theories.

% \paragraph{The error dependence on $p$ and $m$.}
% From the previous results, we find that the error $\sin\theta_k(\Z_T, \U_k)$ won't vanish even when $T$ is sufficiently large.
% The error is closely related to $p$ and the local data size $s$ (which is equivalent to $m=\frac{n}{s}$ when $n$ is fixed).
% We explore the error dependence on $p$ and $m$ in Figure~\ref{fig:sinp}.
% We have the following observations.
% When $p =1$ (or $m$ is quite small), the error can be reduced to (nearly) 0.
% The error increases as $p$ grows but has a finite limit, implying the error is not sensitive to $p$ even when $p$ is very large. 
% Indeed, at the extreme case where $p$ goes infinity, \LP is reduced to one-shot method\footnote{The one-shot method is different from those we introduced in related work. It simply averages local top-k eigenvectors rather than distributed averaging methods (see Algorithm~\ref{alg:unweited_distributed} and~\ref{alg:weited_distributed}). } and the error reaches its highest value.
% In addition, when the total number of nodes $m$ is reduced, the error decrease as each node hosts more data and intuitively the estimate based on local matrix $\M_i$ will be more accurate. 
% These observations imply that we can obtain arbitrary accuracy by reducing $p$ or increasing $s$.
% In some cases (for example, Federated Learning~\cite{mcmahan2017communication}), increasing $s$ is forbidden because it grants the server the right to manipulate data owned by each user (or nodes) and violates users' privacy.
% Hence, the decay strategy is more practical.

\begin{table*}[htbp]
	%\vspace{-3mm}
	\newcommand{\tabincell}[2]{\begin{tabular}{@{}#1@{}}#2\end{tabular}}  
	\caption{We report the errors of three proposed algorithms and three baselines methods on seven datasets.
	We show the mean errors of ten repeated experiments with its standard deviation enclosed in parentheses.
	The result of full fifteen datasets is shown in Table~\ref{table:baseline_additional}.
	}
	\label{table:baseline}
	\vspace{3mm}
	\centering 
		\resizebox{\textwidth}{18mm}{
\begin{tabular}{c|c|c|c|c|c|c}
	\toprule
	{\multirow{2}{*}{Datasets}}& 
	\multicolumn{3}{c}{\LP $(p=4)$} \vline&
	{\multirow{2}{*}{\texttt{DR-SVD}}}& 
	{\multirow{2}{*}{\texttt{UDA}}}& 
	{\multirow{2}{*}{\texttt{WDA}}}\\
	\cline{2-4}
	& OPT & Sign-fixing & Vanilla &  &  & \\
	\hline
	A9a& \textbf{4.09e-03} (\textbf{4.20e-4}) & 5.82e-03 (1.41e-3) &8.13e-02 (3.44e-2)&4.63e-02 (9.24e-3)& 2.64e-02 (1.58e-2) &2.40e-02 (1.50e-2)\\
	Abalone & \textbf{3.16e-03} (2.89e-3)& 3.85e-03 (\textbf{2.54e-3}) &3.03e-02 (5.70e-2)& 3.20e-01 (2.30e-1)& 1.03e-01 (9.38e-2)&1.03e-01 (9.18e-2)\\
	Acoustic & \textbf{1.83e-03} (4.40e-4)& 2.03e-03 (\textbf{3.90e-4})& 2.38e-03 (8.5e-4)& 1.54e-02 (6.59e-3) &7.76e-03 (2.64e-3)& 6.67e-03 (2.42e-3)\\
	Combined &6.01e-03 (1.59e-3)& \textbf{5.57e-03} (\textbf{1.05e-3})&2.47e-02 (3.40e-2)& 5.19e-02 (6.23e-3) &4.63e-02 (2.97e-3)&4.16e-02 (2.76e-2)\\
	Connect-4 &  \textbf{1.27e-02} (4.52e-3)&1.81e-02 (\textbf{3.79e-3})&1.70e-02 (4.35e-3)& 1.61e-02 (2.96e-3)& 1.65e-01 (3.48e-2)&1.56e-0 1(3.26e-2)\\
	Covtype & 7.38e-03 (6.50e-4)& \textbf{6.23e-03} (\textbf{4.70e-4}) & 1.28e-02 (1.88e-3)& 1.82e-01 (8.73e-2)& 6.09e-02 (9.70e-3)& 5.60e-02 (9.41e-3)\\
	MSD & 9.90e-03 (1.21e-3)& \textbf{9.62e-03} (\textbf{5.20e-4})&1.44e-02 (1.58e-3)& 3.01e-02 (9.64e-3)& 1.55e-02 (1.39e-3)&1.92e-02 (1.14e-3)\\
	\bottomrule
\end{tabular}}
\vspace{-0.1in}
\end{table*}
\begin{table*}[!ht]
    \vspace{3mm}
 	\begin{minipage}[b]{0.595\textwidth}
	\caption{Error comparison of \texttt{LocalPower} with decay strategy under the same setting of Table~\ref{table:baseline}.
	See Table~\ref{table:baseline_decay_additional} for full results.
	In theory, \LP with decay strategy achieves zero error.
	}
	\label{table:baseline_decay}
	\vspace{3mm}
	\centering 
			\resizebox{\textwidth}{18.5mm}{
	\begin{tabular}{c|c|c|c}
		\toprule
		{\multirow{2}{*}{Datasets}}& 
		\multicolumn{3}{c}{\LP with $p=4$ and the decay strategy}\\
		\cline{2-4}
		& OPT & Sign-fixing & Vanilla  \\
		\hline
		A9a& 4.84e-03 (1.40e-02)& 1.52e-03 (4.08e-03) &\textbf{3.11e-04} (\textbf{4.84e-04}) \\
		Abalone  &
		\textbf{3.50e-10} (4.10e-10)&4.14e-10  (\textbf{4.00e-10})&6.12e-10 (6.77e-10)\\
		Acoustic &
		\textbf{1.40e-05} (\textbf{2.16e-05}) &1.92e-05 (3.72e-05) &2.28e-05 (4.91e-05)\\
		Combined &
		3.68e-03 (5.63e-03) &7.74e-03 (1.70e-02)& \textbf{2.99e-03} (\textbf{3.88e-03}) \\
		Connect-4 &
		4.90e-03 (8.47e-03)&3.58e-03 (4.35e-03)&\textbf{3.09e-03} (\textbf{3.16e-03})\\
		Covtype &
		5.57e-04 (1.55e-03)&\textbf{4.95e-05} (\textbf{5.40e-05}) &8.01e-05 (8.62e-05)\\
		MSD & 2.75e-05 (3.34e-05)&\textbf{2.47e-05} (3.27e-05)& 3.02e-05 (\textbf{2.10e-05})\\
		\bottomrule
	\end{tabular}}
\end{minipage} 
\hfill
\begin{minipage}[b]{0.395\textwidth}
\caption{
	The value of $\eta$ under uniform partitions on some datasets. 
	It can be seen that for a fixed $n$, the larger $m$, the larger $\eta$.
	Full results see Table~\ref{table:2}.
}
\label{table:eta}
	\vspace{3mm}
\centering 
\resizebox{\textwidth}{18.5mm}{
\begin{tabular}{c|c|c|c}
	\toprule
	Dataset & $m=20$ & $m=40$ & $m=60$\\
	\midrule
	A9a&0.034 & 0.0563 &0.0701\\
	Abalone&0.1089 &0.23 &  0.2458 \\
	Acoustic&0.0063 &0.0107 &0.0134 \\
	Combined&0.006 & 0.0089 &0.0113 \\
	Connect-4&0.0376& 0.054 & 0.0771 \\
	Covtype&0.0078 &0.011 & 0.0159 \\
	MSD&0.0007& 0.0009 &0.0012\\
	\bottomrule
\end{tabular}}
\end{minipage}
\end{table*}

% \red{Shusen: Here, you use $\sin\theta_k (\Z_t, \U_k)$. But in the theory section, you use ``dist''. The two are equivalent. Why not using only one notation?}

\paragraph{Effect of local power iterations.}
In this set of experiments, we set $p$ to $1$, $2$, $4$, or $8$ (without decaying $p$) and compare the convergence curves.
Note that \LP with $p = 1$ is the standard distributed power iteration (\texttt{DPI}).
We plot the error, $\sin\theta_k (\Z_t, \U_k)$, against communications.
The convergence curves indicate how $p$ affects the communication efficiency.
Figure~\ref{fig:p} shows the experimental results on one dataset.
Due to page limit, the results on the other datasets are left to the appendix; see Figures~\ref{fig:p_orth},~\ref{fig:p_sign}, and~\ref{fig:p_iden}.
In all the experiments, large $p$ leads to fast convergence in the beginning but has a nonvanishing error in the end.

Some machine learning tasks, such as principal component analysis and latent semantic analysis~\cite{deerwester1990indexing}, do not require high-precision solutions.
In this case, \LP is advantageous over \texttt{DPI}, as \LP finds a satisfactory solution using very few communications.
For two-stages methods like~\cite{garber2017communication}
It is also implied that \LP helps
If a higher precision is required, we can decay $p$ so that \LP will have the same precision as \texttt{DPI}.
While one-shot algorithms are more communication-efficient, their precision is too low unless each node has a large sample size.

% Previous experiment results show the errors are typically smaller than those of one-shot algorithms.
% To obtain a high-precision solution, we could make use of the decay strategy or increase the per-device sample size (both explored later on).
% However, in machine learning tasks such as principal component analysis and latent semantic analysis~\cite{deerwester1990indexing}, high-precision solutions are unnecessary.
% In such tasks, \texttt{LocalPower} can solve large-scale truncated SVD using a small number of communications.
% Besides, \texttt{LocalPower} serves as an efficient method to provide a good initial eigenvector matrix, which is required by multi-stage PCA algorithms~\cite{garber2017communication}.

\paragraph{The decay strategy.}
We have observed that large $p$ fastens initial convergence but enlarges the final error.
By contrast, $p=1$ has the lowest error (which actually can be zero) but also the lowest convergence rate.
Similar phenomena have been previously observed in distributed empirical risk minimization~\cite{wang2018adaptive,li2019communication}.
To allow for both fast initial convergence and vanishing final error, we are motivated to decay $p$ gradually.
We halve $p$ every iteration until it reaches $1$.
We apply the decay strategy to the three variants of \texttt{LocalPower}.
For each setting and each dataset, we repeat the experiment 10 times and report the mean and std.
Table~\ref{table:baseline_decay} and Figure~\ref{fig:decay} show the results on some datasets.
The results on all the 15 datasets are left to the appendix; see Table~\ref{table:baseline_decay_additional}, Figures~\ref{fig:p_orth_decay},~\ref{fig:p_sign_decay}, and~\ref{fig:p_iden_decay}.
The decay strategy not only makes convergence faster but also improves the final precision well.

% We test the strategy for the three versions of \LP on fifteen LIBSVM datasets.
% We show the errors from ten repeated experiments partially in Table~\ref{table:baseline_decay} and a typical convergence curve in Figure~\ref{fig:decay}.
% The full results of convergence can be found in the Appendix (Table~\ref{table:baseline_decay_additional}, Figure~\ref{fig:p_orth_decay},~\ref{fig:p_sign_decay} and~\ref{fig:p_iden_decay}).
% All results show that the decay strategy not only converges faster than the baseline \DPI, but also improves solution precision greatly.

\paragraph{Stability.}
In almost all the  experiments, \LP with OPT has smaller std and more stable convergence curves than \LP without OPT.
Why does OPT improve stability.
Theorem~\ref{thm:rho} shows that with OPT, $\rho_t$ (in Definition~\ref{def:rho}) has a linear function of $p$.
Even if $p$ is large, Assumption~\ref{assum:error_condition} can be satisfied, and thus Theorem~\ref{thm:main} guarantees the convergence of \LP with OPT.
However, Theorem~\ref{thm:rho} shows that without using OPT, $\rho_t$ is an exponential function of $p$.
If $p$ is large, Assumption~\ref{assum:error_condition} is violated, and thus the convergence of \LP without OPT is not guaranteed.

% As argued above Theorem~\ref{thm:rho}, the uniform smallness of~\eqref{eq:error_cond} is reduced to the smallness of $\eta$.
% Without OPT, $\rho_{t}$ exponentially depends on $p$, implying $\eta$ should be exponentially small.
% For a give partition $\{ \A_i \}_{i=1}^m$, it is often the case the exponential smallness of $\eta$ fails to establish.
% By contrast, with OPT, $\rho_{t}$ has no dependence on $p$, implying the requirement on $\eta$ is much weaker than that of the vanilla.
% Hence, OPT stabilizes \LP by weakening the restriction on $\eta$.

Sign-fixing is practical alternative to OPT.
Table~\ref{table:baseline} and Figures~\ref{fig:p_orth} and~\ref{fig:p_sign} show that sign-fixing has comparable stability as OPT.
To explain why sign-fixing works, we first explain what causes instability.
Note that if we flip the signs of some columns of $\Z_{t}^{(i)}$, the subspace $\RM(\Z_{t}^{(i)})$ remains the same.
During the local power iterations on the $i$-th node, the signs of the columns of $\Z_{t}^{(i)}$ can flip.
While the sign flipping does not affect $\RM(\Z_{t}^{(i)})$, it changes the outcome of the aggregation of $\Z_{t}^{(1)}, \cdots, \Z_{t}^{(m)}$.
The sign-fixing method can counteract sign flippings and thereby stabilizes \texttt{LocalPower}. 

%%Shusen: I don't think it helpful to argue that sign-fixing is better than OPT.
% To solve the issue,~\citet{garber2017communication} proposes a sign-fixing technique when $k=1$.
% In our case of $k > 1$, using orthogonal procrustes transformation is sometimes overcorrect, since it relies heavily on the baseline matrix $\Z_t^{(1)}$.
% As a result, the quality of $\Z_t^{(1)}$ largely determines the efficiency of aggregation.
% By contrast, using sign-fixing will not rely on the baseline matrix $\Z_t^{(1)}$ when each $\Z_t^{(i)}$ is close to $\U_k$ enough (see Lemma 3 of~\citet{garber2017communication}).
% This explains why using sign-fixing is more stable sometimes.

Table~\ref{table:baseline_decay} shows that \LP with decaying $p$ has better stability.
With the decaying strategy used, $p$ will drop to $1$ after several communications, and \LP becomes the standard \texttt{DPI} which does not suffer from the instability issue.

\begin{figure*}[!ht]
	\centering
	\hspace{-0.1in}
	\subfigure[The performance of \LP with different $p$ on Covtype dataset]{
		\includegraphics[height=36mm]{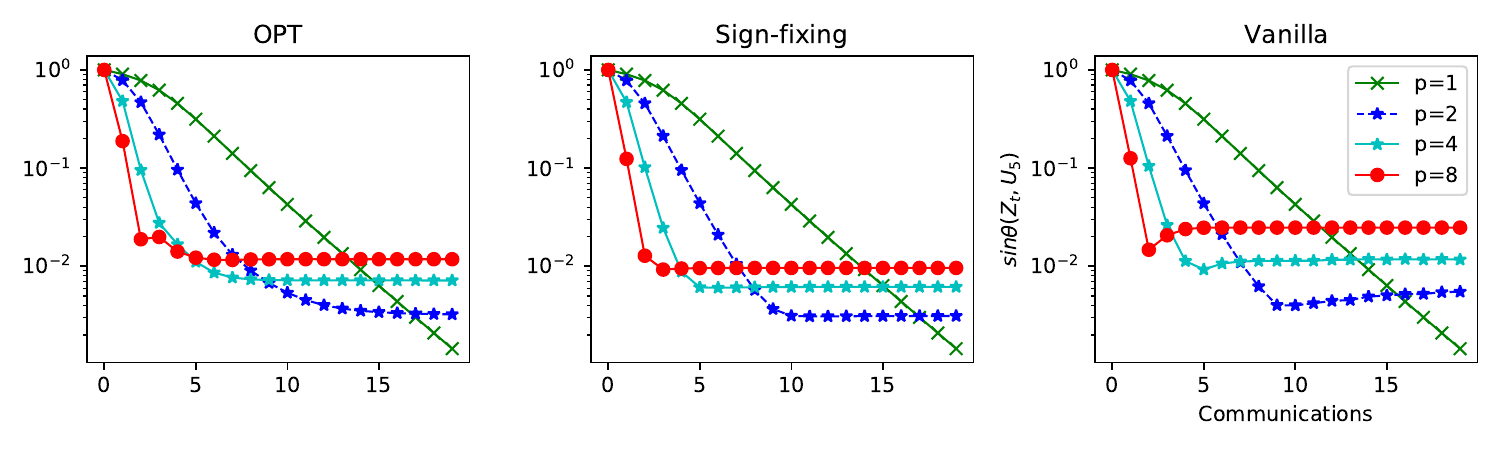}
		\label{fig:p}
	}\hspace{-0.1in}
	\subfigure[Stability on A9a dataset]{
		\includegraphics[height=36mm] {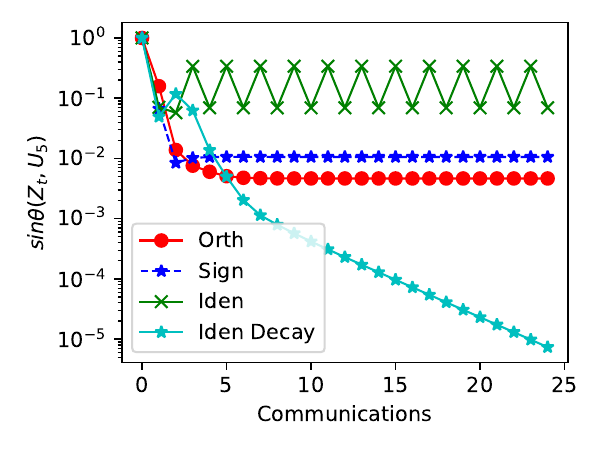}
		\label{fig:stable}
	}\hspace{-0.1in}
		\vspace{-0.1in}
	\caption{(a) We illustrate the convergence of \LP with different $\FM$'s and various $p$ on Covtype dataset where $\A \in \RB^{581,012 \times 54}$.
	See Figure~\ref{fig:p_orth},~\ref{fig:p_sign} and~\ref{fig:p_iden} for full results.
	(b) The vanilla \LP sometimes fluctuates and even diverges (see Figure~\ref{fig:p_iden} for full results). We can stabilize it in two ways: (i) use $\OM_{k}$ or $\DM$ instead or (ii) use the decay strategy.}
	\label{fig:p_all}
	%\vspace{-0.2in}
\end{figure*}

\paragraph{Effect of local sample size.}
Since the $n$ data samples are partitioned among $m$ nodes uniformly at random, every node holds $s=\frac{n}{m}$ samples.
Figure~\ref{fig:m} shows that small $m$, equivalently, big $s$, is good for \texttt{LocalPower}.
We use $\eta = \max_{i \in [m]} \|\M_i -\M\|_2/\|\M\|_2$ to measure the difference between a local covariance matrix and the full one.
We give the values of $\eta$ under different uniform partitions in Table~\ref{table:eta}.
It shows that if $s$ is large (so $m$ is small), $\eta$ is small, which implies $\M_1, \cdots , \M_m$ well approximate the global matrix $\M$, and the residuals accumulated by the local iterations are small.
It in turn makes the curves with small $m$ have small errors.
This can be explained by our theories.

\begin{figure*}[!ht]
    \vspace{3mm}
	\centering
	\subfigure[Decay strategy]{
		\includegraphics[height=36mm] {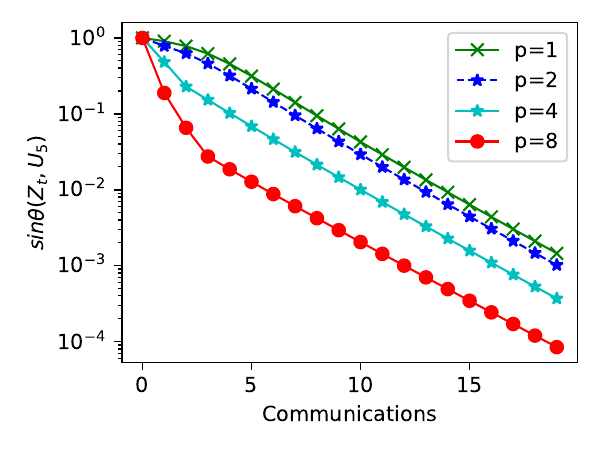}
		\label{fig:decay}
	}
	\subfigure[Vary device number $m$]{
		\includegraphics[height=36mm] {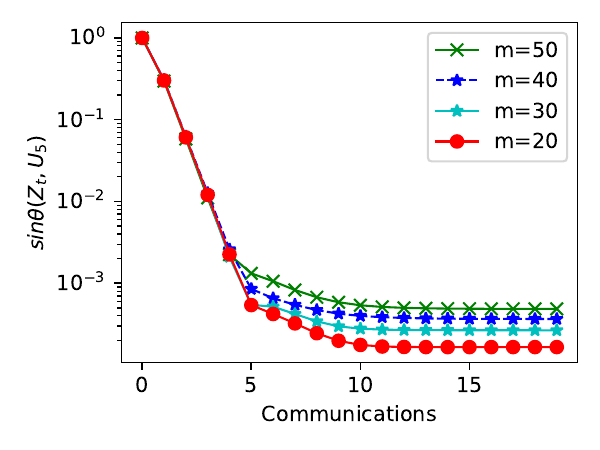}
		\label{fig:m}
	}
	\subfigure[Error dependence on $p$ and $m$]{
		\includegraphics[height=36mm] {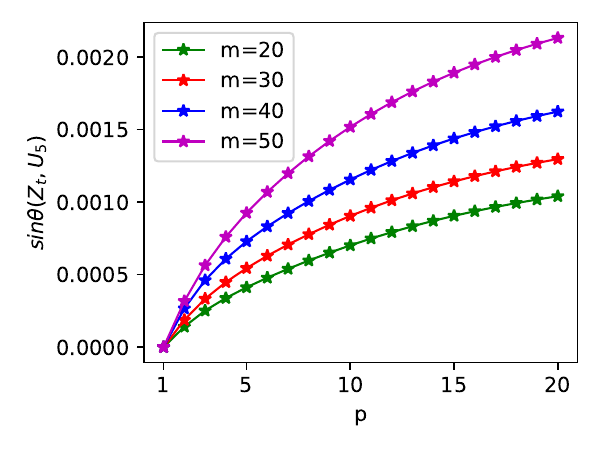}
		\label{fig:sinp}
	}
	\vspace{-1mm}
	\caption{Some results on Covtype dataset.
		(a) A typical convergence curve of the decay strategy.
		See Figure~\ref{fig:p_orth_decay},~\ref{fig:p_sign_decay} and~\ref{fig:p_iden_decay} for full results.
		(b) The smaller $m$, the faster convergence as well as the smaller error. See Figure~\ref{fig:m_orth} and \ref{fig:m_sign} for full results.
		(c) The error depends positively on $p$ and $m$.
		See Figure~\ref{fig:sin_m_p} for full results.
	}
	\label{fig:ver}
\end{figure*}

\section{Discussion}
\label{sec:discuss}
\paragraph{Smallness on $\eta$.}
Theorem~\ref{thm:main} requires $\eta = \OM(\frac{1}{\kappa})$ which might be too stringent in practice. 
If we use a refined analysis just like~\citet{guo2021privacy}, it can be relaxed to $\eta = \OM({1})$ as well as $\epsilon_0$ whose dependence on $\kappa$ can be removed.\footnote{In particular,~\citet{guo2021privacy} analyzes the convergence of the virtual sequence in a form of $\overline{\Z}_t = \sum_{i=1}^np_i\Z_t^{(i)}\D_t^{(i)}$, while we focus on the weighted $Y_t^{(i)}$, i.e., $\overline{\Y}_t= \sum_{i=1}^np_i\Y_t^{(i)}\D_t^{(i)}$.
Roughly speaking, $\|\overline{\Y}_t\|$ is about $\|\M\|_2$ larger than $\|\overline{\Z}_t\|$, while $\|\overline{\Y}_t^\dagger\|$ is about $\|\M^\dagger\|_2$ smaller than $\|\overline{\Z}_t^\dagger\|$.
It leads to an additional factor $\kappa = \|\M\|_2\|\M^\dagger\|_2$.
 } 
 Besides, the concurrent work \citep{charisopoulos2020communication} provides sharper analysis on one-shot average via OPT, which might be used to refine our analysis and relax the strictness on $\eta$ further.
 
 \paragraph{Increase local sample size.}
In addition to  OPT or the decay strategy, we find that increasing local data size also reduces the final error.
Intuitively, if $s_i$ is sufficiently large, then $\M_i = \frac{1}{s_i}\A_i^\top\A_i$ will be very close to $\M = \frac{1}{n}\A^\top\A$.
Actually, this is true if we construct each $\A_i$ by sampling uniformly from the overall data $\A$ (see Lemma~\ref{lem:uniform}).
Therefore, to make $\eta$ sufficiently small, we can increase local data size.
If the total number of rows $n$ is fixed in advance, increasing each $s_i$ is equivalent to decreasing the number of worker nodes $m$.

The term $\eta = \max_{i \in [m]} \|\M_i -\M\|_2/\|\M\|_2$ is commonly used to analyze matrix approximation problems.
It aims to ensure each $\A_i$ is a typical representative of the whole dataset $\A$.
Prior work \cite{gittens2013revisiting,woodruff2014sketching,wang2016spsd} showed that uniform sampling and the partition size in Lemma~\ref{lem:uniform} suffice for that $\M_i$ well approximates $\M$.
The proof is based on matrix Bernstein~\cite{tropp2015introduction}.
Therefore, under uniform sampling, the smallness of $\eta$ means sufficiently large local dataset size (or equivalently a small number of worker nodes).
This can be also seen in Table~\ref{table:eta}.

One may doubt the motivation of each device anticipating the cooperated eigenspace estimation due to the large local dataset assumption.
Here we focus on the empirical PCA rather than the population PCA.
This implies we inevitably suffer a statistic error that will diminish if we have an infinite number of total samples.
As a result, if $m$ devices participate in the training with comparable local data size, the statistical error can be reduced by a factor of $\sqrt{m}$.
See Appendix~\ref{appen:stat_error} for more details.

\begin{lem}[Uniform sampling.]
	\label{lem:uniform}
	Let $\epsilon, \delta \in (0, 1)$.
	Assume the rows of $\A_i$ are sampled from the rows of $\A$ uniformly at random.
	Assume each node has sufficiently many samples, that is, for all $i \in m$, 
	\[
	s_i  \geq \frac{3 \mu \rho}{\epsilon^2}\log \big( \frac{\rho m}{{\delta}} \big),
	\]
	where $\rho= \mathrm{rank}(\A)$ and $\mu$ is the row coherence of $\A$.\footnote{The row coherence of $\A$ is defined by $\mu(\A) = \frac{n}{d} \max_{j} \| \u_j \|_2^2 \in [1, \frac{n}{d}]$ where $\u_j$ comes from the column orthonormal bases of $\A$.}
	With probability greater than $1 - \delta$, we have 
	\[
	\eta = \max_{i \in [m]} \|\M_i -\M\|_2/\|\M\|_2 \le \epsilon.
	\]
\end{lem}

\paragraph{Error dependence.}
The choice of $\IM_T$ determines the frequency \LP communicates.
We explore the use of $\IM_T=\{0, p, 2p, \cdots, p\}$ and the decay strategy in experiments.
When $p = 1$, \texttt{LocalPower} reduces to \texttt{DPI}.
As a result, both the residual errors $\varPsi_t$ and $\varOmega_t$ vanish.
As shown in Lemma~\ref{lem:informal_dpm}, \DPI converges to zero error.
When $p \ge 2$, the error $\sin \theta_k$ typically increases with $p$ and is non-zero.
Corollary~\ref{col:error} depicts the relationship between the error and  problem-dependent parameters including $n, m, p$.
The proof is provided in Appendix~\ref{append:proof-coro}.
It can be proved by Theorem~\ref{thm:rho} and Lemma~\ref{lem:uniform}. 
\begin{col}
	\label{col:error}
	Under uniform sampling and assuming $s_i = \Theta(\frac{n}{m})$ and $n$ is sufficiently large, 
, with probability 1- $\delta$, \LP with OPT has an asymptotic error satisfying
	\[  \limsup_{t \to \infty} \sin\theta_k (\Z_T, \U_k) = \OM\left( h_p\left(\sqrt{\frac{m}{n}}\right)  \right), 
	\]
	where $h_p(x)$ is non-negative and increasing in (typically both $p$ and) $x$, and it satisfies $h_1(x) = 0$ as well as $0 \le h_p(x) \le C x$ for some $C$.
	We hide constants $\sigma_k, k, d, \rho, \kappa, \delta$ in the big-$\OM$ notation and $h_p(\cdot)$.
	However, with any decay strategy in which $p$ converges to $1$ finally, \LP achieves zero error asymptotically.
\end{col}

Corollary~\ref{col:error} says that when $p$ goes to infinity, the error is saturated and has a finite limit, because $h_p(\cdot)$ is  bounded.
The curve of error v.s. $p$ and $m$ in Figure~\ref{fig:sinp} validates the conclusion.
Indeed, the extreme case of super large $p$ means \texttt{LocalPower} reduces to the one-shot method, which has a non-zero optimization error typically.
Corollary~\ref{col:error} also reveals methods to reduce error.
To that end, we can (i) use the decay strategy ($p \downarrow$) to achieve arbitrary error or (ii) reduce the number of devices ($m \downarrow$) or collect more data points $n \uparrow$.
Both methods work in experiments empirically.

\paragraph{Dependence on $\sigma_{k} {-}\sigma_{k+1}$.}
Our result depends on $\sigma_k {-} \sigma_{k+1}$ even when $r>k$ where $r$ is the number of columns used in subspace iteration.
If we borrow the tool of \citet{balcan2016improved} rather than that of \citet{hardt2014noisy}, we can improve the result to a slightly milder dependency on  $\sigma_k{-}\sigma_{q+1}$, where $q$ is any intermediate integer between $k$ and $r$.
In particular, the required iteration $T$ will decrease from $\widetilde{\OM} \left(   \frac{\sigma_k}{\sigma_k - \sigma_{k+1}}\right)$ to $\widetilde{\OM} \left(   \frac{\sigma_k}{\sigma_k - \sigma_{q+1}}  \right)$.
It means using additional columns fastens convergence.
For a formal statement, please refer to Appendix~\ref{appen:gap}.

\paragraph{Further extensions.}
Our proposed \LP is simple, effective and well-grounded.
While we analyze it only on the centralized setting, \texttt{LocalPower} can be extended to broader settings, such as decentralized setting~\cite{gang2019fast} and streaming setting~\cite{raja2020distributed}.
To further reduce the communication complexity, we can combine \texttt{LocalPower} with sketching techniques~\cite{boutsidis2016optimal,balcan2016communication}.
For example, we could sketch each $\Y_t^{(i)}$ and communicate the compressed iterates to a central server in each iteration.
We leave the extensions to our future work.
Besides, in typical federated learning structures, real systems clients might not correspond to the central server due to connection failure.
It is also possible to consider partial participation of clients and the optimal way of client selection~\citep{reisizadeh2020fedpaq,chen2020optimal}.
\citet{guo2021privacy} makes an attempt towards the direction.

%We do not have theories for \texttt{LocalPower} equipped with sign fixing; it will be left for our future work.

\section{Conclusion}
We have developed a communication-efficient distributed algorithm named \texttt{LocalPower} to solve the truncated SVD.
Every worker machine performs multiple (say $p$) local power iterations between two consecutive communications.
We have theoretically shown that \texttt{LocalPower} converges $p$ times faster (in terms of communication) than the baseline distributed power iteration, if the residual error is uniformly small.
To reduce the residual error, we can (i) use OPT or sign-fixing, (ii) make use of a decay strategy that halves $p$ gradually, and (iii) increase local data size.
Both OPT and sign-fixing are more stable, while sign-fixing additionally is computationally efficient. 
The strategy is motivated by an experimental phenomenon that large $p$ often leads to a quick initial drop of loss but a higher final error.
The decay strategy obtains zero error asymptotically in theory and has better convergence performance in experiments.
We have conducted the thorough experiments to show the effectiveness of \texttt{LocalPower} and all the theories are agree with our empirical experiments.

\section*{Acknowledgement}
Li, Chen and Zhang have been supported by the National Key Research and Development Project of China (No. 2018AAA0101004 \& 2020AAA0104400),  and Beijing Academy of Artificial Intelligence (BAAI).

%\newpage
\bibliographystyle{icml2021}
\bibliography{bib/svd,bib/distributed,bib/optimization,bib/matrix,bib/system,bib/ml,bib/decentralized}

\newpage
\appendix

\begin{appendix}
	\onecolumn
	\begin{center}
		{\huge \textbf{Appendix}}
	\end{center}

\section{Proof for Section~\ref{sec:theory}}
\label{appen:proof}

\subsection{Angles Between Two Equidimensional Subspaces}
In this section, we introduce full definitions and lemmas on metrics between two subspaces, which will be useful in our following proof.

\paragraph{Principal Angles.} 
Given two matrices $\U, \tiU \in \OM_{d \times k}$ which are both full rank with $1 \le k \le d$, we define the $i$-th ($1 \le i \le k$) principal angle between $\U$ and $\tiU$ in a recursive manner:
\begin{equation}
\label{eq:theta}
\theta_i(\U, \tiU)  = \min \bigg\{  \arccos\left(  \frac{\x^\top\y}{\|\x\|\|\y\|} \right): \x \in \RM(\U), \y \in \RM(\tiU), \x \perp \x_j, \y \perp \y_j, \forall  j < i  \bigg\}
\end{equation}
where $\RM(\U)$ denotes by the space spanned by all columns of $\U$.
In this definition, we require that $0 \le \theta_1 \le \cdots \le \theta_k \le \frac{\pi}{2}$ and that $\{\x_1, \cdots, \x_k\}$ and $\{\y_1, \cdots, \y_k\}$ are the associated principal vectors.
Principal angles can be used to quantify the differences between two given subspaces.

We have following facts about the $k$-th principal angle between $\U$ and $\tiU$:
\begin{fact}
\label{fact}
Let $\U^{\perp}$ denote by the complement subspace of $\U$ (so that $[\U, \U^{\perp}] \in \RB^{d \times d}$ forms an orthonormal basis of $\RB^d$) and so dose $\tiU^{\perp}$,
	\begin{enumerate}
		\item $ \sin \theta_k(\U, \tiU) =  \| \U^\top \tiU^{\perp}\| = \| \tiU^\top \U^{\perp}\|$;
		\item $ \tan \theta_k(\U, \tiU) = \| \left[ (\U^{\perp})^{\top}\tiU\right] (\U^{\top} \tiU)^{\dagger}\| $ where $\dagger$ denotes by the Moore–Penrose inverse.
		\item For any reversible matrix $\R \in \RB^{k \times k}$, $ \tan \theta_k(\U, \tiU) = \tan \theta_k(\U, \tiU \R)$.
	\end{enumerate}
\end{fact}

 \paragraph{Projection Distance.} 
  Define the projection distance\footnote{Unlike the spectral norm or the Frobenius norm, the projection norm will not fall short of accounting for global orthonormal transformation. Check~\cite{ye2014distance} to find more information about distance between two spaces.} between two subspaces by 
\begin{equation}
\label{eq:dist}
\dist(\U, \tiU) = \| \U \U^\top - \tiU \tiU^\top\|.
\end{equation}
This metric has several equivalent expressions:
\[ \dist(\U, \tiU) = \| \U^\top \tiU^{\perp}\| = \| \tiU^\top \U^{\perp}\| =\sin \theta_k(\U, \tiU).   \]
More generally, for any two matrix $\A, \B \in \RB^{d \times k}$, we define the projection distance between them as
\[
\dist(\A, \B) = \| \U_\A \U_\A^\top - \U_\B \U_\B^\top\|
\]
where $\U_\A, \U_\B$ are the orthogonal basis of $\RM(\A)$ and $\RM(\B)$ respectively.

\paragraph{Orthogonal Procrustes.}
Let $\U, \tiU \in \RB^{d \times k}$ be two orthonormal matrices.
$\RM(\U)$ is close to $\RM(\tiU)$ does not necessarily imply $\U$ is close to $\tiU$, since any orthonormal invariant of $\U$ forms a base of $\RM(\U)$.
However, the converse is true.
If we try to map $\tiU$ to $\U$ using an orthogonal transformation, we arrive at the following optimization
\begin{equation}
\label{eq:best_O}
\O^* = \argmin_{\O \in \OM_k} \|\U -\tiU \O\|_F,
\end{equation}
where $\OM_k$ denotes the set of $k \times k$ orthogonal matrices.
The following lemma shows there is an interesting relationship between the subspace distance and their corresponding basis matrices.
It implies that as a metric on linear space, $\dist(\U, \tiU)$ is equivalent to $\|\U -\tiU \O^*\|_2$ (or $\min_{\O \in \OM_k} \|\U -\tiU \O\|_2$) up to some universal constant.
The optimization problem involved in is named as the orthogonal procrustes problem and has been well studied~\cite{schonemann1966generalized,cape2020orthogonal}.

\begin{lem}
	\label{lem:orth}
	Let $\U, \tiU \in \OM_{d \times k}$ and $\O^*$ is the solution of~\eqref{eq:best_O}.
	Then we have
	\begin{enumerate}
		\item $\O^*$ has a closed form given by $\O^* = \W_1\W_2^\top$ where $\tiU^\top\U = \W_1 \Si \W_2^\top$ is the singular value decomposition of $\tiU^\top\U$.
		\item Define $d(\U, \tiU) :=  \|\U -\tiU \O^*\|_2$ where $\|\cdot\|_2$ is the spectral norm.
		Then we have
		\[
		d(\U, \tiU) = \sqrt{2-2\sqrt{1- \dist(\U, \tiU)^2 }}  = 2\sin\frac{ \theta_k(\U, \tiU)}{2}.
		\]
		\item $d(\U_1, \U_2) = d(\U_2, \U_1)$ for any $\U_1, \U_2 \in \OM_{d \times k}$.
		\item $\dist(\U, \tiU) \le  d(\U, \tiU) \le \sqrt{2} \dist(\U, \tiU) $.
		\item Define 
		\[
		\ell(\U, \tiU) :=  \min_{\O \in \OM_k}\|\U -\tiU \O\|_2.
		\]
		Then $\ell(\U, \tiU)$ is a metric satisfying
		\begin{itemize}
			\item $\ell(\U, \tiU) \ge 0$ for all $\U, \tiU \in \OM_{d \times k}$. 
			$\ell(\U, \tiU) = 0$ if and only if $\RM(\U) = \RM(\tiU)$.
			\item $\ell(\U, \tiU) = \ell(\tiU, \U) $ for all $\U, \tiU \in \OM_{d \times k}$. 
			\item  $\ell(\U_1, \U_2) \le \ell(\U_1, \U_3) + \ell(\U_3, \U_2) $ for any $\U_1, \U_2$ and $\U_3 \in \OM_{d \times k}$. 
		\end{itemize}
		\item $\frac{1}{\sqrt{k}} \dist(\U, \tiU) \le \ell(\U, \tiU) \le  d(\U, \tiU) \le \sqrt{2} \dist(\U, \tiU)$. 
	\end{enumerate}
\end{lem}
\begin{proof}
	The first item comes from~\citet{schonemann1966generalized}.
	The second item comes from~\citet{cape2020orthogonal}.
	The third and forth items follow from the second one.
	The fifth item follows directly from definition.
	For the rightest two $\le$ of the last item, we use $\ell(\U, \tiU) \le d(\U, \tiU)$ and the forth item.
	For the leftest $\le$, we use $\min_{\O \in \OM_k} \|\U -\tiU \O\|_2 \ge \frac{1}{\sqrt{k}} \min_{\O \in \OM_k} \|\U -\tiU \O\|_F$ and $\min_{\O \in \OM_k} \|\U -\tiU \O\|_F \ge  \dist(\U, \tiU)$ (which is referred from Proposition 2.2 of~\citet{vu2013minimax}).
\end{proof}

\subsection{Proof Technique and Useful Lemmas}

% \begin{algorithm*}[!th]
% 	\caption{Distributed Local Power Iteration (\texttt{LocalPower}) } 
% 	\label{alg:local_power} 
% 	\begin{algorithmic}[1]
% 		\STATE {\bfseries Input:} distributed dataset $\{\A_i\}_{i=1}^m$, target rank $k$, iteration rank $r \ge k$, number of iterations $T$. 
% 		\STATE {\bfseries Initialization:} orthonormal $\Z_0^{(i)}=\Z_0 \in \RB^{d \times r}$ by QR on a random Gaussian matrix.
% 		\FOR{$t=1$ {\bfseries to} $T$}
% 		\STATE The $i$-th worker performs $\Y_t^{(i)} =\M_i \Z_{t-1}^{(i)}$ for all $i \in [m]$, where $\M_i = \frac{1}{s_i}  \A_i^\top \A_i$;
% 		\IF{$t \in \IM_T$}
% 		\STATE The server performs aggregation: $\Y_t =  \sum_{i=1}^m p_i \Y_{t}^{(i)}\D_{t}^{(i)}$, where $\D_{t}^{(i)}$ is given in~\eqref{eq:D0};
% 		\STATE Broadcast $\Y_t$ to the worker machines and let $\Y_t^{(i)}=\Y_t$ for all $i \in [m]$;
% 		\ENDIF
% 		\STATE The $i$-th worker independently performs orthogonalization: $\Z_t^{(i)} = \orth (\Y_t^{(i)})$, for all $i \in [m]$; 
% 		\ENDFOR
% 		\STATE {\bfseries Output:} approximated eigen-space $\Z_T = \Z_T^{(i)} \in \RB^{d \times r}$ with orthonormal columns.
% 	\end{algorithmic}
% \end{algorithm*}

\paragraph{Update Rule.}
Assume $1 = \argmax\limits_{i \in [m]} p_i$.
We overwrite $\Y_t^{(i)}$ when $t \in \IM_T$ (line 4 in Algorithm~\ref{alg:local_power}).
To distinguish the difference, we additionally use $\V_t^{(i)}$ to denote the updated but not communicated $\Y_t^{(i)}$.
Then the update rule becomes for all $i \in [m]$,
\begin{align}
\V_t^{(i)} &= \M_i \Z_{t-1}^{(i)};  \label{eq:V} \\
\Y_t^{(i)} &= 
\begin{cases}
\V_t^{(i)} & \text{if} \ t \notin \IM_T; \\
\sum_{i=1}^m p_i \V_t^{(i)} \D_t^{(i)}  & \text{if} \ t \in \IM_T.
\end{cases} \label{eq:Y} \\
\Y_t^{(i)} &= \Z_t^{(i)} \R_t^{(i)} \label{eq:Z}.
\end{align}
Here we abuse the notation a little bit and define $\D_t^{(i)}$ as
\begin{equation}
\label{eq:D0}
\D_t^{(i)} = 
\argmin\limits_{\D \in \FM \cap \OM_k} \| \Z_{t-1}^{(i)} \D - \Z_{t-1}^{(1)}\|_o
\end{equation}
where $\|\cdot\|_o$ can be set as either the Frobenius norm $\|\cdot\|_F$ or the spectrum norm  $\|\cdot\|_2$, though in the body text we use only $\|\cdot\|_F$.
There are some observations about the update rule:
\begin{enumerate}
\item If $t \notin \IM_T$, we have $\M_i  \Z_{t-1}^{(i)}= \V_t^{(i)} =  \Y_t^{(i)} = \Z_t^{(i)}\R_t^{(i)}$.
\item If $t \in \IM_T$, we have $\Y_t^{(1)} = \cdots = \Y_t^{(m)} =  \sum_{i=1}^m p_i \V_t^{(i)} \D_t^{(i)}  = \sum_{i=1}^m p_i\M_i  \Z_{t-1}^{(i)} \D_t^{(i)}$ and thus $\R_t^{(1)} = \cdots = \R_t^{(m)}$ and $\Z_t^{(1)} = \cdots = \Z_t^{(m)}$.
It implies that $\D_{t+1}^{(i)} = \I_k$.
\item If $\FM = \OM_k$, then $\D_t^{(i)}$ is the OPT we introduced in Section~\ref{sec:theory}.
If $\FM = \DM_k$, then $\D_t^{(i)}$ is the sign-fixing. 
If $\FM = \{\I_k\}$, then $\D_t^{(i)}$ is always equal to the identity matrix $\I_k$ and we arrive at the vanilla \texttt{LocalPower}.
The unified view helps us give theoretical analysis in a unified way. 
\end{enumerate}

\paragraph{Virtual Sequence.}
To analyze convergence of \texttt{LocalPower}, we define a virtual sequences defined as the weighted aggregation of local eigenvector matrices, i.e.,
\begin{equation}
\label{eq:virtual}
\bY_t = \sum_{i=1}^m p_i \Y_t^{(i)} \O_t^{(i)}.
\end{equation}
Here $\O_t^{(i)} \in \RB^{k \times k}$ is defined as
\begin{equation*}
\O_t^{(i)} = 
\begin{cases}
\I_k & \text{if} \ t \in \IM_T\\
\D_t^{(i)} & \text{if} \ t \notin \IM_T.
\end{cases}
\end{equation*}
If $t \in \IM_T$, $\bY_{t} = \Y_t^{(i)}$ for $i \in [m]$ and thus is obtainable.
Otherwise, $\bY_{t}$ is a shadow matrix facilitating analysis.

\paragraph{Recurrence Lemma.}
Lemma~\ref{lem:recur} shows that we can express $\bY_{t+1}$ as a linear transformation of $\bY_{t}$.
The resulting expression is similar to the iterates of the noisy power method proposed in~\cite{hardt2014noisy}, which motivates us to apply their technique to prove the main convergence of \texttt{LocalPower}.
Lemma~\ref{lem:recur} holds for any invertible $\R_t \in  \RB^{k \times k}$.
But, to guarantee convergence, we should carefully determine $\R_t$.
In Lemma~\ref{lem:R}, we will give a particular expression of $\R_t$, which plays a crucial role in helping us to bound the noise term $\G_t$.

\begin{lem}[Recurrence]
	\label{lem:recur}
	For any invertible $\R_t \in  \RB^{k \times k}$, we have
	\begin{equation}
	\label{eq:bYt}
	\bY_{t+1}  =  \left(\M \bY_{t}  + \G_t  \right) \R_t^{-1}
	\end{equation}
	where  $\M = \frac{1}{n} \A^\top \A \in \RB^{d \times d}$ and
	\begin{equation}
	\label{eq:G_t}
	\G_t =  \H_t + \W_t
	\end{equation}
	with $\H_t = \sum_{i=1}^m p_i \H_t^{(i)}$ and $ \W_t = \sum_{i=1}^m p_i \W_t^{(i)}$. Here for $i \in [m]$,
	\begin{equation}
	\label{eq:H_W}
	\H_t^{(i)} =  \left( \M_i -\M \right)  \Y_{t}^{(i)} \O_{t}^{(i)}\quad \text{and} \quad
	\W_t^{(i)} =\V_{t+1}^{(i)}   \left[ \D_{t+1}^{(i)} \R_t - \R_{t}^{(i)}\O_{t}^{(i)} \right].
	\end{equation}
\end{lem}

\begin{proof}
	First notice that we always have $\bY_{t} = \sum_{i=1}^m p_i \V_{t}^{(i)} \D_{t}^{(i)}$.
	If $t \in \IM_T$, $\Y_t^{(1)}=\cdots =\Y_t^{(m)}$ and $\O_t^{(i)} = \I_r$, implying the equation follows from~\eqref{eq:Y} and~\eqref{eq:virtual}.
	Otherwise, we have $\Y_t^{(i)} = \V_t^{(i)}$ and $\O_t^{(i)} = \D_t^{(i)}$, then $\bY_{t} = \sum_{i=1}^m p_i \Y_{t}^{(i)} \O_{t}^{(i)}=\sum_{i=1}^m p_i \V_{t}^{(i)} \D_{t}^{(i)}$.
	
	We always have $\V_{t+1}^{(i)} = \M_i \Z_t^{(i)}=\M_i \Y_{t}^{(i)} (\R_{t}^{(i)})^{-1}$.
	Then for any invertible $\R_{t}$, we have
	\begin{align*}
	\bY_{t+1}
	&=\sum_{i=1}^m p_i \V_{t+1}^{(i)} \D_{t+1}^{(i)}  \\
	&=\sum_{i=1}^m p_i  \M_i \Y_{t}^{(i)} (\R_{t}^{(i)})^{-1} \D_{t+1}^{(i)}  \\
	&=\sum_{i=1}^m  p_i \M_i \Y_{t}^{(i)} \O_{t}^{(i)} \R_{t}^{-1}  
	+  \sum_{i=1}^m p_i \M_i \Y_{t}^{(i)} (\R_{t}^{(i)})^{-1} \left[ \D_{t+1}^{(i)} \R_{t} - \R_{t}^{(i)}\O_{t}^{(i)} \right]\R_{t}^{-1} \\
	&\overset{(a)}{=} 
	\sum_{i=1}^m p_i \left( \M  \Y_{t}^{(i)} \O_{t}^{(i)}+ \H_t^{(i)} \right) \R_t^{-1} +
	\sum_{i=1}^m p_i \M_i\Z_{t}^{(i)}  
	\left[ \D_{t+1}^{(i)} \R_{t} - \R_{t}^{(i)}\O_{t}^{(i)} \right]\R_t^{-1} \\
	&=\sum_{i=1}^m p_i \left( \M \Y_{t}^{(i)} \O_{t}^{(i)}+ \H_t^{(i)} \right) \R_t^{-1} +\sum_{i=1}^m p_i \M_i \Z_{t}^{(i)}  
	\left[ \D_{t+1}^{(i)} \R_{t} - \R_{t}^{(i)}\O_{t}^{(i)} \right]\R_t^{-1} \\
	& \overset{(b)}{=}  \left(\M \bY_{t}+ \H_t + \W_t  \right) \R_t^{-1}  
	\end{align*}
	where (a) results from the definition of $\H_t^{(i)}$; and 
	(b) simplifies the equation via defining $\H_t = \sum_{i=1}^m p_i \H_t^{(i)}$ and $ \W_t = \sum_{i=1}^m p_i \W_t^{(i)}$.
	Setting $\G_t = \H_t + \W_t$ completes the proof.

\end{proof}

\paragraph{Convergence Lemma.}
The following lemma is an variant of Lemma 2.2 in~\citet{hardt2014noisy}.
Given the relation $\bY_{t+1}  =  \left(\M \bY_{t}  + \G_t  \right) \R_t^{-1}$,~\citet{hardt2014noisy} requires $\bY_{t}$ to have orthonormal columns, i.e., $\bY_{t}^\top \bY_{t} = \I_r$.
However, it is unlikely to hold in our case.
As a remedy, we slightly change the lemma to allow arbitrary $\bY_{t}$.
This will also change the condition on $\G_t$.

\begin{lem}
	\label{lem:error}
	Let $\U_k \in \RB^{d \times k}$ be the top-$k$ eigenvectors of a positive semi-definite matrix $\M$.
	For $t \ge 1$, assume $\bY_{t}$ satisfies~\eqref{eq:bYt} and $\G_t \in \RB^{d \times k}$ satisfy
	\begin{equation}
	\label{eq:UG_G}
		4\| \U_k^\top \G_t \bY_t^{\dagger} \|_2 \le (\sigma_k - \sigma_{k+1}) \cos \theta_k\left(\U_k, \bY_{t}\right) 
	\quad \text{and} \quad
	4\|\G_t\bY_t^{\dagger}\|_2\le (\sigma_k - \sigma_{k+1}) \epsilon
	\end{equation}
	where $ \bY_t^{\dagger} $ is the Moore–Penrose inverse of $\bY_{t}$ and $\epsilon < 1$. 
	Then
	\[  \tan \theta_k \left( \U_k, \bY_{t+1} \right)  \le \max \left( \epsilon, \max \left( \epsilon, \left(\frac{\sigma_{k+1}}{\sigma_k}\right)^{1/4} \right) \tan \theta_k \left( \U_k, \bY_{t} \right)  \right).  \]
\end{lem}
\begin{proof}
	Let $\bY_{t} = \bZ_t \bR_t$ be the QR factorization of $\bY_{t}$ so that $\bZ_t$ has orthonormal columns.
	The recurrence relation becomes
	$\bY_{t+1}  =  \left(\M \bZ_t \bR_t + \G_t  \right) \R_t^{-1}
	= \left(\M \bZ_t + \G_t \bR_t ^{-1} \right) \bR_t \R_t^{-1}$.
	By the fact~\ref{fact}, we have $\tan \theta_k \left( \U_k, \bY_{t+1} \right) = \tan \theta_k \left( \U_k, \bY_{t+1} \R_t \bR_t^{-1}\right) = \tan \theta_k \left( \U_k, \M \bZ_{t} + \G_t\bR_t^{-1} \right)$.
	By requiring
	\[
		4\| \U_k^\top \G_t \bR_t^{-1}\|_2 \le (\sigma_k - \sigma_{k+1}) \cos \theta_k\left(\U_k, \bZ_{t}\right) 
	\quad \text{and} \quad
	4\|\G_t\bR_t^{-1}\|_2\le (\sigma_k - \sigma_{k+1}) \epsilon,
	\]
	we have from Lemma 2.2 in~\citet{hardt2014noisy} that 
		\[  \tan \theta_k \left( \U_k, \bY_{t+1} \right)  \le \max \left( \epsilon, \max \left( \epsilon, \left(\frac{\sigma_{k+1}}{\sigma_k}\right)^{1/4} \right) \tan \theta_k \left( \U_k, \bZ_{t} \right)  \right).  \]
	Noting that $\RM(\bY_{t}) = \RM(\bZ_{t})$, we have $\theta_k \left( \U_k, \bY_{t} \right) = \theta_k \left( \U_k, \bZ_{t} \right)$ and thus
	\[
	\cos\theta_k \left( \U_k, \bY_{t} \right) = \cos\theta_k \left( \U_k, \bZ_{t} \right)
	\quad \text{and} \quad
	\tan\theta_k \left( \U_k, \bY_{t} \right) = \tan\theta_k \left( \U_k, \bZ_{t} \right).
	\]
	Finally, using $\| \U_k^\top \G_t \bY_t^{\dagger} \|_2 = \| \U_k^\top \G_t \bR_t^{-1} \|_2$ and $\| \G_t \bY_t^{\dagger} \|_2 = \|  \G_t \bR_t^{-1} \|_2$ completes the proof.

\end{proof}

\paragraph{Other Useful Lemma.}
Lemma~\ref{lem:initial} handles $\tan \theta_k(\U, \Z_0)$ with randomly generate $\Z_0$, while Lemma~\ref{lem:bound-bYt} give a upper bound of $\|\bY_{t}^\dagger \M\|_2$.

\begin{lem}[Lemma 2.4 in~\citet{hardt2014noisy}]
	\label{lem:initial}
	For an arbitrary orthonormal $\U$ and random subspace $\Z_0 \in \RB^{d \times r}$, with probability grater than $1 - \tau^{-\Omega(r+1-k)} - e^{-\Omega(d)}$, we have that 
	\[  \tan \theta_k(\U, \Z_0) \le \frac{\tau \sqrt{d}}{\sqrt{r} - \sqrt{k-1}}.  \]
\end{lem}

\begin{lem}
	\label{lem:bound-bYt}
	Recall that $\kappa = \|\M\|_2\|\M^\dagger\|_2$ and 
	$\eta = \max_{i \in [m]}\|\M_i - \M \|_2/\|\M\|_2$.
	Define
	\begin{equation*}
	\mu_{t} =  1 - \eta{\kappa} - \sum_{j=1}^mp_j\| \Z_{t-1}^{(j)} \D_{t}^{(j)} - \Z_{t-1}^{(1)}\|_2
	\end{equation*}
	and assume $\mu_{t} > 0$.
	Then it follows that $\|\bY_{t}^\dagger \M\|_2 \le  \frac{1}{\mu_{t}}$.
% 	\[
% 	\|  \bY_{t}^\dagger \M\|_2 \le  \frac{1}{\mu_{t}}
% 	\quad \text{and} \quad
% 	\max_{i \in [m]}\|\bY_{t}^\dagger\M_i\|_2  \le  \frac{2}{\mu_{t}}.
% 	\]
\end{lem}
\begin{proof}
	For any matrix $\X \in \RB^{d \times k}$, we have
	\[
	\|\X^\dagger\|_2 = \max_{ \x \in \RB^k}   \frac{\|\w\|_2}{\|\X \w \|_2 }
	=\max_{ \|\X \w \|_2 = 1}  \|\w\|_2
	= \max\{ \|\w\|_2 :  \|\X \w \|_2 \le 1  \}.
	\]
	Notice that $\bY_{t}^\dagger \M= (\M^{\dagger}\bY_{t})^\dagger$ and $ \bY_{t} = \sum_{j=1}^m p_j \M_j \Z_{t-1}^{(j)} \D_t^{(j)}$.
	We then have
	\begin{align*}
	\|\bY_{t}^\dagger \M\|_2
	&= \|(\M^{\dagger}\bY_{t})^\dagger\|_2 \\
	&= \max\{ \|\w\|_2 :  \|\M^{\dagger}\bY_{t} \w \|_2 \le 1  \} \\
	&=\max\{ \|\w\|_2 :  \|(\M^{\dagger} \sum_{j=1}^m p_j \M_j \Z_{t-1}^{(j)} \D_t^{(j)} )\w \|_2 \le 1  \} \\
	&\overset{(a)}{\le}\max\{ \|\w\|_2 :  \| \sum_{j=1}^m p_j \Z_{t-1}^{(j)} \D_t^{(j)} \w \|_2  - \eta \kappa \|\w\|_2 \le 1 \} \\
	&\overset{(b)}{\le}\max\{ \|\w\|_2 :  \|\w\|_2 (1-\eta\kappa-\sum_{j=1}^m p_j\| \Z_{t-1}^{(j)} \D_{t}^{(j)} - \Z_{t-1}^{(1)}\|_2 ) \le 1 \}\\
	&\le \frac{1}{1-\eta\kappa-\sum_{j=1}^m p_j\| \Z_{t-1}^{(j)} \D_{t}^{(j)} - \Z_{t-1}^{(1)}\|_2} \le \frac{1}{\mu_t}
	\end{align*}
	where (a) follows because of
	\[
	\|(\M^{\dagger} \sum_{j=1}^m p_j \M_j \Z_{t-1}^{(j)} \D_t^{(j)} )\w \|_2 \ge \| \sum_{j=1}^m p_j \Z_{t-1}^{(j)} \D_t^{(j)} \w \|_2  -  \sum_{i=1}^m p_i \|\M^{\dagger}(\M_j-\M)\|_2 \|\Z_{t-1}^{(j)} \D_t^{(j)} \w\|_2
	\]
	and $\|\M^{\dagger}(\M_j-\M)\|_2 \le \|\M^{\dagger}\|_2\|(\M_j-\M)\|_2 \le \eta \kappa$; and (b) holds since
% 	and $\|\M^{\dagger}(\M_j-\M)\|_2 \le \eta$; and (b) holds since
	\begin{align*}
		\| \sum_{j=1}^m p_j \Z_{t-1}^{(j)} \D_t^{(j)} \w \|_2
		&\ge \| \sum_{j=1}^m p_j \Z_{t-1}^{(1)}\w \|_2 - \| \sum_{j=1}^m p_j (\Z_{t-1}^{(j)} \D_t^{(j)} - \Z_{t-1}^{(1)}) \w \|_2 \\
		&\ge \|\w\|_2 -\sum_{j=1}^m p_j  \| \Z_{t-1}^{(j)} \D_{t}^{(j)} - \Z_{t-1}^{(1)}\|_2 \|\w \|_2 \\
		&= \|\w\|_2 (1-\sum_{j=1}^m p_j\| \Z_{t-1}^{(j)} \D_{t}^{(j)} - \Z_{t-1}^{(1)}\|_2 ).
	\end{align*}
% 	The last inequality holds since $\kappa \le \widetilde{\kappa}$.
% 	Finally, by a similar argument, for any $i \in [m]$,
% 	\[
% 	\|  \bY_{t}^\dagger \M_i\|_2
% 	\le \frac{1}{1-2\eta\widetilde{\kappa}-\sum_{j=1}^m p_j\| \Z_{t-1}^{(j)} \D_{t}^{(j)} - \Z_{t-1}^{(1)}\|_2} 
% 	\le \frac{2}{\mu_t}
% 	\]
% 	The last inequality follows since we require $\mu_{t} > 0$ which implies $\sum_{j=1}^m p_j\| \Z_{t-1}^{(j)} \D_{t}^{(j)} - \Z_{t-1}^{(1)}\|_2 < 1$.
\end{proof}

\subsection{The Choice of $\R_t$}
In this section, we specify the choice of $\R_t$ and analyze the residual error bound $  \| \D_{t+1}^{(i)} \R_t - \R_{t}^{(i)}\O_{t}^{(i)} \|_2$.
Lemma~\ref{lem:R} specifies the way we set $\R_t$.
Given a baseline data matrix $\M_o$, $\R_t$ is the shadow matrix that depicts what the upper triangle matrix ought to be, if we start from the nearest synchronized matrix and perform QR factorization using the matrix $\M_o$.
We will set $\M_o = \M_t^{(1)}$ (by assuming $1 = \argmax\limits_{i \in [m]} p_i$) and analyze $\|\W_{t}^{(i)}\bY_{t}^\dagger\|_2$ and $\|\H_{t}^{(i)}\bY_{t}^\dagger\|_2$ in terms of $ \|\Z_t^{(i)} \D_{t+1}^{(i)} - \Z_t^{(1)}\|_2$.
Latter we will bound $ \|\Z_t^{(i)} \D_{t+1}^{(i)} - \Z_t^{(1)}\|_2$ when $\FM$ is differently set.

\begin{lem}[Choice of $\R_t$]
	\label{lem:R}
	Fix any $t$ and let $t_0 = \tau(t) \in \IM_T$ be the latest synchronization step before $t$, then $t \ge \tau(t)$.
	\begin{itemize}
		\item 	If $t = t_0$, we define $\R_t = \R_t^{(i)}$ for any $i \in [m]$ since all $\R_t^{(i)}$'s are equal.
		\item If $t > t_0$, given a baseline data matrix $\M_o$, we define $\R_t \in \RB^{r \times r}$ recursively as the following.
		Let $\Y_{t_0} = \bY_{t_0} =  \Z_{t_0} \R_{t_0}$, and for $l = t_0, t_0+1, \cdots, t$, we use the following QR factorization to define $\R_t$'s:  
		\[
		\V_{l+1} = \M_o \Z_{l} = \Z_{l+1} \R_{l+1}.
		\]
	\end{itemize}
	Then for any $i \in [m]$, we have
	\begin{equation}
	\label{eq:D-bound}
	\|\D_{t+1}^{(i)} \R_t - \R_{t}^{(i)}\O_{t}^{(i)}\|_2 
	\le \sigma_1(\M_o) \| \Z_t^{(i)}\D_{t+1}^{(i)}-  \Z_t \|_2+\left[ \|\M_o - \M_i\|_2  + \sigma_1(\M_i)\|\Z_{t-1}^{(i)}\D_{t}^{(i)}  -\Z_{t-1} \|_2\right]1_{t \notin \IM_T}.
	\end{equation}
\end{lem}

\begin{proof}
	We are going to bound $\|\D_{t+1}^{(i)} \R_t - \R_{t}^{(i)}\O_{t}^{(i)} \|_2$ in two cases depending on whether $t \in \IM_T$.
	If $t \in \IM_T$, implying $t=t_0 :=\tau(u)$, then $\O_t^{(i)} = \D_{t+1}^{(i)}= \I_r$ and $\R_t = \R_t^{(i)}$.
	Therefore, $\D_{t+1}^{(i)} \R_t - \R_{t}^{(i)}\O_{t}^{(i)}  = \0$.
	
	Otherwise, $t \notin \IM_T$ and thus $t > t_0$.
	Let's fix some $i \in [m]$ and denote $ \Delta \M = \M_i - \M_o$.
	Based on \texttt{LocalPower}, we have $\Y_{t_0}^{(i)} = \bY_{t_0} =  \Z_{t_0}^{(i)} \R_{t_0}^{(i)}, $ and  for $l = t_0, t_0+1, \cdots, t$, 
	\[
	 \V_{l+1}^{(i)} = \M_i \Z_{l}^{(i)} =  \Z_{l+1}^{(i)} \R_{t+1}^{(i)}.
	\]
	Then, 
	\begin{align*}
\Z_{l}^{(i)} \R_{l}^{(i)}\O_{l}^{(i)}  
	& = \M_i  \Z_{l-1}^{(i)}\O_{l}^{(i)}\\
	&= (\M_o + \Delta \M)( \Z_{l-1} + \Delta \Z_{l-1} )\\
	&= \M_o\Z_{l-1} + \Delta \M  \cdot \Z_{l-1}  + \M_i \cdot \Delta \Z_{l-1}\\
	&:= \M_o \Z_{l-1}+ \E_{l-1} = \Z_l \R_l+  \E_{l-1} 
	\end{align*}
	where $\E_{l-1} =  \Delta \M  \cdot \Z_{l-1}  + \M_i \cdot \Delta \Z_{l-1}$ and $\Delta \Z_{l-1}=\Z_{l-1}^{(i)}\O_{l}^{(i)}  -\Z_{l-1} $.
	
	Note that 
	\begin{gather*}
	\Z_t^{(i)}\R_t^{(i)}\O_{t}^{(i)}= \Z_t \R_t+  \E_{t-1}.
	\end{gather*}
	Then we have
	\begin{align*}
	\|\D_{t+1}^{(i)} \R_t - \R_{t}^{(i)}\O_{t}^{(i)} \|_2
	&= 	\|\Z_t^{(i)}\D_{t+1}^{(i)} \R_t - \Z_t^{(i)}\R_{t}^{(i)}\O_{t}^{(i)} \|_2\\
	&\overset{(a)}{=}	\|\Z_t^{(i)}\D_{t+1}^{(i)} \R_t -  \Z_t \R_t-  \E_{t-1} \|_2\\
	&\le	\| (\Z_t^{(i)}\D_{t+1}^{(i)}-  \Z_t ) \R_t \|_2 + \| \E_{t-1}\|_2 \\
	&\overset{(b)}{\le}	\| \Z_t^{(i)}\D_{t+1}^{(i)}-  \Z_t \|_2 \| \R_t \|_2 + \|\Delta \M\|_2  + \|\M_i\|_2\|\Z_{t-1}^{(i)}\O_{t}^{(i)}  -\Z_{t-1} \|_2\\
		&\overset{(c)}{\le}	 \sigma_1(\M_o) \| \Z_t^{(i)}\D_{t+1}^{(i)}-  \Z_t \|_2+ \|\M_o - \M_i\|_2  + \sigma_1(\M_i)\|\Z_{t-1}^{(i)}\D_{t}^{(i)}  -\Z_{t-1} \|_2
	\end{align*}
	where (a) uses the equality of $\Z_t^{(i)}\R_t^{(i)}\O_{t}^{(i)}$; (b) uses the definition of $\E_{t-1}$ and $\O_t^{(i)} = \D_t^{(i)}$ (due to $t \notin \IM_T$); and (c) uses $\|\R_t\|_2 \le \| \M_o\|_2=\sigma_1(\M_o)$.
	
	Combining the two cases, we have for all $t \in [T]$,
	\[
	\|\D_{t+1}^{(i)} \R_t - \R_{t}^{(i)}\O_{t}^{(i)}\|_2 
	\le \sigma_1(\M_o) \| \Z_t^{(i)}\D_{t+1}^{(i)}-  \Z_t \|_2+\left[ \|\M_o - \M_i\|_2  + \sigma_1(\M_i)\|\Z_{t-1}^{(i)}\D_{t}^{(i)}  -\Z_{t-1} \|_2\right]1_{t \notin \IM_T}.
	\]
\end{proof}

\begin{lem}
	\label{eq:error}
Assume $\eta = \max_{i \in [m]}\|\M_i - \M\|_2/\|\M\|_2$ is sufficiently small and $1 = \argmax\limits_{i \in [m]} p_i$.
Define 
\[
\rho_t = \|\Z_t^{(i)} \D_{t+1}^{(i)} - \Z_t^{(1)}\|_2,
\]
we have
\begin{gather}
\|\H_t\bY_{t}^\dagger\|_2 \le  \frac{2\sigma_1 \eta \kappa 1_{t \notin \IM_T}}{1-\eta{\kappa}-(1-\max_{i \in [m]}p_i)\rho_{t-1}}  \label{eq:error-H}\\
\|\W_t\bY_{t}^\dagger\|_2
\le 4(1-\max_{i \in [m]}p_i)\sigma_1 \kappa\frac{\rho_t +( \rho_{t-1} + \eta) 1_{t \notin \IM_T}   }{1-\eta{\kappa} -(1-\max_{i \in [m]}p_i)\rho_{t-1}} \label{eq:error-W}.
\end{gather}
\end{lem}
\begin{proof}
Without loss of generality, we assume $1 = \argmax\limits_{i \in [m]} p_i$ and then set the baseline matrix in Lemma~\ref{lem:R} as $\M_o = \M_1$ and use the $\R_t$ defined therein.
Then Lemma~\ref{lem:R} and Lemma~\ref{lem:Z-Z} imply for all $i \in [m]$,
\begin{align*}
\|\D_{t+1}^{(i)} \R_t - \R_{t}^{(i)}\O_{t}^{(i)}\|_2 
&\le \sigma_1(\M_o) \| \Z_t^{(i)}\D_{t+1}^{(i)}-  \Z_t \|_2+ [\|\M_o - \M_i\|_2  + \sigma_1(\M_i)\|\Z_{t-1}^{(i)}\D_{t}^{(i)}  -\Z_{t-1} \|_2]1_{t \notin \IM_T}\\
&\le (1+\eta)\sigma_1 \left[\rho_{t} + \rho_{t-1}1_{t \notin \IM_T}  \right]+ \eta \sigma_1 1_{i \neq 1 \ \text{and}  \ t \notin \IM_T} 
\end{align*}
where $\sigma_1 = \sigma_1(\M)$ and $1_{i \neq 1 \ \text{and}  \ t \notin \IM_T}$ is the indicator of event $\{ i \neq 1 \} \cap  \{t \notin \IM_T\}$.

Recall the definition of $\rho_{t}$.
By Lemma~\ref{lem:R} and Lemma~\ref{lem:bound-bYt}, we have
\begin{align*}
\|\W_t \bY_t^\dagger\|_2
&=\| \sum_{i=1}^m p_i \M_i \Z_{t}^{(i)}   \left[ \D_{t+1}^{(i)} \R_t - \R_{t}^{(i)}\O_{t}^{(i)} \right]\bY_{t}^\dagger \M  \M^{-1}\|_2\\
&\le \sum_{i=1}^m p_i \|\M^{-1}\|_2\|\M_i\|_2 \| \bY_{t}^\dagger \M \|_2 \|\D_{t+1}^{(i)} \R_t - \R_{t}^{(i)}\O_{t}^{(i)} \|_2\\
&\le 2(1-p_1)\sigma_1 \kappa \frac{\eta 1_{t \notin \IM_T} + 2(\rho_t + \rho_{t-1}1_{t \notin \IM_T}) }{1-\eta{\kappa} -(1-p_1)\rho_{t-1}}\\
&\le 4(1-p_1)\sigma_1\kappa \frac{\rho_t +( \rho_{t-1} + \eta) 1_{t \notin \IM_T}   }{1-\eta{\kappa} -(1-p_1)\rho_{t-1}}.
\end{align*}
Similarly,
\begin{align*}
\|\H_t\bY_{t}^\dagger\|_2 
&= 	\|\sum_{i=1}^m p_i \left( \M_i -\M \right)  \Y_{t}^{(i)} \O_{t}^{(i)} \bY_{t}^\dagger \M \M^{-1}\|_2 \\
&\le \sum_{i=1}^m p_i  \|\M^{-1} \|_2\|(\M_i -\M)\|_2 \| \Y_{t}^{(i)} \O_{t}^{(i)} \|_2 	\|\bY_{t}^\dagger \M\|_2 1_{t \notin \IM_T} \\
&\le   \frac{(1+\eta)\sigma_1 \kappa\eta 1_{t \notin \IM_T}}{1-\eta{\kappa} -(1-p_1)\rho_{t-1}}\\
&\le   \frac{2\sigma_1 \kappa\eta 1_{t \notin \IM_T}}{1-\eta{\kappa} -(1-p_1)\rho_{t-1}}.
\end{align*}
\end{proof}

\subsubsection{The Case When $\FM = \OM_k$}
\begin{lem}
	\label{lem:Z-Z}
	When setting $\FM = \OM_{k}$, no mater $\D_t^{(i)}$ is solved from~\eqref{eq:D0} using $\|\cdot\|_F$ or $\|\cdot\|_2$, we have 
	\begin{equation}
	\|\Z_{t-1}^{i}\D_t^{(i)} - \Z_{t-1}^{(1)}\|_2  \le \sqrt{2}\dist(\Z_{t-1}^{(i)}, \Z_{t-1}^{(1)}).
	\end{equation}
\end{lem}
\begin{proof}
	This follows directly from Lemma~\ref{lem:orth}.
\end{proof}

\begin{lem}[Davis-Kahan $\sin(\theta)$ theorem]
	\label{lem:DK}
	Let the top-$k$ eigenspace of $\M$ and $\widetilde{\M}$ be respectively $\U_k$ and $\tiU_{k}$ (both of which are orthonormal).
	The $k$-largest eigenvalue of $\M$ is denoted by $\sigma_{k}(\M)$ and similarly for $\sigma_{k}(\widetilde{\M})$.
	Define $\delta_k = \min \{ |\sigma_k(\M) - \sigma_{j}(\widetilde{\M})| :j \ge k +1 \}$, then
	\[
	\dist(\U_k, \tiU_{k}) = \sin \theta_k(\U_k, \tiU_{k}) \le \frac{\| \M - \widetilde{\M}\|_2}{\delta_k}.
	\]
\end{lem}

\begin{lem}[Perturbation theorem of projection distance]
	\label{lem:PT-PJ}
	\label{lem:PT}
	Let $\mathrm{rank}(\X) = \mathrm{rank}(\Y)$, then
	\begin{equation*}
	\dist(\X, \Y) \le \min\{ \|\X^\dagger\|_2, \|\Y^\dagger\|_2 \} \|\X-\Y\|_2.
	\end{equation*}
\end{lem}
\begin{proof}
See Theorem 2.3 of~\citet{ji1987perturbation}.
\end{proof}

\begin{lem}
	\label{lem:error-F-OMk}
Assume $\eta = \max_{i \in [m]}\|\M_i - \M \|_2/\|\M\|_2$ is sufficiently small.
If $\D_t^{(i)}$ is solved from~\eqref{eq:D0} with $\FM = \OM_{k}$, then~\eqref{eq:error-H} and~\eqref{eq:error-W} hold with
\[
\rho_t
\le \min \sqrt{2}\left\{\frac{2\kappa^p p \eta (1+\eta)^{p-1} }{(1-\eta)^p} , \frac{\eta\sigma_1}{\delta_k} + 2  \gamma_k^{p/4}  \max_{i \in [m]}\tan \theta_k(\Z_{\tau(t)}, \U_{k}^{(i)})  \right\}.
\]
where 
\begin{itemize}
	\item $\delta_k = \min\limits_{i \in [m]}\delta_k^{(i)}$ with $\delta_k^{(i)} = \min \{ |\sigma_k(\M) - \sigma_{j}({\M_i})|: j \ge k +1 \}$;
	\item$\gamma_k = \max\{\max\limits_{i \in [m]}\frac{\sigma_{k+1}(\M_i)}{\sigma_{k}(\M_i)}, \frac{\sigma_{k+1}(\M)}{\sigma_{k}(\M)}\} \in (0, 1)$;
	\item $ \kappa = \|\M\|_2\|\M^\dagger\|_2$ is the condition number of $\M$;
	\item $p = t -\tau(u)$, $\tau(t) \in \IM_T$ is defined as the nearest synchronization time before $t$.
\end{itemize}

\end{lem}
\begin{proof}
	By Lemma~\ref{lem:error} and Lemma~\ref{lem:Z-Z}, we only need to bound $\max\limits_{i \in [m]}\dist(\Z_{t}^{(i)}, \Z_{t}^{(1)})$.
	We will bound each $\dist(\Z_{t}^{(i)}, \Z_{t}^{(1)})$ uniformly in two ways.
	Then the minimum of the two upper bounds holds for their maximum that is exactly $\rho_{t}$.
	 
		Fix any $i \in [m]$ and $t \in [T]$. 
		Let $\tau(t)$ be the latest synchronization step before $t$ and $p=t-\tau(t)$ be the number of nearest local updates.
	\begin{itemize}
		\item For small $p$, by Lemma~\ref{lem:PT-PJ}, it follows that 
		\begin{align*}
		\dist(\Z_t^{i}, \Z_t^{(1)})
		&= \dist(\M_i^p\Z_{\tau(t)}, \M_1^p\Z_{\tau(t)})\\
		&\le \dist(\M_i^p\Z_{\tau(t)}, \M^p\Z_{\tau(t)}) 
		+ \dist(\M^p\Z_{\tau(t)}, \M_1^p\Z_{\tau(t)})\\
		&\le  \min\{ \|(\M_i^p\Z_{\tau(t)})^\dagger\|_2, \|(\M^p\Z_{\tau(t)})^\dagger\|_2 \} \|(\M_i^p-\M^p)\Z_{\tau(t)}\|_2\\
		& \qquad + \min\{ \|(\M^p\Z_{\tau(t)})^\dagger\|_2, \|(\M_1^p\Z_{\tau(t)})^\dagger\|_2 \} \|(\M^p-\M_1^p)\Z_{\tau(t)}\|_2\\
		&\le 2\kappa^p \frac{(1+\eta)^p-1}{(1-\eta)^p}\\
		&\le \frac{2\kappa^p p \eta (1+\eta)^{p-1} }{(1-\eta)^p}
		\end{align*}
		where $\kappa = \|\M\|_2\|\M^\dagger\|_2$ is the condition number of $\M$.
		\item  For large $p$, let the top-$k$ eigenspace of $\M_1$ and $\M_i$ be respectively $\U_k^{(1)}$ and $\U_{k}^{(i)}$ (both of which are orthonormal).
		The $k$-largest eigenvalue of $\M$ is denoted by $\sigma_{k}(\M_1)$ and similarly for $\sigma_{k}(\M_i)$.
		Then by Lemma~\ref{lem:DK}, we have 	\[
		\dist(\U_k, \U_{k}^{(i)}) 
		\le \frac{\|\M_i - \M\| }{\delta_k^{(i)}}
		\le \frac{\eta\sigma_1}{\delta_k^{(i)}}.
		\]
		where $\sigma_1 = \sigma_1(\M)$ and $\delta_k^{(i)} = \min \{ |\sigma_j(\M_i) - \sigma_{k}(\M)| :j \neq k \}$.
		
		Note that local updates are equivalent to noiseless power method.
		Then, using Lemma~\ref{lem:error} and setting $\epsilon = 0$ and $\G_t = \0$ therein, we have
		\[
		\tan \theta_k(\Z_t^{i}, \U_{k}^{(i)})  \le 
		\left( \frac{\sigma_{k+1}(\M_i)}{\sigma_{k}(\M_i)} \right)^{1/4}   
		\tan \theta_k(\Z_{t-1}^{i}, \U_{k}^{(i)}).
		\]
		
		Hence,
		\begin{align*}
		\dist(\Z_t^{i}, \Z_t^{(1)})
		&\le \dist(\Z_t^{i}, \U_k^{(i)}) +\dist(\U_k^{(i)}, \U_k^{(1)}) +\dist(\U_k^{(1)}, \Z_t^{(1)}) \\
		&\le  \frac{\eta\sigma_1}{\delta_k^{(i)}} +  \left(   \frac{\sigma_{k+1}(\M_i)}{\sigma_{k}(\M_i)} \right)^{p/4}  	\tan \theta_k(\Z_{\tau(t)}, \U_{k}^{(i)})  + 	\left(   \frac{\sigma_{k+1}(\M)}{\sigma_{k}(\M)} \right)^{p/4}   \tan \theta_k(\Z_{\tau(t)}, \U_{k}^{(1)}) \\
		&\le  \frac{\eta\sigma_1}{\min_{i \in [m]}\delta_k^{(i)}} + 2 \gamma_k^{p/4}  \max_{i \in [m]}\tan \theta_k(\Z_{\tau(t)}, \U_{k}^{(i)}).
		\end{align*}
	\end{itemize}
   Combining the two cases, we have
   \[
   \rho_t
   \le \sqrt{2}\min \left\{\frac{2\kappa^p p \eta (1+\eta)^{p-1} }{(1-\eta)^p} , \frac{\eta\sigma_1}{\delta_k} + 2 \gamma_k^{p/4}  \max_{i \in [m]}\tan \theta_k(\Z_{\tau(t)}, \U_{k}^{(i)})  \right\}.
   \]
\end{proof}

\subsubsection{The Case When $\FM = \{ \I_k \}$}
When $\FM$ is only a singleton containing only $\I_k$, it is equivalent to set $ \D_t^{(i)} = \I_r$ for all $t \in [T]$ and $i \in [m]$.
In this case, the virtual sequence is actually a pure average: $\bY_t = \sum_{i=1}^m p_i \V_t^{(i)}$.

\begin{lem}
	\label{lem:diff_r}
	Let $\A \in \RB^{d \times k}$ with $d \ge  k$ be any matrix with full rank.
	Denote by its QR factorization as $\A = \Q \R$ where $\Q$ is an orthgonal metrix.
	Let $\E$ be some perturbation matrix and $\A + \E = \tiQ \tiR$ the resulting QR factorization of $\A + \E$.
	When $\|\E\|_2 \|\A^{\dagger}\|_2 < 1$, $\A + \E$ is of full rank.
	What's more, it follows that
	\[    \| \tiQ - \Q\|_2
	\le
	\sqrt{2k}\frac{\|\A^\dagger\|_2\|\E\|_2}{1-\|\A^\dagger\|_2\|\E\|_2}. \]
\end{lem}
\begin{proof}
	Actually, we have
		\[    \| \tiQ - \Q\|_F \overset{(a)}{\le} 
	\frac{\sqrt{2}\|\E\|_F}{\|\E\|_2} \ln \frac{1}{1-\|\A^\dagger\|_2\|\E\|_2}
	\overset{(b)}{\le} 
	\sqrt{2}\frac{\|\A^\dagger\|_2\|\E\|_F}{1-\|\A^\dagger\|_2\|\E\|_2}
	\overset{(c)}{\le} 
	\sqrt{2k}\frac{\|\A^\dagger\|_2\|\E\|_2}{1-\|\A^\dagger\|_2\|\E\|_2} \]
	where (a) comes from Theorem 5.1 in~\citet{sun1995perturbation};
	(b) uses $\ln(1+x) \le x$ for all $x > -1$; 
	and (c) uses $\|\E\|_F \le \sqrt{k}\|\E\|_2$.
\end{proof}

\begin{lem}
	\label{lem:error-F-Ik}
	Let $\eta = \max_{i \in [m]}\|\M_i - \M \|_2/\|\M\|_2$ be sufficiently small.
	If $\D_t^{(i)}$ is solved from~\eqref{eq:D0} with $\FM = \{ \I_k \}$, then~\eqref{eq:error-H} and~\eqref{eq:error-W} hold with
	\[
	\rho_t \le 4\sqrt{2k}p \kappa^p \eta (1+\eta)^{p-1}
	\]
where $\kappa = \|\M\|_2\|\M^\dagger\|_2$ is the condition number of $\M$, $p = t -\tau(u)$, $\tau(t) \in \IM_T$ is defined as the nearest synchronization time before $t$.
\end{lem}

\begin{proof}
	By Lemma~\ref{lem:error}, we are going to bound $\rho_{t} = \max\limits_{i \in [m]}\|\Z^{(i)}-\Z_t^{(1)}\|_2$.
	Fix any $i \in [m]$ and $t \in [T]$.
	We will bound $\|\Z^{(i)}-\Z_t^{(1)}\|_2$ uniformly so that the bound holds for their maximum.
	
	Fix any $i \in [m]$ and $t \in [T]$.
	Let $\tau(t)$ be the latest synchronization step before $t$ and $p=t-\tau(t)$ be the number of nearest local updates.
	Note that $\Z_{t}^{(i)}$ and $\Z_{t}^{(1)}$ are the $Q$-factor of the QR factorization of $\M_i^p\Z_{\tau(t)}$ and $\M_1^p\Z_{\tau(t)}$.
	Let $\Z_t$ be the  $Q$-factor of the QR factorization of $\M^p\Z_{\tau(t)}$.
	Then Lemma~\ref{lem:diff_r} yields
	\[
	\|\Z_{t}^{(i)}- \Z_{t}\|_2 
	\le  \sqrt{2k} \frac{\|(\M^p\Z_{\tau(t)})^\dagger\|_2\|(\M_i^p-\M^p)\Z_{\tau(t)} \|_2 }{1- \|(\M^p\Z_{\tau(t)})^\dagger\|_2\|(\M_i^p-\M^p)\Z_{\tau(t)} \|_2}
	:= \sqrt{2k}  \frac{\omega}{1-\omega}
	\]
	where $\omega = \|(\M^p\Z_{\tau(t)})^\dagger\|_2\|(\M_i^p-\M^p)\Z_{\tau(t)} \|_2$ for short.
	If $\omega \le 1/2$, then we have $\|\Z_{t}^{(i)}- \Z_{t}\|_2  \le 2\sqrt{2k} \omega$.
	Otherwise, we have $\omega \ge 1/2$ and $\|\Z_{t}^{(i)}- \Z_{t}\|_2  \le 2 \le \sqrt{2k} \le 2\sqrt{2k} \omega$.
	Then we have for all $i \in [m]$,
	\[
		\|\Z_{t}^{(i)}- \Z_{t}\|_2  \le 2\sqrt{2k} \|(\M^p\Z_{\tau(t)})^\dagger\|_2\|(\M_i^p-\M^p)\Z_{\tau(t)} \|_2.
	\]
	
	Hence,
	\begin{align*}
		\rho_t &=  \|\Z_{t}^{(i)}- \Z_{t}^{(1)}\|_2 \\
		&\le \|\Z_{t}^{(i)}- \Z_{t}\|_2  + \|\Z_{t}- \Z_{t}^{(1)}\|_2 \\
		&\le 2\sqrt{2k}  \left[ \|(\M^p\Z_{\tau(t)})^\dagger\|_2\|(\M_i^p-\M^p)\Z_{\tau(t)} \|_2 +
		\|(\M^p\Z_{\tau(t)})^\dagger\|_2\|(\M_1^p-\M^p)\Z_{\tau(t)} \|_2 \right]\\
		&\le 4\sqrt{2k} \kappa^p \left[ (1+\eta)^p -1 \right]\\
			&\le 4\sqrt{2k}p \kappa^p \eta (1+\eta)^{p-1} 
	\end{align*}
			where $\kappa = \|\M\|_2\|\M^\dagger\|_2$ is the condition number of $\M$.
\end{proof}

\subsection{Proof of Theorem~\ref{thm:main} and Theorem~\ref{thm:rho}}
\begin{proof}
	We provide a proof in four steps. 
	
	\paragraph{First step: Perturbed iterate analysis.}
	Recall that we defined a virtual sequence by 
	\[ \bY_t  = \sum_{i=1}^m p_i \Y_t^{(i)}\O_t^{(i)}.  \]
	Notice that this sequence never has to be computed explicitly, it is just a virtual sequence we use in the analysis.
	 From Lemma~\ref{lem:recur}, we construct the iteration of the virtual sequence $\{\bY_{t}\}$ as
	\[\bY_{t+1}  =  \left(\M \bY_{t}  + \G_t  \right) \R_t^{-1}\]
	where $\M = \frac{1}{n} \A^\top \A \in \RB^{d \times d}$, $\G_t$ is the noise term inccured by the variance among different nodes, and $\R_t$ is chosen according to Lemma~\ref{lem:R}.
	Recall that $\G_t=\H_t + \W_t$ is given in~\eqref{eq:G_t} with $\H_t = \sum_{i=1}^m p_i \H_t^{(i)}$ and $ \W_t = \sum_{i=1}^m p_i \W_t^{(i)}$.

	\paragraph{Second step: Bound the noise term $\G_t$.}
	Let $p = \gap(\IM_T)$ denotes by the longest interval between subsequent synchronization steps.
	In order to guarantee convergence, we should make sure the noise term $\G_t$ is small enough.
	In particular, we require
	\begin{equation}
	\label{eq:constaint}
	\| \G_t\bY_{t}^\dagger\|_2 \le\frac{ \sigma_k - \sigma_{k+1}}{5} \min \left( \frac{\sqrt{r} - \sqrt{k-1}}{\tau \sqrt{d}}, \epsilon   \right)  
	\end{equation}
	By Lemma~\ref{lem:error-F-OMk} or~\ref{lem:error-F-Ik}, we always have
	\begin{gather*}
	\|\H_t\bY_{t}^\dagger\|_2 \le  \frac{2\sigma_1 \kappa \eta 1_{t \notin \IM_T}}{1-\eta{\kappa} -(1-\max_{i \in [m]}p_i)\rho_{t-1}} \\
	\|\W_t\bY_{t}^\dagger\|_2
	\le 4(1-\max_{i \in [m]}p_i)\sigma_1 \kappa\frac{\rho_t +( \rho_{t-1} + \eta)1_{t \notin \IM_T}  }{1-\eta{\kappa} -(1-\max_{i \in [m]}p_i)\rho_{t-1}} 
	\end{gather*}
	We assume $\eta {\kappa} \le 1/3$ and additionally assume $(1-\max_{i \in [m]}p_i)\rho_{t-1} \le \frac{1}{3}$.
	Then the last two inequalities become
	\begin{gather*}
	\|\H_t\bY_{t}^\dagger\|_2 \le 6\sigma_1 \kappa \eta 1_{t \notin \IM_T}   := 6\sigma_1\kappa \varPsi_t \\ \\
	\|\W_t\bY_{t}^\dagger\|_2
	\le 12(1-\max_{i \in [m]}p_i)\sigma_1 \kappa \left[\rho_t +( \rho_{t-1} + \eta)1_{t \notin \IM_T}  \right] := 12\sigma_1 \kappa \varOmega_t
	\end{gather*}	
	Then in order to ensure~\eqref{eq:constaint}, we only need to ensure
	\[
	6\sigma_1\varPsi_t + 12\sigma_1 \varOmega_t
	\le \frac{ \sigma_k - \sigma_{k+1}}{5 \kappa } \min \left( \frac{\sqrt{r} - \sqrt{k-1}}{\tau \sqrt{d}}, \epsilon   \right).  
	\]
A sufficient condition to that is 
	\begin{equation}
	\label{eq:error_consition_append}
		\varPsi_t +\varOmega_t
		\le \frac{1}{60}  \frac{ \sigma_k - \sigma_{k+1}}{\sigma_1 \kappa} \min \left( \frac{\sqrt{r} - \sqrt{k-1}}{\tau \sqrt{d}}, \epsilon   \right) = \OM(\epsilon_0).  
	\end{equation}
	Finally, we argue that the condition $(1-\max_{i \in [m]}p_i)\rho_{t-1} \le \frac{1}{3}$ is indicated in the uniform boundedness of~\eqref{eq:error_consition_append} (i.e.,~\eqref{eq:error_consition_append} holds for all $t \in [T]$).
	This is because
	\[
	(1-\max_{i \in [m]}p_i)\rho_{t-1}  
	\le  \varOmega_{t-1} \le \varPsi_{t-1} +\varOmega_{t-1} \le  \frac{\epsilon_0}{60} < \frac{1}{3}.
	\]
	
	\paragraph{Third step: Bound $\rho_{t}$.}
	Let $\kappa = \|\M\|_2\|\M^\dagger\|_2$ be the condition number of $\M$ and $p = t -\tau(u)$ with $\tau(t) \in \IM_T$ defined as the nearest synchronization time before $t$.
	Then, we can prove Theorem~\ref{thm:rho} now.
	\begin{itemize}
		\item If $\FM = \OM_{k}$, then 
		\[
		\rho_t
		\le \sqrt{2} \min\left\{\frac{2\kappa^ p \eta (1+\eta)^{p-1} }{(1-\eta)^p} , \frac{\eta\sigma_1}{\delta_k} + 2 \gamma_k^{p/4}  \max_{i \in [m]}\tan \theta_k(\Z_{\tau(t)}, \U_{k}^{(i)})  \right\}.
		\]
		with the parameters $\delta_k, \gamma_k$ given in Lemma~\ref{lem:error-F-OMk}.
		By requiring $\eta \le 1/p$, we have
		$\frac{ (1+\eta)^{p-1}}{(1-\eta)^p} \le \frac{ (1+1/p)^{p-1}}{(1-1/p)^p} \le \mathrm{e}^2$.
		Define $C_t = \max_{i \in [m]}\tan \theta_k(\Z_{\tau(t)}, \U_{k}^{(i)})$. 
		Latter we will show that since \LP converges under Assumption~\ref{assum:error_condition}, then $\lim\limits_{t \to \infty} \sin \theta_k(\Z_{\tau(t)}, \U_{k}) \le \epsilon$.
		Then, we have
		\begin{align*}
		  \limsup\limits_{t \to \infty} C_t 
		  &= 
		 \limsup\limits_{t \to \infty} \max_{i \in [m]}  \tan \theta_k(\Z_{\tau(t)}, \U_{k}^{(i)} )\\
		 &=\limsup\limits_{t \to \infty} \max_{i \in [m]} \tan \arg\sin \sin\theta_k(\Z_{\tau(t)}, \U_{k}^{(i)} )\\
		 &\le\limsup\limits_{t \to \infty} \max_{i \in [m]} \tan \arg\sin (\sin\theta_k(\Z_{\tau(t)}, \U_{k})+ \sin\theta_k(\U_{k}, \U_{k}^{(i)} ))\\
		&\le\max_{i \in [m]}  \tan \arg\sin (\frac{\eta\sigma_1}{\delta_k} + \epsilon)
		= \OM(\eta+\epsilon).
		\end{align*}
		It can be seen that when $p$ is sufficiently large, $\rho_t = \OM(\eta)$ which is independent with $p$.
		\item If $\FM = \{\I_k\}$, then
		\[
		\rho_t \le 4\sqrt{2k}p \kappa^p \eta (1+\eta)^{p-1}
		\le 4\mathrm{e}\sqrt{2k}p \kappa^p \eta.
		\]
	\end{itemize}
   
	Simply put together, if $\varPsi + \varOmega \le \epsilon_0$, we can firmly ensure~\eqref{eq:constaint} holds.

	\paragraph{Forth step: Establish convergence.}
	Let's first assume ~\eqref{eq:UG_G} holds.
	With~\eqref{eq:UG_G}, the following argument is quite similar to~\citet{hardt2014noisy}.
	Note that 
	Specifically, we will see that at every step $t$ of the algorithm,
	\[\tan \theta_k(\U_k, \bY_t) \le \max\left( \epsilon, \tan \theta_k(\U_k, \Z_0) \right),\]
	which implies for $\epsilon \le \frac{1}{2}$ that
	\[ \cos \theta_k(\U_k, \bZ_t) \ge \min \left(  1 - \epsilon^2/2, \cos \theta_k(\U_k, \Z_0) \right) \ge \frac{7}{8} \cos \theta_k(\U_k, \Z_0)  \]
	so Lemma~\ref{lem:error} applies at every step. 
	This means that
	\[\tan \theta_k(\U_k, \bY_{t+1}) \le \max\left( \epsilon, \delta \tan \theta_k(\U_k, \bY_{t}) \right)\]
	for $\delta = \max(\epsilon, ({\sigma_{k+1}}/{\sigma_{k}})^{1/4})$. 
	After $T \ge \log_{1/\delta} \frac{\tan \theta_k(\U_k, \Z_0)}{\epsilon}$ steps, the tangent will reach the accuracy $\epsilon$ and remain there. 
	So we have
	\[  \| ( \I-\Z_{T} \Z_{T}^{\top}) \U\| 
	= \sin \theta_k(\U_k, \bY_T) 
	\le   \tan \theta_k(\U_k, \bY_T) \le 
	 \epsilon.  \]
	 Plus the observation that 
	 \[ \log(1/\delta) 
	 \ge c \min (\log(1/\epsilon), \log(\sigma_{k}/\sigma_{k+1})) 
	 \ge c \min\left(1, \log\frac{1}{1-\gamma}\right) 
	 \ge  c\min (1, \gamma) =  c \gamma  \]
	 where $\gamma = 1 - \sigma_{k+1}/\sigma_{k}$ and $c = \frac{1}{4}$, we can set $T \in \IM_T$ and
	 \[ T =  \Omega\left(  \frac{\sigma_{k}}{\sigma_{k}-\sigma_{k+1}} \log (d \tau / \varepsilon) \right). \]
	 
	 Finally we are going to show that once the noise term $\G_t$ is bounded as~\eqref{eq:constaint},~\eqref{eq:UG_G} would naturally hold.
	 From Lemma~\ref{lem:initial}, we have 
	 	\[  \tan \theta_k(\U, \Z_0) \le \frac{\tau \sqrt{d}}{\sqrt{r} - \sqrt{k-1}}  \]
	 with all but $\tau^{-\Omega(p+1-k)}+e^{-\Omega(d)}$ probability.
	 Hence 
	 \[  \cos \theta_k(\U, \Z_0) \ge \frac{1}{1 + \tan \theta_k(\U, \Z_0)} \ge \frac{\sqrt{r}-\sqrt{k-1}}{2\tau\sqrt{d}}.  \]
\end{proof}

\subsection{Proof of Corollary~\ref{col:error}}
\label{append:proof-coro}
\begin{proof}
From Theorem~\ref{thm:main}, our algorithm has error no larger than $\epsilon$.
We then find the minimum $\epsilon$ that is a function of $m, n, p$ by combining Theorem~\ref{thm:rho} and Lemma~\ref{lem:uniform}.
For a fixed $n/m$, Lemma~\ref{lem:uniform} bounds $\eta$ in terms of $s$ or equivalently $n/m$, implies $\eta = \sqrt{\frac{3 \mu \rho}{s_i}\log \big( \frac{\rho m}{{\delta}} \big)} =
\widetilde{\Theta}(\sqrt{\frac{\mu\rho}{s}})= \widetilde{\Theta}(\sqrt{\frac{m\mu\rho}{n}})$.
For sufficiently small $\epsilon$, we have $\epsilon = \frac{\sigma_1 \kappa }{\sigma_k-\sigma_{k+1}} \epsilon_0$.
Let $\epsilon$ be sufficiently small such that~\eqref{eq:error_cond} just holds.
Then we have
\[
\epsilon = \Theta(\epsilon_0) = \Theta(\eta + \sup_{t}(\rho_t + \rho_{t-1}) ) = \OM(h_p(\eta)+\eta)
\]
where the last equality follows from Theorem~\ref{thm:rho} which bounds $\rho_t$ in terms of $\eta$.
It is in the form of $\rho_t \le h_p(\eta)$ where $h_1(\cdot) = 0$ and $h_p(\eta)$ typically increases in $p$ and $\eta$.
With OPT, $h_p(\eta) = O(\eta)$, while without OPT, $h_p(\eta) = \OM(\sqrt{k}p \kappa^p \eta).$

If we use any decay strategy in which $p$ converges to $1$ finally, then \LP is reduced to \DPI finally and thus of course achieves zero error asymptotically.
\end{proof}

\end{appendix}

\section{Statistical Error Between the Empirical Matrix and the Population One}
\label{appen:stat_error}
Recall that $\M = \frac{1}{n} \A \A^\top = \frac{1}{n} \sum_{i=1}^n \x_i \x_i^\top$ is the empirical correlation matrix and $\M_{*} = \EB_{\x \sim \DM} \x \x^\top$ is the population one.
By Matrix Hoeffding theorem, we can bound $\|\M - \M_{*}\|$ in terms of samples.
\begin{lem}[Matrix Hoeffding inequality~\citet{tropp2012user}]
Let $\DM$ be a distribution over vectors with squared $\ell_2$
norm at most $b$.
Let $\M_{*} = \EB_{\x \sim \DM} \x \x^\top$ and $\M = \frac{1}{n} \sum_{i=1}^n \x_i \x_i^\top$ where $\x_1, \cdots, \x_n $ are sampled
i.i.d. from $\DM$, then it holds that
\[
\PB\left(   \| \M_{*} - \M \| \ge t  \right) \le d \cdot \exp\left( - \frac{t^2 n}{16b^2}  \right).
\]
\end{lem}

Let the top-$k$ eigenspace of $\M$ and $\M_{*}$ be respectively $\V_k$ and $\V_{k, *}$ (both of which are orthonormal).
Let $\widehat{\V}$ be any estimated top-$k$ eigenvector matrix (for example, $\Z_T$ produced by \texttt{LocalPower}).
If we care how accurately $\widehat{\V}$ approximate $\V_{k, *}$, by the triangle inequality, 
\[
\dist(\widehat{\V}, \V_{k, *})
\le \underbrace{ \dist(\widehat{\V}, \V_{k})}_{\text{optimization \ error }}
+ \underbrace{ \dist(\V_{k}, \V_{k,*})}_{\text{statistical \ error }}
\]
% The item of accuracy in Table~\ref{table:1} actually denotes by $\dist(\widehat{\V}, \V_{k, *})$.
Theorem~\ref{thm:main} characterizes the diminishing speed of the optimization term, however, has nothing to do with the statistical error.
The latter is controlled by the available samples through the combination of the Davis-Kahan $\sin(\theta)$ theorem (Lemma~\ref{lem:DK}) and $\| \M_{*} - \M \|$.
In particular, with probability greater than $1-\delta$, the statistical error is no larger than
\[
\frac{1}{\delta_k} 4b\sqrt{ \frac{\ln \frac{d}{\delta}}{n}}.
\]
If only a single machine attends the training, $n=s$, while if $m$ machines cooperate, $n = ms$.
From the last inequality, the statistical error is reduced by a factor of $\sqrt{m}$.

\section{Dependence on $\sigma_k - \sigma_{k+1}$}
\label{appen:gap}
Our result depends on $\sigma_k - \sigma_{k+1}$ even when $r>k$ where $r$ is the number of columns used in subspace iteration.
This is mainly because we borrow tools from~\citet{hardt2014noisy} to prove the theory.
In the analysis of~\citet{hardt2014noisy}, the required iteration depends on the consecutive eigengap $\sigma_k - \sigma_{k+1}$ even when $r>k$ where $r$ is the number of columns used in subspace iteration.
Note that $\sigma_k - \sigma_{k+1}$ can be unimaginably small in practical large-scale problems.
\citet{balcan2016improved} improved the result to a slightly milder dependency on  $\sigma_k-\sigma_{q+1}$ by proposing a novel characterization measuring the discrepancy between the running rank-$r$ subspace $\Z_t$ and target top-$k$ eigenspace $\U_k$, where $q$ is any intermediate integer between $k$ and $r$.
If we borrow the idea from the improved analysis of \citet{balcan2016improved}, we can refine the result.
In that case, the needed computation rounds will depend on $\sigma_k-\sigma_{q+1}$ as a result.
All the above discussion can be easily parallel.

\begin{theorem}
   Let Assumption~\ref{assum:error_condition} hold with sufficiently small $\eta \kappa\le\frac{1}{3}$ where $\kappa = \|\M\|\|\M^{\dagger}\|$ is the condition number of $\M$.
    Let Assumption~\ref{assum:error_condition} holds with $\tau > 0$ and the following $\epsilon_0$
    \[
		\epsilon_0 = \frac{\sqrt{r} - \sqrt{q-1}}{\tau \sqrt{d}}
		 \min \left\{ \frac{ \sigma_k - \sigma_{q+1}}{\sigma_1}\epsilon, \frac{\sigma_q}{\sigma_1}   \right\}.
	\]
	Let $k \le q \le r$.
	If we borrow the refined analysis in~\citet{balcan2016improved}, then for sufficiently small
	$\epsilon$ satisfying
	\[ \epsilon = 
	\OM\left( \frac{\sigma_q}{\sigma_k} \cdot \min \bigg\{  \frac{1}{\log(\sigma_k/\sigma_q)}, \frac{1}{\log(\tau d)} \bigg\}     \right),  \]
	when
	\[ T = \Omega \left(   \frac{\sigma_k}{\sigma_k - \sigma_{q+1}} \log\left(\frac{\tau d}{\epsilon}\right)  \right) \]
	after $ | \IM_T|$ rounds of communication, with probability at least $1 - \tau^{-\Omega(r+1-q)} - e^{-\Omega(d)}$, we have
	\[  \dist(\Z_T, \U_k) = \sin\theta_k (\Z_T, \U_k) = \| \left( \I_d - \Z_T\Z_T^\top \right)\U_k\| \le \epsilon.   \]
\end{theorem}
\begin{proof}
We use Corollary A.1 in~\citet{balcan2016improved} instead of Lemma~\ref{lem:error} in the third step of the proof of Theorem~\ref{thm:main}.
\end{proof}

\section{Related Work}
\label{sec:related}
Truncated SVD or principal component analysis (PCA) is one of the most important and popular techniques in data analysis and machine learning.
A multitude of researches focus on iterative algorithms such as power iterations or its variants~\citep{golub2012matrix,saad2011numerical}.
These deterministic algorithms inevitably depends on the spectral gap, which can be quite large in large scale problems.
Another branch of algorithm seek alternatives in stochastic and incremental algorithms~\cite{oja1985stochastic,arora2013stochastic,shamir2015stochastic,shamir2016convergence,de2018accelerated}.
Some work could achieve eigengap-free convergence rate and low-iteration-complexity~\citep{musco2015randomized,shamir2016convergence,allen2016lazysvd}.

Large-scale problems necessitate cooperation among multiple worker nodes to overcome the obstacles of data storage and heavy computation.
For a review of distributed algorithms for PCA, one could refer to~\citet{wu2018review}.
One feasible approach is divide-and-conquer algorithm which performs a one-shot averaging of the individual top-$k$ eigenvectors (or subspace) returned by worker nodes~\citep{garber2017communication,fan2019distributed,bhaskara2019distributed,charisopoulos2020communication}.
In particular, the concurrent work~\citep{charisopoulos2020communication} proposes to average local eigenvector matrices via OPT as ours, though they focus on one-shot scenario and obtain better error analysis.
The divide-and-conquer algorithms have only one round of communication. 
To reach a certain accuracy, it often requires that the per-machine sample size $s$ to grow with the number of machines $m$~\citep{garber2017communication}, which means it is only effective in large local dataset regime.

% \citet{bhaskara2019distributed} analyzed a variant of distributed averaging approach, which has a better sample complexity and eigenvalue-dependent bound.

Another line of results for distributed eigenspace estimation uses iterative algorithms that perform multiple communication rounds.
They require much smaller sample size and can often achieve arbitrary accuracy.
For example, in our work, we only require the per-machine sample size $s$ depends on $m$ in a very mild way like $\OM(\ln m)$, however,~\citet{garber2017communication} requires $s = \widetilde{\OM}(m)$ to reach a comparable result.
Some works make use of shift-and-invert framework (S\&I) for PCA~\citep{garber2015fast,garber2016faster,allen2016lazysvd}.
S\&I methods turn the problem of computing the leading eigenvector to that of approximately solving a small system of linear equations.
This, in turn, could be solved by arbitrary convex solvers~\citep{xu2018gradient}, and, therefore, can be extended in distributed settings naturally.
~\citet{garber2017communication} coupled S\&I methods with a distributed  first-order convex solver, giving guarantees in terms of communication costs.
~\citet{gang2019fast} turns the problem of distributed PCA into a constraint optimization problem (by letting each device hold a independent parameter and adding a constraint that all local parameter should be same), and then uses gradient-based methods to solve it iteratively. 
~\citet{chen2021distributed} combined S\&I methods with a distributed approximate Newton method where the communication cost is saved by only using the Hessian information on the first machine.
Very recently,~\citet{grammenos2019federated} proposed a federated, asynchronous, and differential privacy algorithm for distributed PCA. Methodologically, the algorithm is not power-iteration-based. 
Instead, their algorithm incrementally computes local model updates using streaming procedure and adaptively estimates its leading principle components. 
In particular, they assume the clients are arranged in a tree-like structure, while we did not make such assumption.

Recently, the technique of local updates emerges as a simple but powerful tool in distributed empirical risk minimization~\citep{mcmahan2017communication,zhou2017convergence,stich2018local,wang2018cooperative,yu2019parallel,li2019convergence,li2019communication,khaled2019first}.
Distributed algorithms with local updates typically alternate between local computation and periodical communication.
Therefore, local updates allow less frequent communication but incur more computation due to the inevitably accumulated residual errors.
This paper uses local updates for the distributed power iteration.
However, our analysis is totally different from the local SGD algorithms \citep{zhou2017convergence,stich2018local,wang2018cooperative,yu2019parallel,li2019convergence,li2019communication,khaled2019first}.
A main challenge in analyzing \texttt{LocalPower} is that the local SGD algorithms for empirical risk minimization often involve an explicit form of (stochastic) gradients.
For SVD or PCA, a canonical example of non-convex problems, the gradient cannot be explicitly expressed, so the existing techniques cannot be applied~\citep{simchowitz2017gap}.
Instead, we borrowed tools from the noisy power method \citep{hardt2014noisy,balcan2016improved} and carefully analyze the residual errors.

In our paper, we only consider the centralize PCA, where there is a server connecting all other nodes.
However, the technique of local updates can also be used in other settings like decentralized or streaming PCA~\citep{gang2019fast,raja2020distributed}.

\section{Experiments}
\label{appen:exp_all}

\subsection{Experimental Settings}
\label{appen:exp_setting}
We conduct experiments to demonstrate the communication efficiency of \texttt{LocalPower}.
We use several datasets from the LIBSVM website and summarize them in Table~\ref{tab:dataset}.
Our focus is the needed communication round required to minimize the optimization error and analyze \texttt{LocalPower} through different lens.
For comparison, we consider the following baselines:
\begin{enumerate}
	\item Weighted Distributed Averaging~\cite{bhaskara2019distributed};
	\item Unweighted Distributed Averaging~\cite{fan2019distributed};
	\item Distributed Randomized SVD.
	\item Distributed power Iteration (the case of \LP when $p=1$)
\end{enumerate}
For completeness, we include the former three algorithms in the next subsection.
We also study the effect of different choice of $m$, $p$, and the decay strategy.
Throughout, we use either $\IM_T = \{0, p ,2p, \cdots, T\}$ or the decay strategy.

\paragraph{Preprocessing.}
The data are randomly shuffled and partitioned among $m$ nodes.
We scale each feature by dividing it by the maximum value of each coordinate, so that each feature locates between $[-1, 1]$.
In particular, we will first find the maximum value for each feature coordinate among all workers in the system and share it with all participants.
All the experiments use the same initialization $\Z_0 \in \RB^{d\times r}$ (for any $r> k$) which contains a set of randomly generated orthonormal bases.

\paragraph{Experimental.}
All the experiments are conducted on a single machine.
We fix the target rank to $k=5$.
We plot $\text{dist}(\Z_t, \U_k) = \|(1-\Z_t\Z_t^\top)\U_k\| = \sin\theta_k (\Z_t, \U_k) $ against the number of communications to evaluate communication efficiency.
In Table~\ref{tab:dataset}, we list the information of $(n, d)$ for the datasets we use, all satisfying $n \ll d$. 
Though we focus on large $n$ regime, latter we also test large $d$ regimes namely $n \approx d$ for completeness.
In Table~\ref{table:2}, we estimate $\eta$ by $\max_{i \in [m]}\|\M_i - \M\|_2/\|\M\|_2$.
Under uniform sampling, when we fix $n$, the larger $m$ (equals to smaller $s$), the larger $\eta$.

\begin{table}[!ht]
	\caption{A summary of used data sets from the LIBSVM website.}
	\label{tab:dataset}
	\centering
		\begin{tabular}{lcc|lcc}
			\toprule
			Data set & $n$ & $d$ & 	Data set & $n$ & $d$ \\
			\midrule
			A9a & 32561 & 123 &
			Abalone & 2114 & 8\\
			Acoustic & 78823 & 50&
			Aloi & 108000  & 128 \\
			Combined & 78823 &100 &
			Connect-4 & 7990 & 125\\
			Covtype & 581,012 & 54 &
			Housing & 506 & 13\\
			Ijcnn1 & 49990 & 22&
			MNIST & 60,000 & 780 \\
			Poker & 25010 & 10 &
			Space-ga & 3107 & 6  \\
				Splice & 1000 & 24&
				W8a & 49749 & 300\\
			MSD & 463,715 & 90 &  &  & \\
			\bottomrule
		\end{tabular}
\end{table}

\begin{table}[!ht]
	\caption{
		The value of $\eta$ under uniform partitions on fifteen datasets. 
		In the following experiments, we uniformly distribute $n$ samples into $m = \max(\lfloor\frac{n}{1000}\rfloor, 3)$ so that each device has about 1000 samples. 
		It implies $m$ ranges from 20 to 100, which is the range we consider here.
		To fill the following table, we distributed $n$ samples into $m$ devices and estimate it by $\eta = \max_{i \in [m]}\|\M -\M_i\|_2/\|\M\|_2$.
		It can be seen that for a fixed $n$, the larger $m$, the larger $\eta$.
}
	\label{table:2}
	\centering 
	\begin{tabular}{c|c|c|c|c|c}
		\toprule
		Dataset & $m=20$ & $m=40$ & $m=60$  &$m=80$ & $m=100$ \\
		\midrule
		A9a&0.034 & 0.0563 &0.0701& 0.0906& 0.0998\\
		Abalone&0.1089 &0.23 &  0.2458 &0.2629 &0.3556\\
		Acoustic&0.0063 &0.0107 &0.0134 &0.0179 &0.0199\\
		Aloi&0.0479 &0.0659 &0.1023 &0.1162 &0.203 \\
		Combined&0.006 & 0.0089 &0.0113 &0.014 & 0.0158\\
		Connect-4&0.0376& 0.054 & 0.0771 &0.0791 &0.0899\\
		Covtype&0.0078 &0.011 & 0.0159 &0.0164& 0.0202\\
		Housing&0.3117 &0.3747& 0.5062 &0.6442 &0.6741\\
		Ijcnn1&0.016 & 0.0288 &0.0348 &0.0363& 0.0489\\
		MNIST&0.0396 &0.0584 &0.0689& 0.0896& 0.0904\\
		Poker&0.0369 &0.0519 &0.0702 &0.0803 &0.0904\\
		Space-ga&0.0855 &0.1317& 0.1495 &0.2111 &0.3446\\
		Splice&0.1627 &0.2484 &0.3154 &0.3957 &0.4717\\
		W8a&0.1046 &0.1664 &0.1937& 0.2515 &0.3167\\
		MSD&0.0007& 0.0009 &0.0012& 0.0014 &0.0015\\
		\bottomrule
	\end{tabular}
\end{table}

\newpage
\subsection{One-shot Baseline Algorithms}
\label{appen:baseline}

\begin{algorithm*}[!ht]
	\caption{Unweighted Distributed Averaging (\texttt{UDA})~\cite{fan2019distributed}} 
	\label{alg:unweited_distributed} 
	\begin{algorithmic}[1]
		\STATE {\bfseries Input:} distributed dataset $\{\A_i\}_{i=1}^m$ with $\A_i \in \RB^{s_i \times d}$, target rank $k$.
		\STATE  {\bfseries Local:} Each device computes the rank-$k$ SVD of $\M_i = \frac{1}{s_i} \A_i^\top \A_i$ as $\widehat{\V}_i \Si_i \widehat{\V}_i^\top$ with $\Si_i  \in \RB^{k \times k}$ and $\widehat{\V}_i \in \RB^{d \times k}$.
		\STATE  {\bfseries Server:} The central server computes $\widetilde{\M} = \frac{1}{m} \sum_{i=1}^n \widehat{\V}_i \widehat{\V}_i^\top$, then output the top $k$
eigenvalues and the corresponding eigenvectors of $\widetilde{\M}$.
	\end{algorithmic}
\end{algorithm*}

\begin{algorithm*}[!th]
	\caption{Weighted Distributed Averaging (\texttt{WDA})~\cite{bhaskara2019distributed}} 
	\label{alg:weited_distributed} 
	\begin{algorithmic}[1]
		\STATE {\bfseries Input:} distributed dataset $\{\A_i\}_{i=1}^m$ with $\A_i \in \RB^{s_i \times d}$, target rank $k$.
		\STATE  {\bfseries Local:} Each device computes the rank-$k$ SVD of $\M_i = \frac{1}{s_i} \A_i^\top \A_i$ as $\widehat{\V}_i \Si_i \widehat{\V}_i^\top$ with $\Si_i  \in \RB^{k \times k}$ and $\widehat{\V}_i \in \RB^{d \times k}$.
		\STATE  {\bfseries Server:} The central server computes $\widetilde{\M} = \frac{1}{m} \sum_{i=1}^n \widehat{\V}_i \Si_i \widehat{\V}_i^\top$, then output the top $k$
		eigenvalues and the corresponding eigenvectors of $\widetilde{\M}$.
	\end{algorithmic}
\end{algorithm*}

\begin{algorithm*}[!th]
	\caption{Distributed Randomized SVD (\texttt{DR-SVD}) (A distributed variant of Randomized SVD in~\citet{halko2011finding})} 
	\label{alg:local_rSVD} 
	\begin{algorithmic}[1]
		\STATE {\bfseries Input:} distributed dataset $\{\A_i\}_{i=1}^m$, $\A = [\A_1^\top, \cdots, \A_m^\top]^\top \in \RB^{n \times d}$ with target rank $k$, $\A_i \in \RB^{s_i \times d}$ and $r = k + \lfloor \frac{d-k}{4}\rfloor$.
		\STATE The server generates a $d \times r$ random Gaussian matrix $\bf{\Omega}$;
		\STATE The server learns $\Y = \A \A^\top \A\bf{\Omega}$ and obtains an orthonormal $\Q \in \RB^{n \times r}$ by QR decomposition on $\Y$;
		\STATE Let $\Q = [\Q_1^\top, \cdots, \Q_m^\top]^\top$ with $\Q_i \in \RB^{s_i \times r}$ and each worker receives $\Q_i$;
		\STATE The $i$-th worker computes $\B_i = \Q_i^\top \A_i \in \RB^{r \times d}$ for all $i \in [m]$;
		\STATE The server aggregate $\B = \sum_{i=1}^m \B_i = \Q^\top \A$ and perform SVD: $\B = \widetilde{\U} \widehat{\Si}  \widehat{\V}^T$;
		\STATE Set $\widehat{\U} = \Q \widetilde{\U} $;
		\STATE {\bfseries Output:} the first $k$ columns of $(\widehat{\U}, \widehat{\Si},  \widehat{\V})$.
	\end{algorithmic}
\end{algorithm*}

\newpage
\subsection{Additional Experiments Results}
\label{appen:exp}

 \begin{table}[!ht]
	\newcommand{\tabincell}[2]{\begin{tabular}{@{}#1@{}}#2\end{tabular}}  
	\caption{Error comparison among three one-shot baseline algorithms and our \texttt{LocalPower}.
		We uniformly distribute $n$ samples into $m = \max(\lfloor\frac{n}{1000}\rfloor, 3)$ devices so that each device has about 1000 samples.
	We show the mean errors of ten repeated experiments with its standard deviation enclosed in parentheses.
	Here we use $p=4$ for all variants of \LP and sufficiently large $T$'s which guarantee \LP converges.
	For better visualization, we show the box plot of final errors of ten repeated experiments in Figure~\ref{fig:error_box}.
	}
	\label{table:baseline_additional}
	\centering 
	\vspace{0.05in}
		\resizebox{\textwidth}{30mm}{
\begin{tabular}{c|c|c|c|c|c|c}
		\toprule
	{\multirow{2}{*}{Datasets}}& 
	\multicolumn{3}{c}{\LP with $p=4$}\vline&
	{\multirow{2}{*}{\texttt{DR-SVD}}}& 
	{\multirow{2}{*}{\texttt{UDA}}}& 
	{\multirow{2}{*}{\texttt{WDA}}}\\
	\cline{2-4}
	  & OPT & Sign-fixing & Vanilla &  &  & \\
	\hline
	A9a&  \textbf{4.09e-03} (\textbf{4.20e-04})& 5.82e-03 (1.41e-03) &8.13e-02 (3.44e-02)&4.63e-02 (9.24e-03)& 2.64e-02 (1.58e-02)&2.40e-02 (1.50e-02)\\
    Abalone & \textbf{3.16e-03} (2.89e-03)& 3.85e-03 (\textbf{2.54e-03})&3.03e-02 (5.70e-02)& 3.20e-01 (2.30e-01)& 1.03e-01 (9.38e-02)&1.03e-01 (9.18e-02)\\    
	Acoustic & \textbf{1.83e-03} (4.40e-04)& 2.03e-03 (3.90e-04)& 2.38e-03 (8.50e-04)& 1.54e-02 (6.59e-03)&7.76e-03 (2.64e-03)& 6.67e-03 (2.41e-03)\\
	Aloi & 3.07e-02 (1.10e-02)& 6.57e-02 (1.06e-02)& 5.24e-02 (1.10e-02)& \textbf{1.92e-03} (\textbf{4.30e-04})& 4.80e-02 (1.10e-02)& 4.37e-02 (4.73e-03)\\
	Combined &6.01e-03 (1.59e-03)& \textbf{5.57e-03} (\textbf{1.05e-03}) &2.47e-02 (3.40e-02)& 5.19e-02 (6.23e-03)&4.63e-02 (2.97e-02)&4.16e-02 (2.76e-02)\\
	Connect-4 &  \textbf{1.27e-02} (4.52e-03)&1.81e-02 (3.79e-03)&1.70e-02 (4.35e-03)& 1.61e-02 (2.96e-03)& 1.65e-01 (3.48e-02) &1.56e-01 (3.26e-02)\\
	Covtype & 7.38e-03 (8.50e-04)& \textbf{6.23e-03} (\textbf{3.30e-04})& 1.28e-02 (1.88e-03)& 1.82e-01 (8.73e-02)& 6.09e-02 (9.70e-03)& 5.60e-02 (9.41e-03)\\
	Housing& \textbf{1.18e-02} (\textbf{5.45e-03})& 2.76e-02 (1.14e-02)& 3.84e-02 (5.11e-02)& 5.66e-01 (2.62e-01)& 9.16e-02 (5.09e-02)& 5.89e-02 (3.25e-02)\\
	Ijcnn1& \textbf{1.53e-01} (1.87e-01)& 1.95e-01 (2.45e-01)& 3.23e-01 (2.24e-01)& 1.21e+00 (1.70e-01)& 3.85e-01 (7.62e-02)& 3.67e-01 (\textbf{7.59e-02})\\	
	MNIST  & 2.62e-03 (3.40e-04)& 4.85e-03 (8.00e-04)& 5.08e-03 (7.90e-04)& \textbf{5.00e-05} (\textbf{0.00e+00})& 1.08e-02 (3.00e-03) &8.91e-03 (2.53e-03)\\	
	Poker & \textbf{6.45e-03} (\textbf{1.90e-03})&  1.08e-02 (3.34e-03)& 5.33e-02 (3.63e-02)& 1.25e+00 (1.61e-01)& 2.39e-02 (3.00e-03) & 2.00e-02 (2.19e-03)\\	
	Space-ga& \textbf{2.80e-04} (\textbf{1.40e-04})& 5.10e-04 (2.90e-04)& 6.50e-04 (3.60e-04)& 7.40e-01 (2.14e-01)& 2.83e-02 (2.46e-02)& 3.82e-02 (2.72e-02)\\
	Splice & \textbf{1.61e-02} (\textbf{5.46e-03})& 2.87e-02 (8.93e-03)& 7.45e-02 (9.26e-02)&4.52e-01 (1.37e-01)& 1.56e-01 (7.08e-02)& 1.34e-01 (6.26e-02)\\
	W8a&1.90e-02 (2.46e-03)&1.75e-02 (1.76e-03)& \textbf{1.68e-02} (\textbf{1.29e-03}) & 7.13e-02 (2.06e-02)& 1.52e-01 (4.37e-02 )&1.51e-01 (4.11e-02)\\
	 MSD & 9.90e-03 (1.21e-03)& \textbf{9.62e-03} (\textbf{5.20e-04})&1.44e-02 (1.58e-03)& 3.01e-02 (9.64e-03)& 1.55e-02 (1.39e-03)&1.92e-02 (1.14e-03)\\
	\bottomrule
\end{tabular}}
\end{table}

 \begin{table}[!ht]
	\caption{Error comparison among \texttt{LocalPower} with the decay strategy and three different $\FM$.
		We uniformly distribute $n$ samples into $m = \max(\lfloor\frac{n}{1000}\rfloor, 3)$ devices so that each device has about 1000 samples.
		We show the mean errors of ten repeated experiments with its standard deviation enclosed in parentheses.
		Here we use $p=4$ for all variants of \LP and sufficiently large $T$'s which guarantee \LP converges.
	}
	\label{table:baseline_decay_additional}
	\centering 
		\begin{tabular}{c|c|c|c}
			\toprule
			{\multirow{2}{*}{Datasets}}& 
			\multicolumn{3}{c}{\LP with the decay strategy}\\
			\cline{2-4}
			& OPT & Sign-fixing & Vanilla  \\
			\hline
			A9a& 4.84e-03 (1.40e-02)& 1.52e-03 (4.08e-03) &\textbf{3.11e-04} (\textbf{4.84e-04}) \\
			Abalone  &
			\textbf{3.50e-10} (4.10e-10)&4.14e-10  (\textbf{4.00e-10})&6.12e-10 (6.77e-10)\\
			Acoustic &
			\textbf{1.40e-05} (\textbf{2.16e-05}) &1.92e-05 (3.72e-05) &2.28e-05 (4.91e-05)\\
			Aloi &
			\textbf{5.82e-10} (\textbf{5.17e-10}) &1.71e-09 (2.20e-09)&2.36e-09 (2.14e-09)\\
			Combined &
			3.68e-03 (5.63e-03) &7.74e-03 (1.70e-02)& \textbf{2.99e-03} (\textbf{3.88e-03}) \\
			Connect-4 &
			4.90e-03 (8.47e-03)&3.58e-03 (4.35e-03)&\textbf{3.09e-03} (\textbf{3.16e-03})\\
			Covtype &
			5.57e-04 (1.55e-03)&\textbf{4.95e-05} (\textbf{5.40e-05}) &8.01e-05 (8.62e-05)\\
			Housing &
			\textbf{1.38e-05} (\textbf{2.88e-05})&2.20e-05 (5.66e-05)&2.08e-05 (5.68e-05)\\
			Ijcnn1 &
			3.56e-01 (1.97e-01)&3.33e-01 (\textbf{1.67e-01})& \textbf{3.32e-01} (1.72e-01) \\
			MNIST  &
			2.06e-05 (2.38e-05)&\textbf{1.72e-05} (1.62e-05)&1.72e-05 (\textbf{1.62e-05})\\
			Poker &
			\textbf{3.08e-03} (\textbf{1.49e-03})&3.22e-03 (1.82e-03)&3.22e-03 (1.93e-03)\\
			Space-ga&
			\textbf{3.47e-14} (2.13e-14)&3.56e-14 (\textbf{2.11e-14})&3.87e-14 (2.27e-14)\\
			Splice &
			\textbf{4.11e-07} (\textbf{5.29e-07})&8.88e-07 (1.24e-06)&1.01e-06 (1.34e-06)\\
			W8a& \textbf{1.70e-03} (\textbf{2.46e-03}) & 1.85e-02 (4.94e-02)&6.09e-03 (9.60e-03)\\
			MSD & 2.75e-05 (3.34e-05)&\textbf{2.47e-05} (3.27e-05)& 3.02e-05 (\textbf{2.10e-05})\\
			\bottomrule
	\end{tabular}
\end{table}

\begin{figure*}[!ht]
	\centering
		\includegraphics[width=\columnwidth] {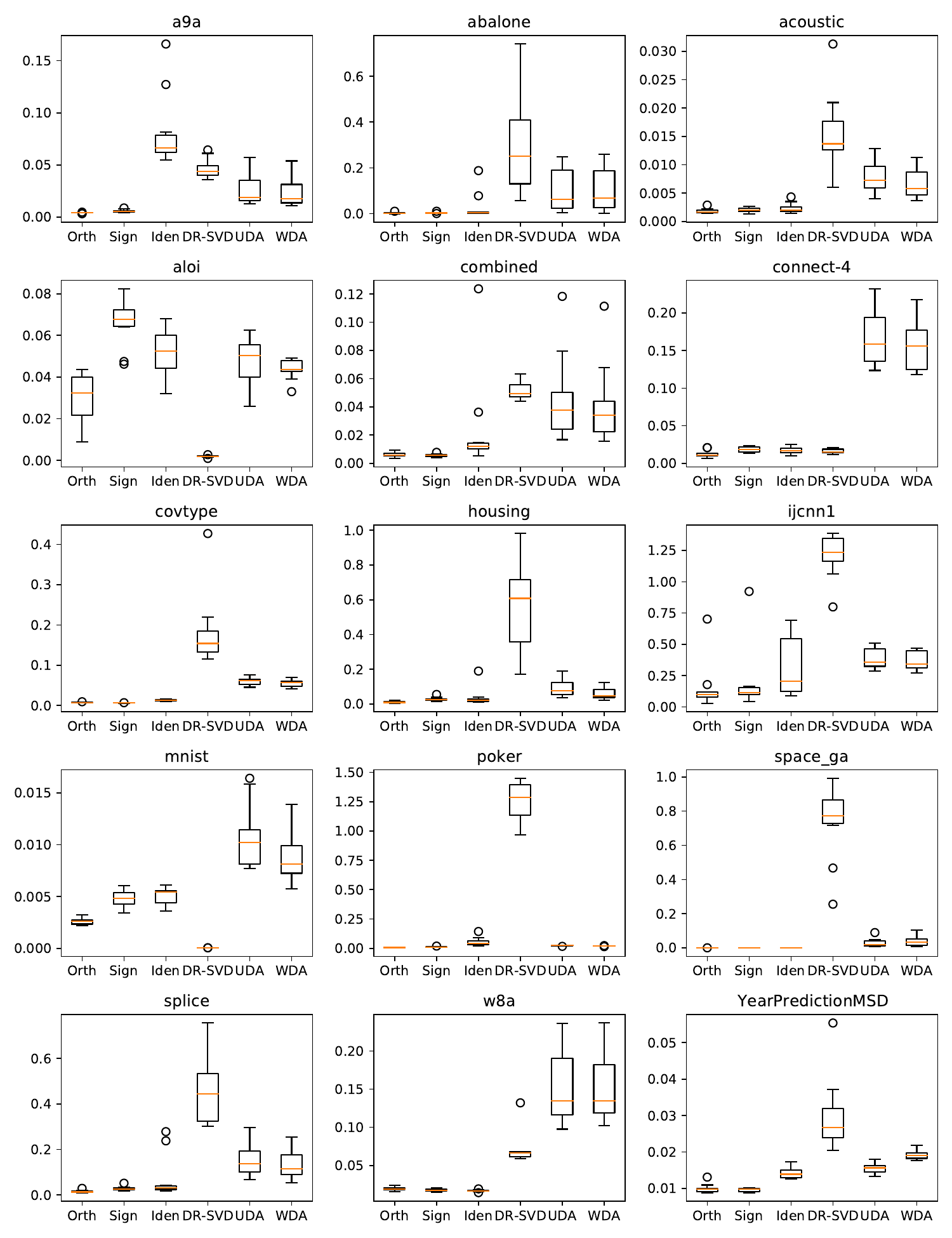}
	\caption{Box plot of Table~\ref{table:baseline_additional} for better visualization.
	Here \textsf{Orth}, \textsf{Sign} and \textsf{Iden} represents OPT, sign-fixing and the vanilla \LP respectively.\protect\footnotemark ~
	We can see that for most datasets, \LP with $p=4$ obtains smallest error and more stability.
    We can obtain zero error if we use the decay strategy.}
	\label{fig:error_box}
\end{figure*}
\footnotetext{Actually, it means setting $\FM$ for \LP as $\OM_{k}, \DM_k$ and $\{\I_k\}$ respectively (see~\eqref{eq:D0} for the reason).}

\begin{figure*}[!ht]
	\centering
	\includegraphics[width=\columnwidth] {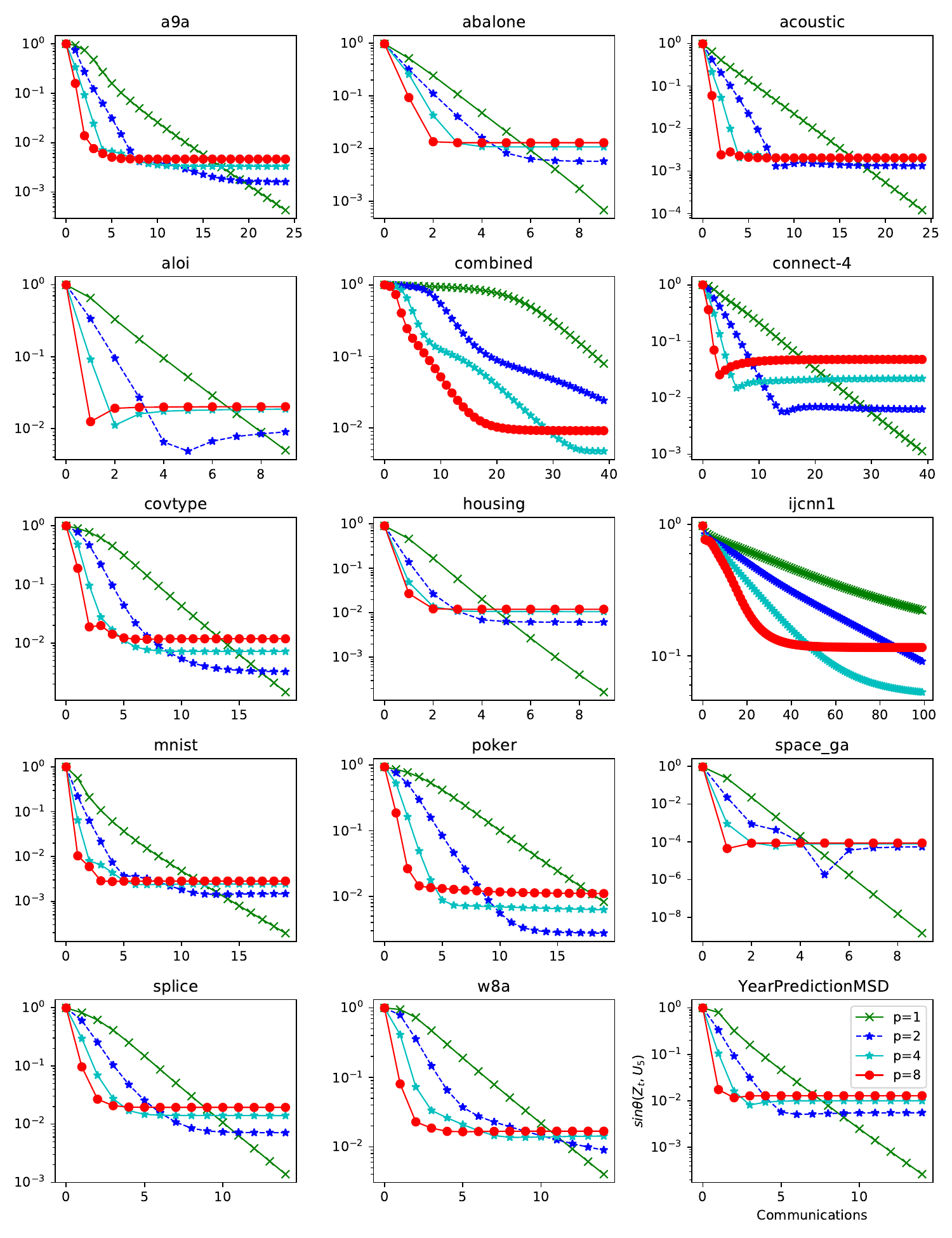}
	\caption{Vary $p$ for \LP with OPT.
	Typically, the larger $p$, the larger error, which is consistent with our theory.
%	The curve of $p=8$ on Ijcnn1 is quite strange.
%	We speculate this is because the dataset Ijcnn1 is quite hard; indeed, from Table~\ref{table:baseline_additional}, the final errors obtained by all results are much larger than other datasets.
%	However, from Figure~\ref{fig:p_sign}, we find that \LP with $\FM = \DM$ is much stable and its curve is smoother.
	Typically, \LP with OPT achieves the smallest error among our three proposed methods.
}
	\label{fig:p_orth}
\end{figure*}

\begin{figure*}[!ht]
	\centering
	\includegraphics[width=\columnwidth] {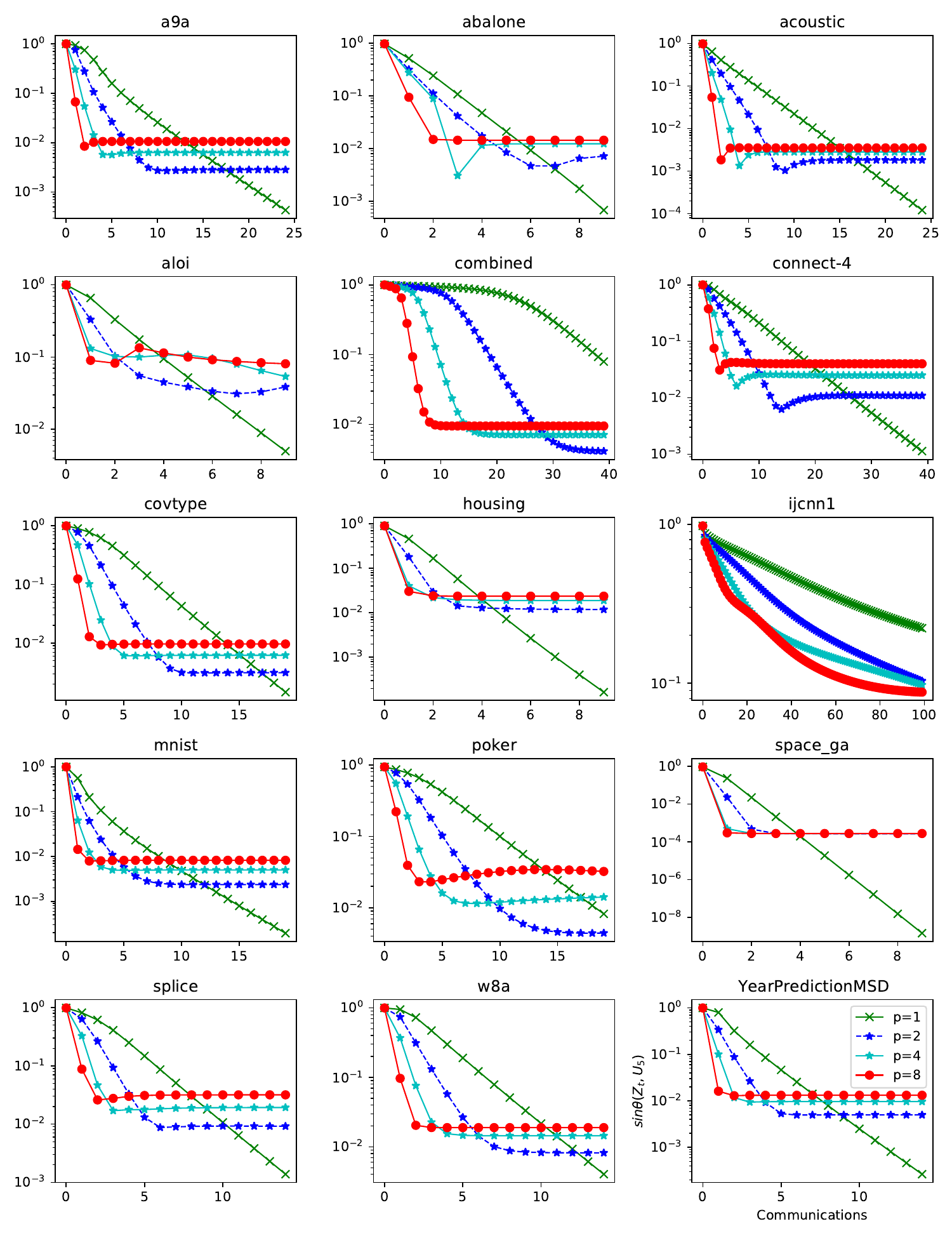}
	\caption{Vary $p$ for \LP with sign-fixing
		Similar to Figure~\ref{fig:p_orth}, the larger $p$, the larger error, which is consistent with our theory.
	\LP with sign-fixing is much computation efficient than that with OPT.
	Sign-fixing can be viewed as a good practical of surrogate of OPT.
	}
	\label{fig:p_sign}
\end{figure*}

\begin{figure*}[!ht]
	\centering
	\includegraphics[width=\columnwidth] {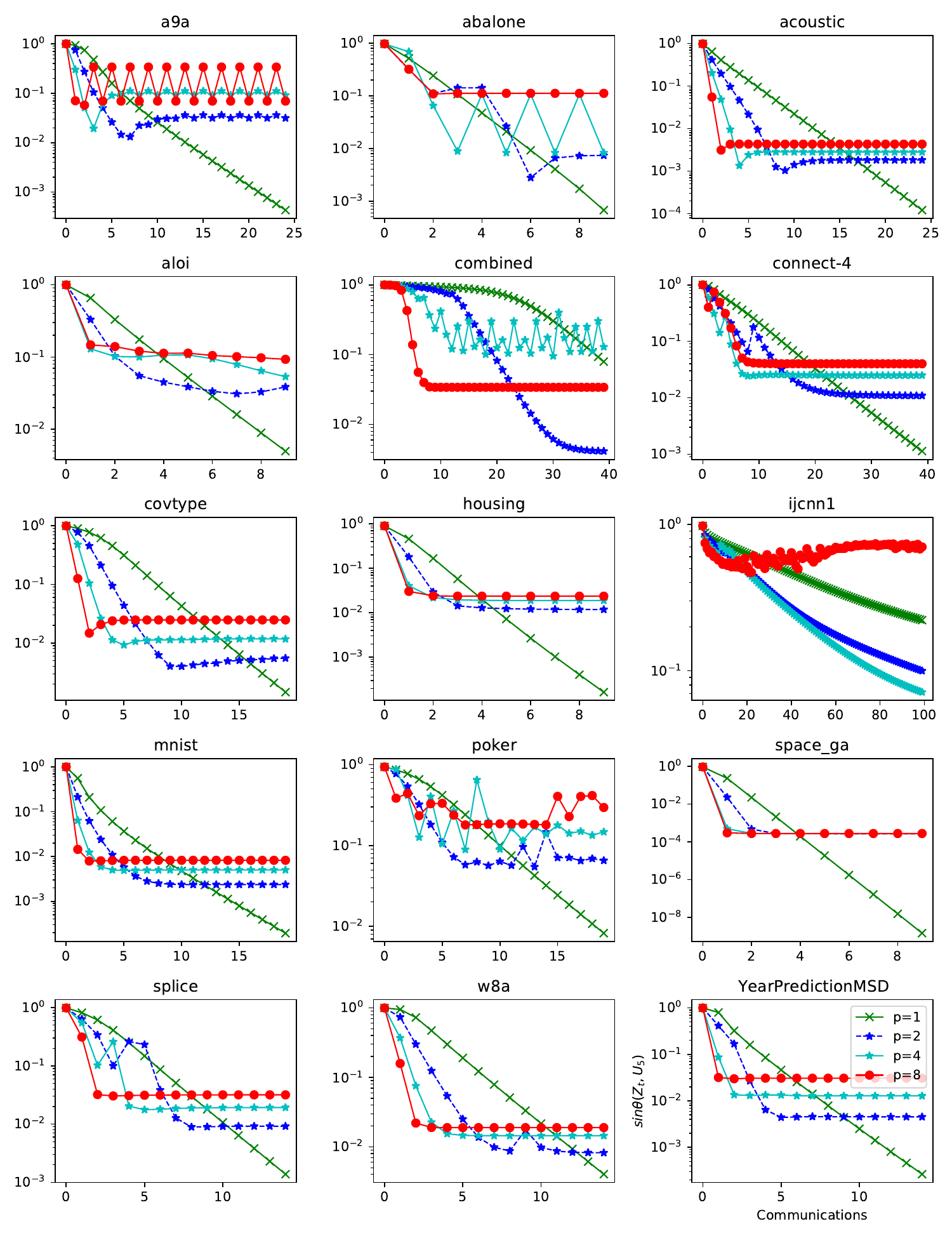}
	\caption{Vary $p$ for vanilla \LP. 
		For most datasets, vanilla \LP converges and the similar pattern that the larger $p$, the larger error occurs.
		However, for large $p$, it fluctuates and even diverges on some datasets (including A9a, Abalone, Combined, Ijcnn1 and Poker).
		This is because $\eta$ can't meet required smallness.
		As argued,  \LP with OPT or sign-fixing typically is more stable than the vanilla one, since it requires less strict smallness of $\eta$.
		Besides, we can use the decay strategy or decreases the number of devices.
	}
	\label{fig:p_iden}
\end{figure*}

\begin{figure*}[!ht]
	\centering
	\includegraphics[width=\columnwidth] {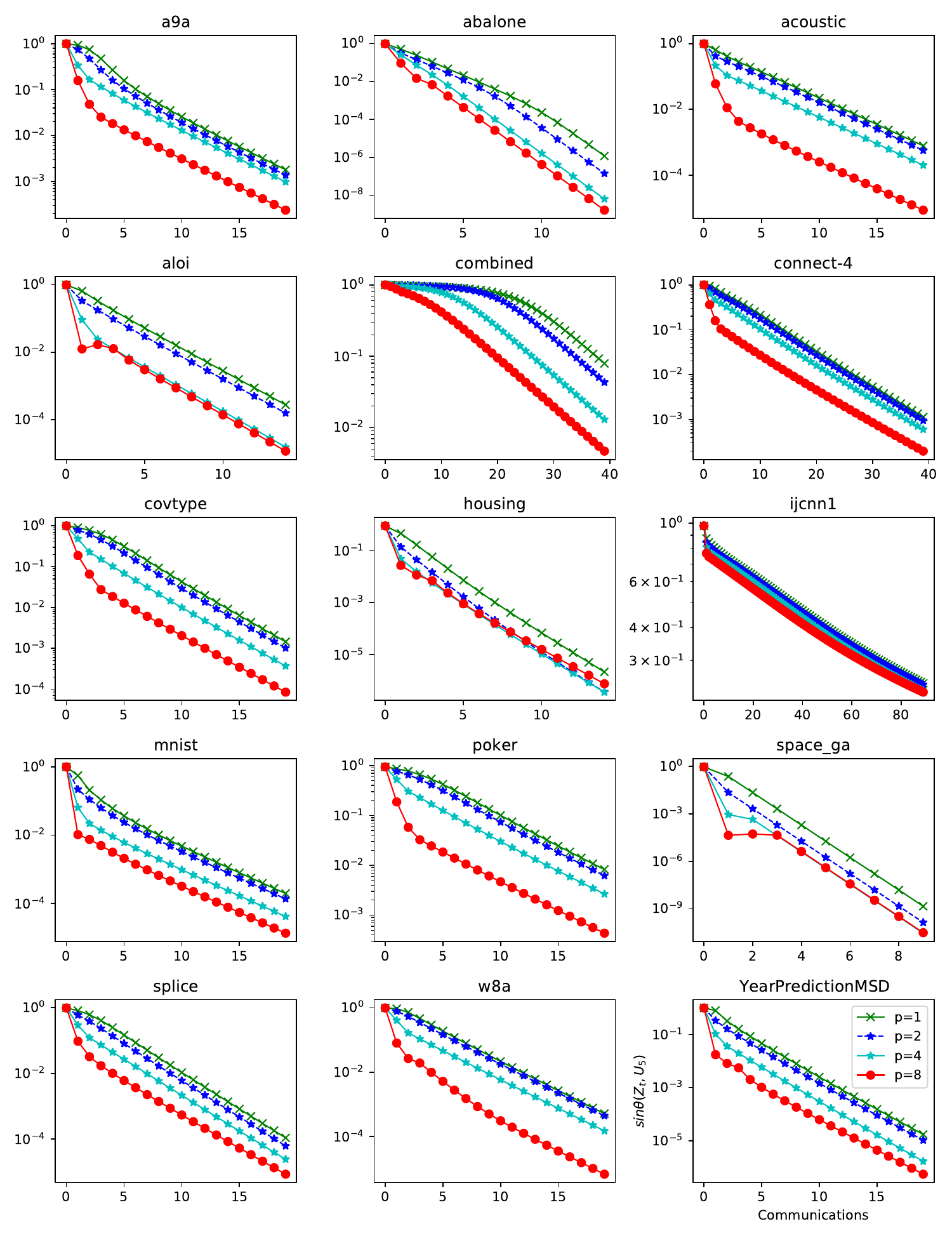}
	\caption{Decay strategy for \LP with OPT.
		For most datasets, \LP with OPT converges faster and achieves much less error than non-decay counterparts (see Figure~\ref{fig:p_orth}).
		Theoretically, \LP with decay strategy can achieve zero error.
	}
	\label{fig:p_orth_decay}
\end{figure*}

\begin{figure*}[!ht]
	\centering
	\includegraphics[width=\columnwidth] {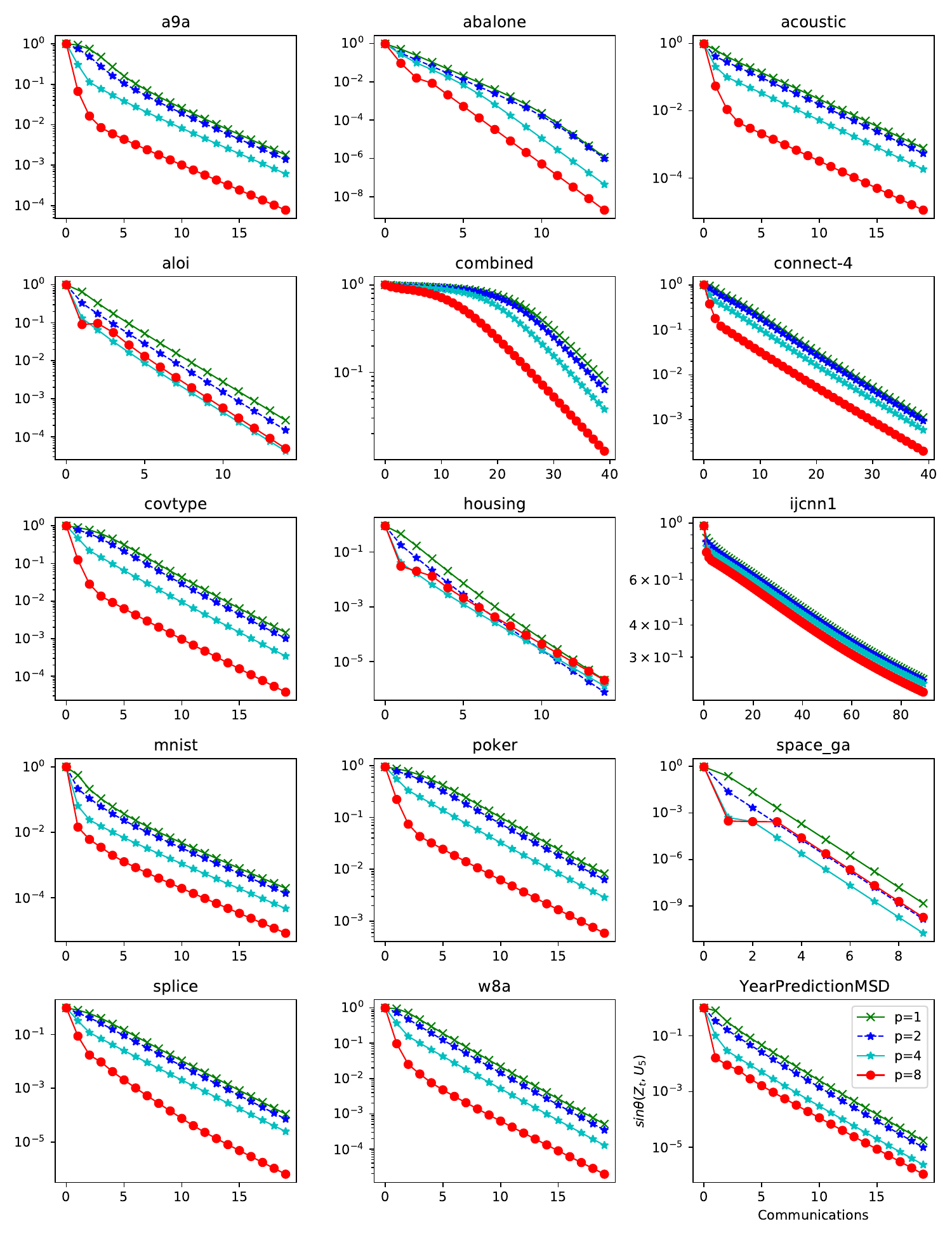}
	\caption{Decay strategy for \LP with sign-fixing.
		For most datasets, \LP with sign-fixing converges faster and achieves much less error than non-decay counterparts (see Figure~\ref{fig:p_sign}).
		Theoretically, \LP with decay strategy can achieve zero error.
	}
	\label{fig:p_sign_decay}
\end{figure*}

\begin{figure*}[!ht]
	\centering
	\includegraphics[width=\columnwidth] {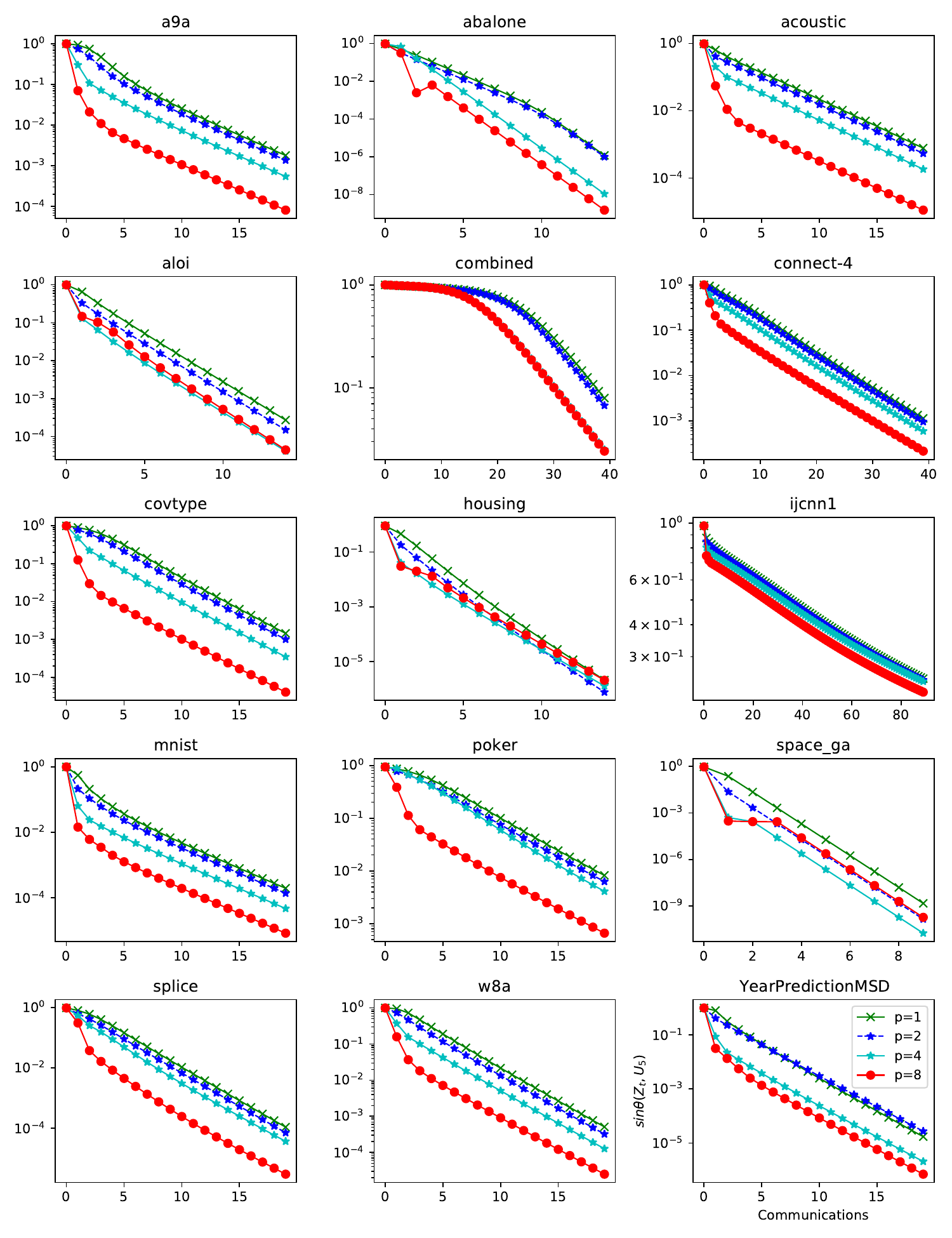}
	\caption{Decay strategy for vanilla \LP.
		For most datasets, vanilla \LP converges faster and more stable than non-decay counterparts (see Figure~\ref{fig:p_sign}).
		It typically achieves much less error than non-decay counterparts.
		Theoretically, \LP with decay strategy can achieve zero error.
	}
	\label{fig:p_iden_decay}
\end{figure*}

\begin{figure*}[!ht]
	\centering
	\includegraphics[width=\columnwidth] {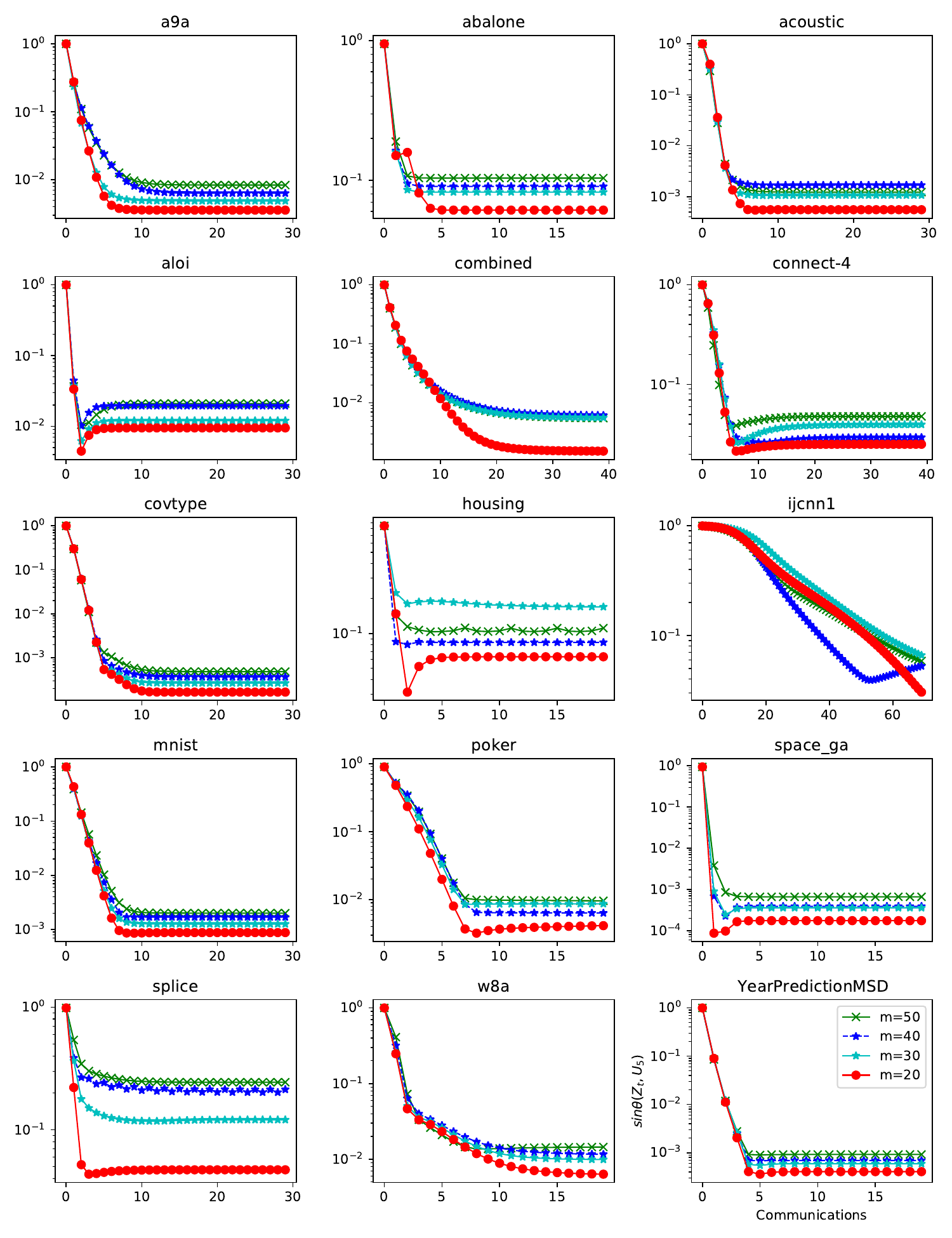}
	\caption{Various $m$ for \LP with OPT. Typically, the smaller $m$ has smaller errors.
	}
	\label{fig:m_orth}
\end{figure*}

\begin{figure*}[!ht]
	\centering
	\includegraphics[width=\columnwidth] {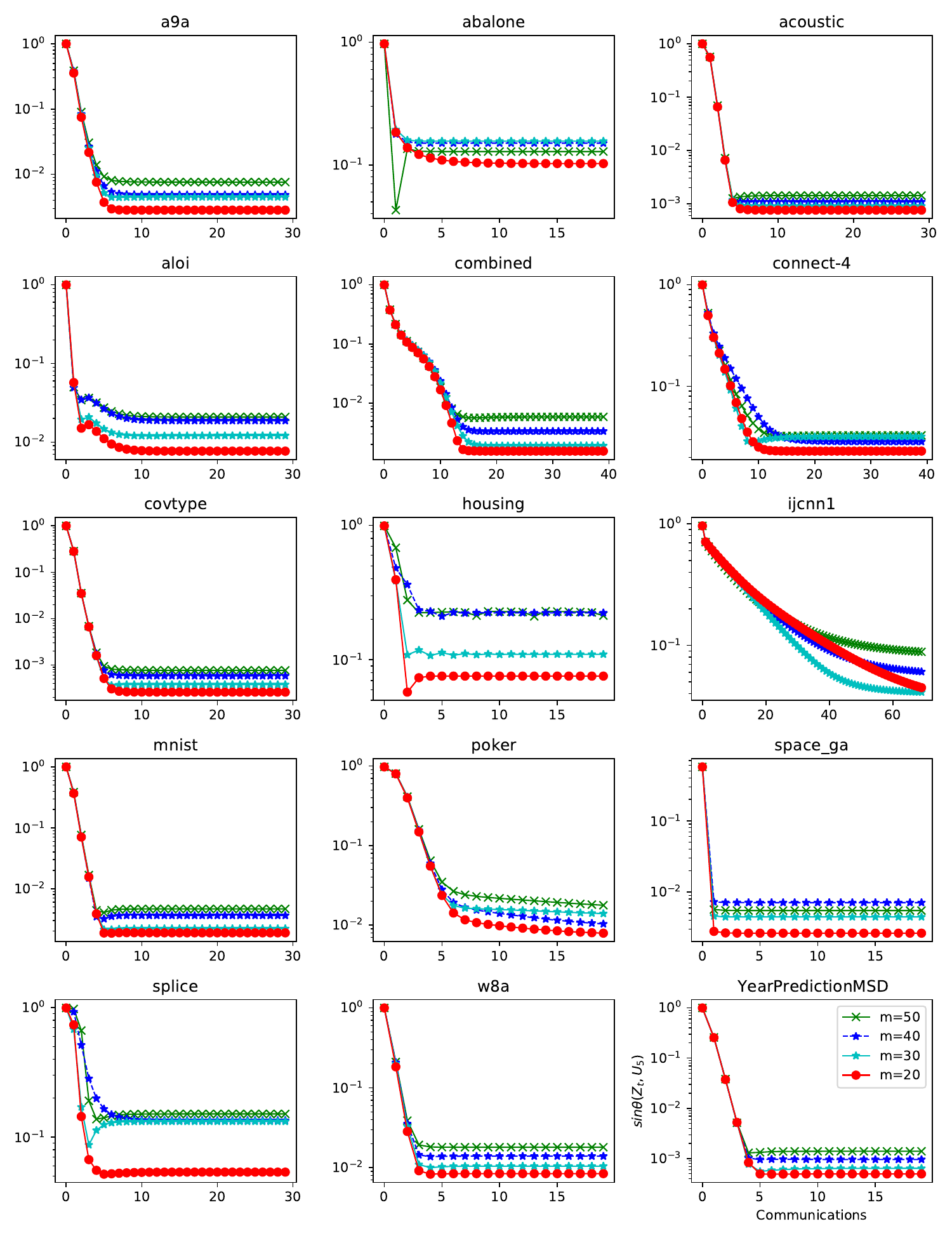}
	\caption{Various $m$ for \LP with sign-fixing. Typically, the smaller $m$ has smaller errors.
	}
	\label{fig:m_sign}
\end{figure*}

\begin{figure*}[!ht]
	\centering
	\includegraphics[width=\columnwidth] {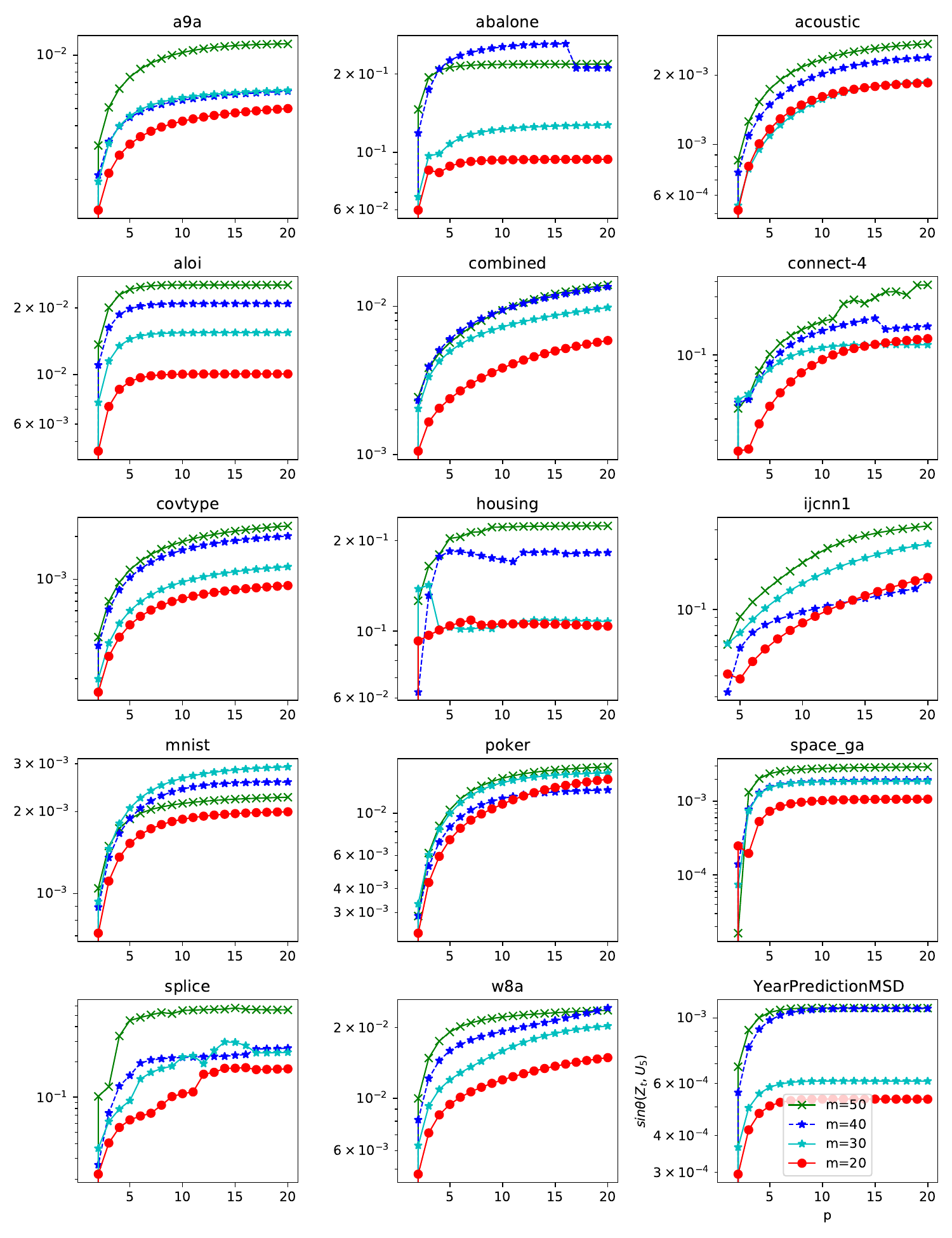}
	\caption{Error dependence of \LP with OPT.
	}
	\label{fig:sin_m_p}
\end{figure*}
\end{document}